\documentclass{article}

\usepackage{xr}
\usepackage[utf8]{inputenc}
\usepackage{amssymb,amsthm}
\usepackage{graphicx,calc}
\usepackage{natbib}
\usepackage[top=3cm,inner=3cm,outer=3cm,bottom=4cm]{geometry}

\usepackage{hyperref}
\usepackage{url}
\usepackage{amsmath,stmaryrd}
\usepackage{color}
\usepackage{enumitem}
\usepackage{upgreek}
\usepackage{epstopdf}

\let\originalleft\left
\let\originalright\right
\renewcommand{\left}{\mathopen{}\mathclose\bgroup\originalleft}
\renewcommand{\right}{\aftergroup\egroup\originalright}

\newtheorem{theorem}{Theorem}[section]

\newtheorem{lemma}[theorem]{Lemma}
\newtheorem{proposition}[theorem]{Proposition}
\newtheorem{definition}[theorem]{Definition}
\newtheorem{remark}[theorem]{Remark}

\newcommand{\paren}[1]{\left(#1\right)}
\newcommand{\brac}[1]{\left[#1\right]}
\newcommand{\inner}[1]{\left\langle#1\right\rangle}
\newcommand{\norm}[1]{\left\|#1\right\|}
\newcommand{\set}[1]{\left\{#1\right\}}
\newcommand{\abs}[1]{\left\lvert #1 \right\rvert}

\newcommand{\order}[1]{\mathcal{O}\left(#1\right)}

\def \Exp {\mathbb{E}}
\def\PP{\mathbb{P}}

\newcommand{\eps}{\varepsilon}

\renewcommand{\det}[1]{\operatorname{det}\left(#1\right)}

\def \RR {\mathbb{R}}
\def \CC {\mathbb{C}}

\def \Simplex {\mathbb{S}}

\def \SketchingOperator {\mathrm{\Phi}}
\def \SketchingOperatorProb {\mathcal{A}}

\def \nMeasures {m}
\def\nSamples{n}
\def\sampleDim{d}

\def\sample{x}
\def\Sample{X}

\def\dataset{\mathbf{X}}

\def \Prob {\pi}

\def \mProb {{\tau}}
\def \empProb{\hat{\Prob}_{\nSamples}}
\def \Param {\theta}
\def \HH {\mu}
\def \estProb{\widetilde{\Prob}}

\def \ParamSpace {\mathbf{\Theta}}
\def \SampleSpace{\mathcal{Z}}

\def \FClass {\mathcal{F}}
\def \GClass {\mathcal{G}}
\def \MutualCoherence {M}
\def \Model {\mathfrak{S}}
\def \ModelML {\Model^{\mathtt{ML}}}
\def \ModelCT {\Model^{\mathtt{CT}}}
\def \BasicSet {\mathcal{T}}

\newcommand{\MixSetSep}[1]{\Model_{#1}\left(\BasicSet\right)}
\newcommand{\sep}{\eps}
\def\secant{\mathcal{S}}
\def \dipoleSet {\mathcal{D}}
\def \monopoleSet {\mathcal{M}}

\def \distIOPexgen {d}
\newcommand\normfclass[2]{\norm{#1}_{\mathcal{#2}}}
\newcommand\dnormlossf[2]{\norm{#1}_{\Delta\LossClass(\HypClass_{#2})}}
\newcommand\dnormloss[2]{\norm{#1}_{\Delta\LossClass}}
\newcommand\normdloss[2]{\norm{#1}_{\DLossClass(\HypClass_{#2})}}

\newcommand\normrff[1]{\norm{#1}_{\FClass}}
\newcommand\normrffd[1]{\norm{#1}_{\FClass'}}
\newcommand\normrffdd[1]{\norm{#1}_{\FClass''}}
\newcommand\normmah[2]{\norm{#1}_{#2}}
\newcommand\normkern[1]{\norm{#1}_{\kernel}}
\newcommand{\normTV}[1]{\norm{#1}_{\textnormal{TV}}}
\def \metricParam {\varrho}
\newcommand\KLdiv[2] {\textnormal{KL}(#1||#2)}
\newcommand\Entropy[1] {\textnormal{H}(#1)}

\def \loss {\ell}

\def \LossClass {\mathcal{L}}
\def \DLossClass {\Delta\mathcal{L}}
\def \hyp {h}
\def \HypClass{\mathcal{H}}
\def \HypClassSep{\HypClass_{\mathtt{sep}}}
\def \HypClassOpt{\HypClass_{\mathtt{opt}}}
\def \HypClassRef{\overline{\HypClass}}

\def \Risk {\mathcal{R}}

\def \proxyRisk {R}

\def \Sketch {\mathtt{Sketch}}

\def \ConcFn {c_{\kernel}}
\def \covnumsymbol {\mathrm{N}}
\newcommand{\covnum}[3]{\covnumsymbol\left(#1,#2,#3\right)}
\newcommand{\covnumsq}[3]{\covnumsymbol^2\left(#1,#2,#3\right)}
\newcommand{\covnumpow}[4]{\covnumsymbol^{#1}\left(#2,#3,#4\right)}

\def \coveps {\delta}
\def \probLevel {\zeta}

\newcommand{\dloss}{\Delta \loss}
\newcommand{\drisk}{\Delta \Risk}

\newcommand{\divp}{D}
\newcommand{\divg}{d}

\def \kernel {\kappa}
\def \nkernel {\overline{\kernel}}
\def \freqdist {\Lambda}
\def \freq {\omega}

\def \freqSpace{\Omega}
\def \feat {\phi}
\def \rfeat {{\feat_{\omega}}}

\def \rfeatj {{\feat_{\omega_j}}}
\def \rfeatn {{\feat_{\omega_n}}}

\def \vy {\mathbf{y}}
\def \vc {\mathbf{c}}
\def \vu {\mathbf{u}}

\def \mSigma {{\boldsymbol\Sigma}}
\def \mGamma {{\boldsymbol\Gamma}}

\def \mA {\mathbf{A}}

\def \mR {\mathbf{R}}
\def \mI {\mathbf{I}}

\newcommand{\enet}{\mathcal{C}}

\def \embd {\varphi}

\def \Ball {\mathcal{B}}

\def \weightSet {\mathcal{W}}

\def \covar {\mSigma}

\def \Cov {\mathbf{\Sigma}}

\definecolor{darkpurple}{rgb}{0.3,0,0.3}

\def \PCAdim {k}

\newcommand\rev[1]{{#1}}

\def\citepartone{\rev{\citep{gribonval:hal-01544609}}}
\newcommand\citeppartone[1]{\rev{\citep[#1]{gribonval:hal-01544609}}}
 \usepackage{supertabular}

\begin{document}

\title{\rev{Statistical Learning Guarantees for\protect\\ Compressive Clustering and Compressive Mixture Modeling}}
\author{
R{\'e}mi Gribonval\thanks{Univ Lyon, Inria, CNRS, ENS de Lyon, UCB Lyon 1, LIP UMR 5668, F-69342, Lyon, France\protect\\
This work was initiated while R. Gribonval, N. Keriven and Y. Traonmilin were with Univ Rennes, Inria, CNRS, IRISA\protect\\ F-35000 Rennes, France}  
\qquad \hfill remi.gribonval@inria.fr\\
Gilles Blanchard 
\thanks{Université Paris-Saclay, CNRS,  Laboratoire de mathématiques d'Orsay\protect\\
  F-91405, Orsay, France.}
\qquad \hfill gilles.blanchard@universite-paris-saclay.fr\\
Nicolas Keriven\thanks{CNRS, GIPSA-lab, UMR 5216, F-38400 Saint-Martin-d'H\`eres, France}
\hfill 
nicolas.keriven@gipsa-lab.grenoble-inp.fr\\
Yann Traonmilin\thanks{CNRS, Univ. Bordeaux, Bordeaux INP,  IMB, UMR 5251, F-33400 Talence, France.} 
\hfill
yann.traonmilin@u-bordeaux.fr
 }

\maketitle

\begin{abstract}
\rev{We provide statistical learning guarantees for two unsupervised learning tasks in the context of {\em compressive statistical learning}, a general framework for resource-efficient large-scale learning that we introduced in a companion paper.
The principle of compressive statistical learning is to compress  a training collection, in one pass, into a low-dimensional {\em sketch} (a vector of random empirical generalized moments) that captures the information relevant to the considered learning task. We \rev{explicitly describe and analyze} random feature functions which empirical averages preserve the needed information for {\em compressive clustering} and {\em compressive Gaussian mixture  modeling} with fixed known variance, and establish sufficient sketch sizes given the problem dimensions.}
\end{abstract}


\section{Introduction}

\rev{Motivated by the need to handle large-scale learning in streaming or distributed contexts with limited memory, we studied in a companion paper \citepartone\ a general {\em compressive learning framework} based on a generic sketching mechanism, using empirical averages of a random feature function to compress a whole training collection into a single {\em sketch} vector. Learning from such a sketch is expressed as a (generalized) moment fitting problem. Statistical learning guarantees control the excess risk of the overall procedure, provided the used random feature function satisfies certain properties with respect to the considered learning task. The size of the sketch can also be controlled.}

\rev{For compressive clustering, also known as compressive k-means \citep{Keriven2016}, and compressive Gaussian mixture modeling \citep{Keriven2015,keriven:hal-01329195}, good empirical results have been obtained with compressive statistical learning using {\em random Fourier moments}, i.e. using empirical averages of random Fourier features \citep{Rahimi2007}. Based on the general framework of \citepartone, we establish that these empirical results are supported by statistical learning guarantees controlling the excess risk and the sketch sizes as a function of the problem dimensions.}

\rev{ For {\bf compressive clustering} in dimension $\sampleDim$, we demonstrate that a sketch of size}
$$m \geq C\PCAdim^{2} \sampleDim \cdot \log^{2}k \cdot \left(\log(\PCAdim\sampleDim)+ \log(R/\sep)\right),$$ with $\PCAdim$ the prescribed number of clusters,
 $R$ a bound on the norm of the centroids, $\sep$ the separation between them, and $C$ some universal constant, is sufficient to obtain statistical guarantees. 
 \rev{To the best of our knowledge these are the first guarantees of this kind:
   while there is a substantial literature on asymptotical and nonasympotical
 guarantees on convergence rates for clustering 
 \cite{pollard1982central,chou1994distortion, linder1994rates, bartlett1998minimax, antos2005individual, antos2005improved, fischer2010quantization, Levrard:2013ho}, they are based on the full
 data while our focus is on analyzing {\em sketched} clustering.}

\rev{For {\bf compressive Gaussian mixture estimation} with known covariance in dimension $\sampleDim$, we identify} a finite sketch size sufficient to obtain statistical guarantees under a separation assumption between means, expressed in the Mahalanobis norm associated to the known covariance matrix. A parameter embodies the tradeoff between sketch size and separation. 
 At one end of the spectrum  
 the sketch size is quadratic 
 in $\PCAdim$ and exponential in $\sampleDim$ and guarantees are given for means that can be separated in $\order{\sqrt{\log\PCAdim}}$. This compares favorably to existing literature \citep{Achlioptas2005,Vempala2004} (recent works make use of more complex conditions that theoretically permit arbitrary separation \citep{Belkin2010}, however all these approaches use the full data while we consider a compressive approach that uses only a sketch of the data). At the other end  
 the sketch size is quadratic
 in $\PCAdim$ and {\em linear} in $\sampleDim$, however the required separation is \rev{of the order of} $\sqrt{\sampleDim\log\PCAdim}$.
 
 \rev{After recalling the outline of the general framework for compressive statistical learning of \citepartone\ in Section~\ref{sec:general_framework},} a sketching procedure and the associated learning guarantees for compressive clustering (respectively for compressive Gaussian mixture estimation) are given in Section~\ref{sec:clustering} (respectively Section~\ref{sec:gmm}). \rev{The rest of the paper is dedicated to establishing these results.} Section~\ref{sec:general_mixtures} describes a generic approach to establish learning guarantees when estimation of mixtures of elementary distributions are involved, as in the two \rev{considered} examples. The results are then specialized in Section~\ref{sec:RBFKernel} to mixtures based on location families, using weighted random Fourier features. The most technical proofs are postponed to appendices, including the proof of the main results of Sections~\ref{sec:clustering} and~\ref{sec:gmm}.

\section{\rev{Overview of compressive statistical learning}}\label{sec:general_framework}

\rev{In statistical learning, one is given a training collection  $\dataset=\{\sample_{i}\}_{i=1}^{\nSamples} \in \SampleSpace^{\nSamples}$ assumed to be drawn i.i.d. from a probability distribution $\Prob$ on the \rev{measurable space $(\SampleSpace,\mathfrak{Z})$}. In our examples, $\SampleSpace = \RR^\sampleDim$ \rev{is endowed with the
Borel $\sigma$-algebra $\mathfrak{Z}$}. A learning task (supervised or unsupervised) is formally defined through a {\em loss function} $\loss: (\sample,\hyp) \mapsto \loss(\sample,\hyp) \in \RR$ which measures how adapted is a training sample $\sample$ to a hypothesis $\hyp$ from some hypothesis class $\HypClass$.
\rev{The overall goal is to select a hypothesis $\hyp^{\star}$ minimizing the {\em expected risk} 
\(
\Risk(\Prob,\hyp) := \Exp_{\Sample \sim \Prob}\ \loss(\Sample,\hyp)\,,
\)
\begin{equation}\label{eq:BestHyp}
\hyp^{\star} \in\arg\min_{\hyp \in \HypClass} \Risk(\Prob,\hyp).
\end{equation}
\begin{remark}\label{rmk:integrability} We will always implicitly restrict our attention to probability distributions $\Prob$ that are $\LossClass(\HypClass)$-integrable, i.e. such that $x\mapsto \ell(x,h)$ is measurable and $\Prob$-integrable for all $h\in \HypClass$.\end{remark}
In practice, since the expected risk cannot be computed from the training collection, a common strategy is instead to minimize the {\em empirical risk} $\Risk(\empProb,\hyp)$ (or a regularized version) associated to the  {\em empirical probability distribution} $\empProb := \tfrac{1}{\nSamples}\sum_{i=1}^{\nSamples} \delta_{\sample_{i}}$ of the training samples. 
The two primary tasks considered in this paper are:}
\begin{itemize}
\item {\bf $k$-means \rev{(resp. $k$-medians)} clustering}: each hypothesis $\hyp$ corresponds to a set of (at most) $k$ candidate cluster centers, $c_1,\ldots,c_k$, and the loss is defined by the $k$-means cost $\loss(\sample,\hyp) = \min_{1 \leq l \leq k} \norm{\sample-c_{l}}_{2}^{2}$, \rev{(resp. the $k$-medians cost $\loss(\sample,\hyp) = \min_{1 \leq l \leq k} \norm{\sample-c_{l}}_{2}$)}. The hypothesis class $\HypClass$ may be further reduced by defining constraints on the considered centers (e.g., in some domain, or as we will see with some separation between centers). %
\item {\bf Gaussian Mixture Modeling  with fixed \rev{covariance matrix $\covar$}}: each hypothesis $h$ corresponds to the collection of weights $\alpha_{l}$ and means $c_{l}$ of a mixture of $k$ Gaussians \rev{$\Prob_{c_{l}}:=\mathcal{N}(c_{l},\covar)$}, which probability density function is denoted $\Prob_h(\sample) = \sum_{l=1}^{k}\alpha_{l} \mathcal{N}(c_{l},\covar)$. \rev{The mixture parameters may again further be constrained by boundedness or separation assumptions.} The loss function $\ell(\sample,h)=-\log \Prob_h(\sample)$ \rev{is associated to the maximum likelihood estimation principle}. %
\end{itemize}
\subsection{Principles of {\em compressive} statistical learning}
Compressive learning proposes instead to choose a measurable (nonlinear) {\em feature function} $\SketchingOperator:
 \SampleSpace \mapsto  \RR^{\nMeasures}\ \text{or}\ \CC^{\nMeasures}$ and to proceed in two steps:
\begin{enumerate}
\item Compute empirical averages of the feature function to obtain a \emph{sketch vector}
\begin{equation}\label{eq:GenericSketching}
\vy := \Sketch(\dataset) := \frac{1}{n} \sum_{i=1}^{\nSamples}\SketchingOperator(\sample_i) \in \RR^{\nMeasures} \text{or}\ \CC^{\nMeasures};
\end{equation}
\item Produce a hypothesis from the sketch by solving an optimization problem
\begin{equation}\label{eq:DefRiskProxy}
\hat{\hyp} \in \arg\min_{\hyp \in \HypClass} \proxyRisk(\vy,\hyp)
\end{equation}
with some adequate proxy $\proxyRisk(\vy,\cdot)$ for the empirical risk $\Risk(\empProb,\cdot)$.
\end{enumerate}
}
\rev{For the learning tasks considered in this paper, the feature function $\SketchingOperator$ is built using random Fourier features \citep{Rahimi2007,Rahimi2009}. For compressive $k$-medians/$k$-means with $h = (c_{1},\ldots,c_{k})$, the proxy is 
 \begin{equation}\label{eq:RiskProxyKMeans}
\proxyRisk_{\mathtt{clust.}}(\vy,\hyp) := \min_{\alpha \in \Simplex_{\PCAdim-1}}  \norm{\sum_{l=1}^k\alpha_l \SketchingOperator(c_l)-\vy}_2,
\end{equation}
with $\Simplex_{\PCAdim-1} := \set{\alpha \in \RR^{\PCAdim}: \alpha_\ell \geq 0; \sum_{l=1}^{k}\alpha_{l}=1}$ the simplex. For compressive GMM, it reads
 \begin{equation}\label{eq:RiskProxyGMM}
\proxyRisk_{\mathtt{GMM}}(\vy,\hyp) :=  \norm{\sum_{l=1}^k\alpha_l\Psi(c_{l})-\vy}_2.
\end{equation}
with $\Psi(c_{l}) := \Exp_{X \sim \mathcal{N}(c_l,\covar)} \SketchingOperator(X)$. The nonlinear parametric optimization problems corresponding to the minimization of~\eqref{eq:RiskProxyKMeans}-\eqref{eq:RiskProxyGMM} have be empirically addressed with success using continuous analogs of greedy algorithms for sparse reconstruction in inverse linear problems \citep{Keriven2015,Keriven2016,keriven:hal-01329195}.}

\subsection{\rev{Statistical learning guarantees for compressive statistical learning}}\label{sec:csl}
\rev{From a statistical learning perspective,  the goal} is to control the {\em excess risk} $\Risk(\Prob,\hat{\hyp}) - \Risk(\Prob,\hyp^\star)$.
In compressive statistical learning, this requires measuring a ``distance'' between the data distribution $\Prob$ and some {\em model set} $\Model$. 
Specific instances of model sets considered in this paper are:
\begin{itemize}
\item for compressive $k$-means / $k$-medians: mixtures of $k$ (separated) Diracs;
\item for compressive Gaussian mixture modeling: mixtures of $k$ (separated) Gaussians.
\end{itemize}
An important conceptual tool is the so-called {\em sketching operator} $\SketchingOperatorProb$ defined by 
\begin{equation}\label{eq:SketchingOperatorProbDef}
 \SketchingOperatorProb(\Prob)  :=\Exp_{\Sample \sim \Prob}\SketchingOperator(\Sample)
\end{equation}
which is linear in the sense that\footnote{One can indeed extend $\SketchingOperatorProb$ to a linear operator on the space of finite signed measures such that $\SketchingOperator$ is integrable, see \citeppartone{Appendix~\ref{P1-sec:FiniteSignedMeasures}}.} $\SketchingOperatorProb(\theta \Prob+(1-\theta)\Prob') = \theta\SketchingOperatorProb(\Prob)+(1-\theta)\SketchingOperatorProb(\Prob')$ for any $\Prob,\Prob'$ and $0 \leq \theta \leq 1$. 
A key property of this operator for compressive statistical learning is the preservation of certain task-driven metrics on probability distributions from the considered model set $\Model$. \rev{Denoting
 \begin{equation}\label{eq:defexcessrisk}
  \drisk_{\hyp_0}(\pi,\hyp) := \Risk(\pi,\hyp) - \Risk(\pi,\hyp_0),
\end{equation}
the excess risk relative to a reference hypothesis $\hyp_0$, the {\em excess risk divergence} with respect to $\hyp_0$ is
\begin{equation}\label{eq:lossdiv}
  \divp_{\hyp_0}^{\HypClass}(\Prob\|\Prob') := \sup_{\hyp \in \HypClass}  \paren{ \drisk_{\hyp_0}(\Prob,\hyp) - \drisk_{\hyp_0}(\Prob',\hyp)}
  \geq 
  \paren{ \drisk_{\hyp_0}(\Prob,\hyp_{0}) - \drisk_{\hyp_0}(\Prob',\hyp_{0})} = 0.
\end{equation}
}
Given a class $\GClass$ of measurable
functions $g: \SampleSpace \to \RR\ \text{or}\ \CC$, we denote 
\begin{equation}
\label{eq:DefFNorm}
\normfclass{\Prob-\Prob'}{G} := \sup_{g \in \GClass }\abs{\Exp_{\Sample \sim \Prob} g(\Sample)-\Exp_{\Sample' \sim \Prob'}g(\Sample')}.
\end{equation}
\rev{
Specializing this to the particular class $\GClass = \Delta\LossClass(\HypClass)$, where
\begin{equation}\label{eq:DefDLossClass}
\Delta\LossClass(\HypClass) := \set{\loss(\cdot,\hyp)-\loss(\cdot,\hyp'): \hyp,\hyp' \in \HypClass},
\end{equation} we get a task-driven metric $\normdloss{\Prob-\Prob'}{} = \sup_{\hyp_0}  \divp_{\hyp_0}^{\HypClass}(\Prob\|\Prob')$, capturing differences in terms of excess risks. The first main result of \citepartone\ that we will use is the following theorem.}
\begin{theorem}[{\citeppartone{Theorem~\ref{P1-thm:LRIPsuff_excess}}}]\label{thm:LRIPsuff_excess}
\rev{Consider a loss class $\LossClass(\HypClass) := \{\loss(\cdot,\hyp): \hyp \in \HypClass\}$, a feature function $\SketchingOperator$, and a model set $\Model$ that is both $\LossClass(\HypClass)$-integrable and $\set{\SketchingOperator}$-integrable (cf Remark~\ref{rmk:integrability}). 
 Assume that the sketching operator $\SketchingOperatorProb$ associated to $\SketchingOperator$ satisfies the following  
 {\em lower restricted isometry property (LRIP)} 
\begin{equation}
\label{eq:lowerDRIP}
\normdloss{\mProb'-\mProb}{} \leq C_{\SketchingOperatorProb} \norm{\SketchingOperatorProb(\mProb')-\SketchingOperatorProb(\mProb)}_{2}+\eta\quad \forall \mProb,\mProb' \in \Model
\end{equation}
for some finite constants $C_{\SketchingOperatorProb}>0$  and $\eta \geq 0$.}

\rev{Consider any training collection $\dataset=\{\sample_{i}\}_{i=1}^{\nSamples} \in \SampleSpace^{\nSamples}$  and denote $\empProb := \tfrac{1}{n} \sum_{i=1}^{n} \delta_{\sample_{i}}$.
Define}
\rev{\begin{align}
\vy &:= \Sketch(\dataset) = \SketchingOperatorProb(\empProb),\label{eq:ThmSketch2}\\
  \estProb \in \Model  \text{ satisfying }
             \norm{\SketchingOperatorProb(\estProb)-\vy}_2 & \leq (1+\nu) \inf_{\mProb \in \Model}
             \norm{\SketchingOperatorProb(\mProb) -\vy}_2 + \nu', 
             & \text{ for some constants } \nu,\nu' \geq 0,
             \label{eq:ThmDecoder2}\\
  \hat{\hyp}  \text{ satisfying }
               \Risk(\estProb,\hat\hyp) & \leq \inf_{\hyp \in \HypClass} \Risk(\estProb,\hyp) + \eps',
               & \text{ for some constant } \eps' \geq 0.
                              \label{eq:ThmLearn2}
\end{align}
}
\rev{
Then, for any probability distribution $\Prob$ that is both $\LossClass(\HypClass)$-integrable and $\set{\SketchingOperator}$-integrable:}
\rev{
\begin{equation}\label{eq:MainBoundExcessRisk}
  \forall \hyp_{0} \in \HypClass : \drisk_{\hyp_{0}}(\Prob,\hat{\hyp}) 
  \leq \distIOPexgen_{\hyp_{0}}^{\HypClass}(\Prob,\Model) +
   (2+\nu) C_{\SketchingOperatorProb} \norm{\SketchingOperatorProb (\Prob)-\SketchingOperatorProb(\empProb)}_2 + \eta + C_{\SketchingOperatorProb}\nu' + \eps',
 \end{equation}
where
 \begin{equation}\label{eq:DefMDist2}
\distIOPexgen_{\hyp_{0}}^{\HypClass}(\Prob,\Model) := \inf_{\mProb \in \Model} \paren{\divp_{\hyp_{0}}^{\HypClass}(\Prob\|\mProb)
   + (2+\nu)C_{\SketchingOperatorProb} \norm{\SketchingOperatorProb(\Prob) - \SketchingOperatorProb(\mProb)}_2}.
   \end{equation}
   }
\end{theorem}
\rev{
The  bound~\eqref{eq:MainBoundExcessRisk} holds regardless of any distribution assumption on the training collection $\dataset$, and is valid for any $\hyp_{0} \in \HypClass$. The estimate is of course primarily of interest when $\dataset$ is drawn i.i.d. from $\Prob$ and with $\hyp_{0} = \hyp^{\star}$, in which case the left-hand side is the excess risk with respect to the optimum hypothesis, and $\norm{\SketchingOperatorProb(\Prob)-\SketchingOperatorProb(\empProb)}_{2}$ typically decays as $1/\sqrt{n}$.} \rev{As we will see, other choices of $\hyp_{0}$ will nevertheless turn out to be useful for the analysis of compressive $k$-means and compressive Gaussian Mixture Modeling, where we need to introduce a class $\HypClass$ satisfying a separation assumption and may be interested in the best hypothesis over a larger hypothesis class $\bar{\HypClass} \supseteq \HypClass$. }

\rev{The model sets $\Model$ considered for compressive $k$-means/$k$-medians consist of the probability distributions for which the optimum risk vanishes, i.e.  $\Risk(\Prob,\hyp^{\star}) = 0$. These model sets are precisely  chosen to ensure that the bias term \eqref{eq:DefMDist2} %
in the control~\eqref{eq:MainBoundExcessRisk} of the excess risk vanishes when the optimum risk itself vanishes, and we will provide more explicit bounds on this bias term. With these model sets, solving~\eqref{eq:ThmDecoder2} and~\eqref{eq:ThmLearn2} with $\nu=\nu'=\varepsilon'=0$ also precisely corresponds to minimizing the proxy~\eqref{eq:RiskProxyKMeans}. For compressive GMM, the considered model set $\Model$ consists of mixtures of Gaussians. Again, solving~\eqref{eq:ThmDecoder2} and~\eqref{eq:ThmLearn2} with $\nu=\nu'=\varepsilon'=0$ corresponds to minimizing the proxy~\eqref{eq:RiskProxyGMM}.}

\rev{The main technical contribution of this paper is to establish that the main assumption~\eqref{eq:lowerDRIP} of Theorem~\ref{thm:LRIPsuff_excess} holds for compressive clustering (resp. compressive GMM), using sketching operators based on random Fourier features with controlled sketch size $\nMeasures$. For this, we rely on the general approach described in \citeppartone{Section~\ref{P1-sec:ChoiceSketch}} relating random (Fourier) features and kernel mean embeddings of probability distributions. Another contribution is to provide more concrete estimates of the ``bias term'' \eqref{eq:DefMDist2}.
The results are first stated in the next sections, before giving the technical ingredients for their proof in the rest of the paper.}

\section{Compressive Clustering}\label{sec:clustering}
 We consider here two losses that measure clustering performance: the {\bf $\PCAdim$-means} and {\bf $\PCAdim$-medians} losses. 
Hypotheses are 
$k$-tuples $(c_{1},\ldots,c_{k})$ where the elements $c_{l} \in \RR^{\sampleDim}$ are the so-called centers of clusters. \rev{We speak of {\em unconstrained} $k$-means or $k$-medians if the cluster centers can
  be arbitrary points of $\RR^{\sampleDim\times k}$, and {\em constrained} otherwise (if the $k$-tuple of centers
  must belong to a specific subset of $\RR^{\sampleDim\times k}$, for instance if there 
  is a separation or a maximum norm constraint).}
The loss function for the clustering task is  
\begin{equation}
\label{eq:KMeansMediansLoss}
\loss(\sample,\hyp) := \min_{1 \leq l \leq k} \norm{\sample-c_{l}}_{2}^{p}
\end{equation}
with $p=2$ for $\PCAdim$-means (resp. $p=1$ for $\PCAdim$-medians) and 
$\Risk(\Prob,\hyp) = \Exp_{\Sample \sim \Prob}  \min_{1 \leq l \leq k} \norm{\Sample-c_{l}}_{2}^{p}$. 

\subsection{Main theoretical guarantees} \label{sec:main-kmeans}

\paragraph{\rev{Existence and properties of a minimizer of the risk.}}
For unconstrained $k$-means clustering, given any $\LossClass\rev{(\HypClass)}$-integrable probability distribution $\Prob$ there exists a global minimizer \rev{$\hyp^*_\Prob$} of the risk, see e.g~\cite[Lemma 8]{Pollard:1982by} which only assumes that the support of the distribution $\Prob$ contains at least $k$ elements. When this support has at most $k-1$ elements the existence of a global minimizer is trivial, and a zero risk is achieved. The existence of a global optimum was also proved in more general settings that include \rev{unconstrained} $k$-medians clustering, see, e.g. \cite[Theorem 1]{Graf:2007iw}.
For \rev{unconstrained} $k$-means, when the support of $\Prob$ contains at least $k$ elements, the global minimizer satisfies necessary conditions that were already identified in the work of \cite{Steinhaus_H_1956_j-bull-acad-polon-sci_division_cmp} and were formalized more recently in a generalized setting \cite[Proposition~1]{Graf:2007iw}. For a generic hypothesis $\hyp = (c_{1},\ldots,c_{k})$ denote 
\begin{equation}\label{eq:DefVoronoiCell}
V_{l}(\hyp):= \set{\sample \in \RR^{\sampleDim}: \norm{x-c_{l}}_{2} = \min_{j} \norm{\sample-c_{j}}_{2}},\qquad 1 \leq l \leq k
\end{equation}
the Voronoi cells of the $l$-th center. \rev{Let $(W_j(\hyp))_{1\leq h \leq k}$ be an arbitrary \emph{Voronoi partition} associated
  to these Voronoi cells, i.e., $W_j(\hyp)\subseteq V_j(\hyp)$ and $(W_j(\hyp))_{1\leq j\leq k}$ form a partition
  of $\RR^{\sampleDim}$ (in other words, this partition breaks the ``ties'' at the boundary of the Voronoi cells arbitrarily). For $x\in \SampleSpace = \RR^{\sampleDim}$, let $P_\hyp \sample = c_j$ if and only if $\sample \in W_{j}(h)$ (i.e. $P_\hyp$ maps $\sample$ to the closest cluster center with tie-breaking given by
  the Voronoi partition), and $P_h\Prob$ be the push-forward of a probability distribution $\Prob$ through
$P_h$. In other words, putting $\alpha_{l}(\Prob,\hyp) := \Prob(\Sample \in \rev{W}_{l}(\hyp)) = \Exp_{\Sample \sim \Prob} \mathbf{1}_{W_{l}(\hyp)}(\Sample)$ the probability that a sample  belongs to a \rev{piece of the Voronoi partition}, with $\mathbf{1}_{E}$ the indicator function of set $E$, we have $P_h\Prob = \sum_{i=1}^k \alpha_i(\Prob,\hyp) \delta_{c_i}$.}

\rev{For unconstrained $k$-means and $\Prob$ with a support containing at least $k$ points}, the $k$ optimal centers associated to $\hyp^{\star}_{\Prob}$ are pairwise distinct ($c_{i} \neq c_{j}$ for $i \neq j$) and satisfy the so-called centroid condition: for $1 \leq l \leq k$ we have $\alpha_{l}(\Prob,\hyp^{\star})>0$ and 
\begin{equation}\label{eq:CentroidCondition}
c_{l} = \Exp_{\Sample \sim \Prob}(\Sample | \Sample \in W_{l}(\hyp^{\star})) = \frac{\Exp_{\Sample \sim \Prob } 
\mathbf{1}_{W_{l}(\hyp^{\star})}(\Sample) \cdot  \Sample}{\alpha_{l}(\Prob,\hyp^{\star})}.
\end{equation}
Finally, the optimal centers are such that $\Prob(\Sample \in V_{i}(\hyp^{\star}) \cap V_{j}(\hyp^{\star})) = 0$ for each $i \neq j$, \rev{i.e., the distinction between Voronoi cells and partition pieces becomes essentially moot at the optimum}.

\paragraph{\rev{Choice of a model set.}}
\rev{Both $k$-means and $k$-medians are ``compression-type'' tasks \citeppartone{Definition~\ref{P1-def:comptypetask}}  (see reminders in Appendix~\ref{subsec:biasclustgmm}). Given a hypothesis $\hyp = (c_{1},\ldots,c_{k})$,} the distributions such that $\Risk(\Prob,\hyp)=0$ are precisely mixtures of $k$ Diracs, 
\begin{equation}\label{eq:DefKDiracs}
\rev{\ModelCT_{\hyp}} = \set{\sum_{l=1}^{k} \alpha_{l} \delta_{c_{l}}: %
\boldmath{\alpha} \in \Simplex_{\PCAdim-1}}
\end{equation}
where we recall that $\Simplex_{\PCAdim-1} := \set{\boldmath{\alpha} \in \RR^{\PCAdim}:\ \alpha_{l} \geq 0, \sum_{l=1}^{\PCAdim}\alpha_{l}=1}$ denotes the $(\PCAdim-1)$-dimensional simplex. \rev{Given a hypothesis class $\HypClass \subseteq (\RR^{\sampleDim})^{\PCAdim}$, following the approach outlined in \citeppartone{Section~\ref{P1-se:leastrestrictedmodel}}, we consider the model set $\ModelCT(\HypClass) := \cup_{\hyp \in \HypClass} \ModelCT_{\hyp}$.} 

\paragraph{\rev{Choice of a sketching function.}} Given that \rev{each} model set 
$\rev{\ModelCT(\HypClass)}$ consists of mixtures of Diracs, and by analogy with compressive sensing where random Fourier sensing yields RIP guarantees, \rev{it was proposed \citep{Keriven2016} to perform compressive clustering using random Fourier moments}. To establish our theoretical guarantees we rely on a reweighted version where the feature function $\SketchingOperator: \RR^\sampleDim \rightarrow \CC^\nMeasures$ is defined as:
  \begin{equation} \label{eq:weightkmeansmain}
  \SketchingOperator(\sample) := \frac{1}{\sqrt{\nMeasures}} \bigg[ \frac{e^{\jmath \freq_j^T \sample}}{w(\freq_{j})} \bigg]_{j=1,\ldots,\nMeasures}\quad
  \text{with}\ 
  w(\freq) :=
  1+\frac{s^{2}\norm{\freq}_{2}^2}{\sampleDim}
  \end{equation}
  where $s>0$ is a scale parameter and $\jmath$ is the imaginary unit.   The random frequencies $\freq_1,\ldots,\freq_\nMeasures$ in $\RR^\sampleDim$ are sampled independently from the distribution with density
    \begin{equation}\label{eq:samplingdensitymain}
      \freqdist(\freq) = \freqdist_{w,s}(\freq) \propto w^{2}(\freq) e^{ - \frac{s^{2}\norm{\freq}^2}{2}}\,,
      \end{equation}
This sketching operator is based on a reweighting of Random Fourier Features \citep{Rahimi2007}. The weights $w(\freq)$ are required for technical reasons (see general proof strategy in Section~\ref{sec:general_mixtures}) %
but may be an artefact of our proof technique.

\paragraph{\rev{Learning from a sketch by minimizing a proxy for the risk.}}
For any distribution in the model set, $\estProb = \sum_{l=1}^{k} \alpha_{l} \delta_{c_{l}} \in \rev{\ModelCT(\HypClass)}$, the optimum of minimization~\rev{\eqref{eq:ThmLearn2} (with $\eps'=0$)} is achieved with $\hat{\hyp} = (c_{1},\ldots,c_{k})$ hence, given a sketch vector $\vy$ and a class of hypotheses $\HypClass$, finding near minimizers of \eqref{eq:ThmDecoder2} and \eqref{eq:ThmLearn2} in \rev{Theorem~\ref{thm:LRIPsuff_excess}} corresponds to finding a (near) minimizer $\hat{\hyp} = (\hat{c}_{1},\ldots,\hat{c}_{k})\in \HypClass$ of the following proxy for the risk
\begin{equation}\label{eq:DefRiskProxyClustering}
\proxyRisk_{\mathtt{clust.}}(\vy,\hyp) := \min_{\alpha \in \Simplex_{k-1}} \norm{\vy-\sum_{i=1}^{k}\alpha_{i}\SketchingOperator(c_{i})}_{2}
\end{equation}
with $h = (c_{1},\ldots, c_{k})$
\rev{in a \emph{constrained} hypothesis class $\HypClass$ that we now describe.}

\paragraph{Separation constraint and approximate optimization.}
Because one can show (see Lemma~\ref{lem:SeparationNecessary}) that it is a necessary assumption to establish an LRIP, we optimize the proxy $\proxyRisk(\vy,\hyp)$ on a restricted hypothesis class
\begin{equation}\label{eq:DefSeparation}
\HypClass_{k,2\sep,R} := \set{(c_{l})_{l=1}^{\PCAdim}:\quad 
\min_{\rev{c_{l}\neq c_{l'}}} \norm{c_{l}-c_{l'}}_{2} \geq 2\sep,\quad \max_{l} \norm{c_{l}}_{2} \leq R}.
\end{equation}
\rev{There can be exact repetitions ($c_{l}=c_{l'}$ with $l \neq l'$), however distinct centers  $c_{l} \neq c_{l'}$ must be separated. The parameters $\sep$ and $R$ represent the resolution \rev{and extent} at which we seek clusters in the data through the minimization of $\proxyRisk(\vy,\hyp)$.
Observe that $\HypClass_{k,2\sep,R}$ is non-empty whenever $0 \leq \sep \leq R$. When $R/\sep$ is close to one, its elements may nevertheless be forced to have repeated entries. }
The following result is proved in Appendix~\ref{subsec:clustfinalproof}. Besides leveraging Theorem~\ref{thm:LRIPsuff_excess} the proof exploits a generic approach developed in Section~\ref{sec:general_mixtures} to establish the LRIP on mixture models using Theorem~\ref{thm:mainLRIP}, and its specialization to mixtures based on location families, which is developed in Section~\ref{sec:RBFKernel}.
\rev{We remind the reader that $P_{\hyp}$ is the projection onto the centroids of $\hyp$, as described at the beginning of this section.}
\begin{theorem}
  \label{thm:mainkmeansthm}
Consider $\SketchingOperator$ as in~\eqref{eq:weightkmeansmain} where the $\freq_{j}$ are drawn according to~\eqref{eq:samplingdensitymain} with scale $s>0$. Given an integer $k \geq 1$ define $\sep:= \rev{4 s\sqrt{\log(ek)}}$
and consider $R \geq \sep$. 
\begin{enumerate}
\item There is a universal constant $C>0$ such that, for any $\probLevel,\coveps \in (0,1)$, if the sketch size satisfies
\footnote{\label{ft:logek}The sketch sizes from the introduction and \citeppartone{Table~\ref{P1-tab:summary}} involve $\log(\cdot)$ factors which have correct order of magnitude when their argument is large, but vanish when their argument is one. Factors $\log(e \cdot)$ do not have this issue.}
  \begin{equation}
\label{eq:sksizekmeans}
  \nMeasures \geq C \coveps^{-2} 
    \left[k^{2}\sampleDim \cdot \left(1+\log k\sampleDim + \log \tfrac{R}{\sep} + \log \tfrac{1}{\coveps}\right) + k \log \tfrac{1}{\probLevel} \right] \cdot \log(ke) \min\left(\log(ek),\sampleDim\right),
  \end{equation}
 with probability at least $1-\probLevel$ on the draw of $(\freq_{j})_{j=1}^{m}$ the operator $\SketchingOperatorProb$ induced by $\SketchingOperator$ satisfies 
\begin{equation}\label{eq:mainRIPKernelClustering}
1-\delta \leq 
\frac{\norm{\SketchingOperatorProb(\mProb)-\SketchingOperatorProb(\mProb')}_{2}^{2}}{\norm{\mProb-\mProb'}_{\kappa}^{2}}
\leq 1+\delta, \qquad \forall \mProb,\mProb' \in \ModelCT(\HypClass_{k,2\sep,R}).
\end{equation}
\rev{with $\normkern{\cdot}$ the mean map discrepancy
\footnote{see \eqref{eq:DefMMD} in Section~\ref{sec:ingredientsLRIP} for details.}  (MMD) associated to kernel $\kernel(x,y) \propto \exp(-\norm{x-y}^{2}_{2}/s^{2})$.}
\item 
If~\eqref{eq:mainRIPKernelClustering} holds then:
 \begin{itemize}
 \item The function $\SketchingOperator$ is $L$-Lipschitz with $L=\sqrt{1+\delta}/s$ with respect to Euclidean norms.
 \item Consider any samples $\sample_{i} \in \RR^{\sampleDim}$,  $1 \leq i \leq \nSamples$ (represented by the empirical distribution $\empProb$) and any probability distribution $\Prob$ on $\RR^{\sampleDim}$ \rev{that is both $\LossClass(\HypClass)$-integrable and $\set{\SketchingOperator}$-integrable}.   \\
 Consider a constrained class $\HypClass \subseteq \HypClass_{k,2\sep,R} \subseteq \HypClassRef := (\RR^{\sampleDim})^{k}$ 
 and denote $\Risk(\cdot,\hyp)$ the risk associated to $k$-medians (resp. $k$-means), $\hyp^{\star} \in \arg\min_{\hyp \in \HypClassRef} \Risk(\Prob,\hyp)$, $\Prob^{\star} := P_{\hyp^{\star}}\Prob$.
 Consider 
 $\hat{h} \in \HypClass$ and $\nu,\nu'>0$ such that
\begin{equation}\label{eq:DefQuasiOptimumClusteringViaProxy}
\proxyRisk_{\mathtt{clust.}}(\vy,\hat{\hyp}) \leq (1+\nu) \inf_{\hyp \in \HypClass} \proxyRisk_{\mathtt{clust.}}(\vy,\hyp) + \nu'.
\end{equation}
with the proxy $\proxyRisk_{\mathtt{clust.}}(\vy,\cdot)$ defined in~\eqref{eq:DefRiskProxyClustering} and the sketch vector $\vy := \frac{1}{n} \sum_{i=1}^{\nSamples}\SketchingOperator(\sample_i) = \SketchingOperatorProb(\empProb)$.

The excess risk of $\hat{\hyp}$ with respect to $\hyp^{\star}$ satisfies
\begin{align}\label{eq:MainBoundClusteringThm}
  \drisk_{\hyp^{\star}}(\Prob,\hat{\hyp})
    \leq & 
    (2+\nu)C_{\SketchingOperatorProb}
\norm{\SketchingOperatorProb(\Prob)-\SketchingOperatorProb(\empProb)}_{2} 
+(2+\nu)C_{\SketchingOperatorProb}
\norm{\SketchingOperatorProb(\Prob)-\SketchingOperatorProb(\Prob^{\star})}_{2}
+ d(\Prob^{\star},\HypClass)
+ C_{\SketchingOperatorProb}\nu'
\end{align}
where $C_{\SketchingOperatorProb} \leq 56\sqrt{k/(1-\delta)} (2R)^{p}$
(with $p=1$ for $k$-medians, $p=2$ for $k$-means) and
\begin{equation}\label{eq:BiasKMeansMainThm}
d(\Prob^{\star},\HypClass)
  := \inf_{\mProb \in \ModelCT(\HypClass)} 
\left\{
\sup_{\hyp \in \HypClass} \left(\Risk(\Prob^{\star},\hyp)-\Risk(\mProb,\hyp)\right)
+ (2+\nu)C_{\SketchingOperatorProb}\norm{\SketchingOperatorProb(\Prob^{\star})-\SketchingOperatorProb(\mProb)}_{2}\right\}.
\end{equation}
 \end{itemize}
 \end{enumerate}
\end{theorem}
The first term in~\eqref{eq:MainBoundClusteringThm} is a measure of statistical error that can be easily controlled since $\norm{\SketchingOperator(\sample)}_2 \leq 1$ by construction~\eqref{eq:weightkmeansmain}. By the vectorial Hoeffding's inequality \citep{Pinelis92}, if $x_{i}, 1 \leq i \leq n$ are drawn i.i.d. with respect to $\Prob$ then with high probability it holds that $\norm{\SketchingOperatorProb(\Prob)-\SketchingOperatorProb(\empProb)}_{2}\lesssim 1/ \sqrt{n}$.
The constants $\nu$ and $\nu'$ measure the numerical error in optimizing $\proxyRisk_{\mathtt{clust.}}(\vy,\hyp)$ over $\HypClass$.

The second term measures how close $\Prob$ is to $\Prob^{\star} = P_{\hyp^{\star}}\Prob$, i.e., how ``clusterable'' is $\Prob$. By \citeppartone{Lemma~\ref{P1-lem:LemmaBiasTermBis}} and  the Lipschitz property of $\SketchingOperator$ this is bounded by $L \cdot \Risk^{1/p}(\Prob,\hyp^{\star})$. This bound seems however pessimistic and we leave to future work a possible sharpening for certain settings.

The third term $d(\Prob^{\star},\HypClass)$ measures of how far $\hyp^{\star}$ (weighted by the coefficients $\alpha_{i}$) is from belonging to $\HypClass$.   The Lipschitz property of $\SketchingOperator$ yields more explicit bounds that may deserve to be further sharpened.
\begin{lemma}\label{lem:bounddisttosepclust}
With the notations of Theorem~\ref{thm:mainkmeansthm} (recall in particular that we assume $\HypClass \subseteq \HypClass_{k,2\sep,R}$), consider a mixture of $k$ Diracs $\Prob^{\star} := \sum_{i=1}^{k}\alpha_{i} \delta_{c_{i}}$ with $c_{i} \in \RR^{\sampleDim}$ and $\alpha \in \Simplex_{k-1}$.  If~\eqref{eq:mainRIPKernelClustering} holds then for the $k$-medians clustering task we have
\begin{align*}
d_{\mathtt{k-medians}}(\Prob^{\star},\HypClass)
&\leq
C \cdot \inf_{\hyp \in \HypClass} \Risk_{\mathtt{k-medians}}(\Prob^{\star},\hyp),
                                                  \intertext{while for the $k$-means clustering task we have}
d_{\mathtt{k-means}}(\Prob^{\star},\HypClass)
& \leq
\inf_{\hyp \in \HypClass} \Big\{
\Risk_{\mathtt{k-means}}(\Prob^{\star},\hyp) + 4 C R \cdot \Risk_{\mathtt{k-medians}}(\Prob^{\star},\hyp)
\Big\},
\end{align*}
with $C :=  1 + (2+\nu) \rev{500} \sqrt{k\log(ek)(1+\delta)/(1-\delta)} \frac{R}{\sep}$.
\end{lemma}
The proof is in Appendix~\ref{pf:bounddisttosepclust}. By Jensen's inequality we have $\Risk_{k-\mathtt{medians}}(\Prob,\hyp) \leq \sqrt{\Risk_{k-\mathtt{means}}(\Prob,\hyp)}$ hence $d_{\mathtt{k-means}}(\Prob^{\star},\HypClass)$ can also be bounded in terms of the optimum $k$-means risk over $\HypClass$.

\begin{remark}
Just as Theorem~\ref{thm:mainkmeansthm}, the above lemma is valid for an arbitrary class $\HypClass \subseteq \HypClass_{k,2\sep,R}$, hence they can be exploited for instance when the $d \times k$ matrix of centroids $\mathbf{C} = [c_{1},\ldots,c_{k}]$ is constrained to satisfy certain structural constraints (besides $2\sep$-separation and $R$-boundedness). For example, $\mathbf{C}$ could be required to be a product of sparse matrices to enable accelerated clustering \citep{Giffon:2019wq}. This is however left to future work.
\end{remark}

\begin{remark}
  Even if $\Prob^{\star} \in \HypClass$ and the third term  in~\eqref{eq:MainBoundClusteringThm} vanishes, the second term
  is in general positive and does not vanish with $n\rightarrow \infty$, except if $\Risk(\Prob,\hyp^\star)=0$, i.e. the source distribution $\Prob$ is itself exactly a mixture of $k$ Diracs.
  This comes from the fact that the proposed method finds an optimal clustering for
  an implicitly reconstructed distribution which is of this form:
  if the source distribution is not of this form, this introduces an inherent bias.
  Hence, the result above does not
  recover consistency or convergence rates for the clustering risk available under broad conditions
  when using the full data (see Introduction for references on this topic). An interesting direction left
  for future work is to investigate if consistency could be established when considering a larger
  model set such as mixtures of $r$ Diracs for $r>k$, and $r$ depending on $n$, similarly
  to considerations on sketched PCA \citeppartone{discussion following Theorem~4.1}
\end{remark}

\subsection{Role of the separation assumption}\label{sec:mainseparationdiscussion}
It is natural to wonder whether the separation assumption is an artefact of our proof technique. The following result shows that it is in fact a necessary assumption to establish an LRIP for compressive $k$-means and compressive $k$-medians clustering, \rev{for any smooth sketching operator (in particular one
based on Fourier features).} The proof is in Appendix~\ref{sec:separationnecessaryDiracs}.
\begin{lemma}\label{lem:SeparationNecessary}
  Consider a loss associated to a clustering task: $\loss(\sample,\hyp) := \min_l\norm{\sample-c_l}_{2}^p$ where $\hyp = (c_{1},\ldots,c_{k})$, $k \geq 2$, with $p=1$ ($k$-medians) or $p=2$ ($k$-means). \rev{Let $\SketchingOperator$ be a sketching function
of class $\mathcal{C}^2$ and $\SketchingOperatorProb$ be the associated sketching operator.}%
  There is a constant $c_{\SketchingOperator}>0$ such that for any $R >0$ and any $0 < \sep \leq R$, with $\HypClass = \HypClass_{k,\sep,R}$ we have
\[
\sup_{\mProb,\mProb' \in \ModelCT(\HypClass)} \frac{\norm{\mProb-\mProb'}_{\DLossClass(\HypClass)}}{\norm{\SketchingOperatorProb(\mProb)-\SketchingOperatorProb(\mProb')}_{2}} \geq c_{\SketchingOperator} R^{p-1} /\sep.
\]
In particular, the LRIP~\eqref{eq:lowerDRIP} cannot hold with $\eta=0$ and a finite constant $C$ on $\HypClass_{k,0,R}$.
\end{lemma}
Although this shows that the separation $\sep$ is important in the derivation of learning guarantees,
its role in the final bounds is less stringent than it may appear at first sight, for several reasons.  

First, in the most favorable case, the globally optimal hypothesis $\hyp^{\star}$ indeed belongs to the constrained class $\HypClass$, in which case $d(\Prob^{\star},\HypClass) = 0$ since $\Prob^{\star} := P_{\hyp^{\star}}\Prob = \sum_{i=1}^{k} \alpha_{i} \delta_{c^{\star}_{i}}$ is a mixture of Diracs. In particular, for $\HypClass = \HypClass_{k,2\sep,R}$, if we have prior information on the minimum separation $\sep^{\star} := \min_{i \neq j} \norm{c^{\star}_{i}-c^{\star}_{j}}_{2}$ of $\hyp^{\star} = (c^{\star}_{1},\ldots,c^{\star}_{k})$ and on $R^{\star} := \max_{l} \norm{c^{\star}_{l}}_{2}$ then it is enough to choose the scale parameter $s \leq %
\rev{\sep^{\star}(4\sqrt{\log (ek)})^{-1}}$ and the sketch size large enough (with logarithmic dependency on $R^{\star}/\sep^{\star}$) to ensure that $d(\Prob^{\star},\HypClass) = 0$. Note that this holds with the sample space $\SampleSpace = \RR^{\sampleDim}$, i.e., we only restrict the optimization of the proxy $\proxyRisk_{\mathtt{clust.}}(\vy,\hyp)$, \emph{not the data} to centers in the Euclidean ball of radius $R \geq R^{\star}$, $\Ball_{\RR^\sampleDim,\norm{\cdot}_2}(0,R)$. 

Second, as we now show with a  focus on $\HypClass = \HypClass_{k,2\sep,R}$, even when $\hyp^{\star}$ is not separated the term $d(\Prob^{\star},\HypClass)$ can remain under control. This comes from a combination of two main facts: 1) the risk for $k$-means and $k$-medians varies gently with respect to a natural distance between hypotheses; and 2) the distance between an arbitrary $\hyp = (c_{1},\ldots,c_{k})$ and the closest separated one is controlled. 
To formalize these facts we introduce the following notation:
\begin{definition}\label{def:disttoseparated}
Given $\vc = (c_{1},\ldots,c_{k})$ and $\vc' = (c'_{1},\ldots,c'_{k})$ two $k$-tuples with $c_{i},c'_{j} \in \RR^{\sampleDim}$, denote 
\begin{align}\label{eq:disttoseparated}
d(c_{i}\| \vc') & :=  \min_{1 \leq j \leq k} \norm{c_{i}-c'_{j}}_{2}; & 
d(\vc\| \vc') & :=   \max_{1 \leq i \leq k} d(c_{i}\|\vc').
\end{align}
Since $d(\cdot\|\cdot)$ is not symmetric we define $d(\vc,\vc') := \max(d(\vc\|\vc'),d(\vc'\|\vc))$.
\end{definition}
\begin{definition}\label{def:Isep}
For $\vc=(c_1,\ldots,c_k)$ with $c_{i} \in \RR^d$, the (index) set of ``$\sep$-isolated'' centroids of $\vc$ (ignoring repetitions) is denoted $I_\sep(\vc) :=\set{ i: 1 \leq i \leq k;
  \forall j\neq i: c_j=c_i \text{ or } \norm{c_i-c_j}_{2} \geq \sep}$. 
\end{definition}
  
The following lemmas are proved in Appendix~\ref{sec:link_risk_separation}. 
\begin{lemma}\label{lem:ClustRiskLipschitzWrtDist} 
  Consider $k \geq 2$, $\hyp,\hyp' \in \HypClass = (\RR^{\sampleDim})^{k}$, and $\Prob$ with integrable $k$-means (resp. $k$-medians) loss.  With $p=2$ for $k$-means (resp. $p=1$ for $k$-medians) we have
\(
\abs{\Risk(\Prob,\hyp)^{1/p}-\Risk(\Prob,\hyp')^{1/p}} \leq d(\hyp,\hyp').
\)
\end{lemma}

Since $\Prob^{\star} := P_{\hyp^{\star}} \Prob = \sum_{i=1}^{k}\alpha_{i} \delta_{c^{\star}_{i}}$ satisfies 
$\Risk(\Prob^{\star},\hyp^{\star}) = 0$, by Lemma~\ref{lem:ClustRiskLipschitzWrtDist} we get $\Risk^{1/p}(\Prob^{\star},\hyp) \leq d(\hyp^{\star},\hyp)$ for each $\hyp \in \HypClass$, with $p=1$ for $k$-medians and $p=2$ for $k$-means. Hence, with the constant $C$ from Lemma~\ref{lem:bounddisttosepclust}, we have when $\HypClass \subseteq \HypClass_{k,2\sep,R}$
\begin{align*}
d_{\mathtt{k-medians}}(\Prob^{\star},\HypClass)
&\leq 
C \inf_{\hyp \in \HypClass} d(\hyp^{\star},\hyp);\\
d_{\mathtt{k-means}}(\Prob^{\star},\HypClass)
&\leq 
\inf_{\hyp \in \HypClass}
\left\{
 d^{2}(\hyp^{\star},\hyp)+4CR d(\hyp^{\star},\hyp)
 \right\}.
\end{align*}
\begin{lemma}\label{lem:DistToSeparatedSet}
  Given $\rev{\sep \geq 0}$ and $\vc=(c_1,\ldots,c_k) \in \HypClass_{k,0,R}$, there exists $\vc'\in \HypClass_{k,\sep,R}$ such that \rev{$d(\vc,\vc') \rev{<} \sep$, and such that all $\sep$-isolated centroids of $\vc$, $\set{c_i, i \in I_\sep(\vc)}$, are centroids of $\vc'$ ($\sep$-isolated centroids of $\vc$ appearing
    multiple times may appear only once in $\vc'$.)}
\end{lemma}
Denoting $R^{\star} := \max_{i} \norm{c^{\star}_{i}}_{2}$, by Lemma~\ref{lem:DistToSeparatedSet}, 
there exists $\hyp_{\sep} \in \HypClass_{k,2\sep,R^{\star}}$ such that $d(\hyp^{\star},\hyp_{\sep}) \leq 2\sep$. 
A consequence is that if $\HypClass_{k,2\sep,R^{\star}} \subseteq \HypClass \subseteq \HypClass_{k,2\sep,R}$, then 
\begin{align*}
d_{\mathtt{k-medians}}(\Prob^{\star},\HypClass)
&\leq 
C d(\hyp^{\star},\HypClass_{k,2\sep,R^{\star}}) \leq 2 C\sep;\\
d_{\mathtt{k-means}}(\Prob^{\star},\HypClass)
&\leq
d^{2}(\hyp^{\star},\HypClass_{k,2\sep,R^{\star}})
+4CR d(\hyp^{\star},\HypClass_{k,2\sep,R^{\star}})
\leq 
4\sep^{2}+8CR\sep.
\end{align*}
Of course these are worst-case bounds.  In particular, observe (using the definition of
the constant $C$ in Lemma~\ref{lem:bounddisttosepclust}) that for large $k$ we have $C\sep \gg R$ and $CR\sep \gg R^{2}$. Better bounds can be obtained by taking more finely into account the possible ``near'' separation of $\hyp^{\star}$ as well as the weights $\alpha_{i}$ of the non $\sep$-isolated centroids.

\begin{lemma}\label{lem:interpretableclusteringriskbound}
With the notations and assumptions of Theorem~\ref{thm:mainkmeansthm}, let  $\Prob^{\star} = \sum_{i=1}^{k}\alpha_{i} \delta_{c_{i}}$ be an (arbitrary) $k$-mixture of Diracs and
  $\hyp^{\star}:=(c_1,\ldots,c_k)$ its centroids. Denote $R^{\star} := \max_{i} \norm{c_{i}}_{2}$, and
  $\bar{W}(\Prob^{\star},\sep):= \sum_{i \not \in I_\sep(\hyp^{\star})} \alpha_i \in [0,1]$ the weight of non-$\sep$-isolated centroids. 
  If $\HypClass_{k,2\sep,R^{\star}} \subseteq \HypClass \subseteq \HypClass_{k,2\sep,R}$, then for the $k$-medians or $k$-means risk we have %
  \begin{align}
d_{\mathtt{k-medians}}(\Prob^{\star},\HypClass)
&\leq 
C     \min\left\{ \bar{W}(\Prob^{\star},4\sep) \cdot d(\hyp^{\star},\HypClass_{k,2\sep,R^{\star}}),\ \bar{W}(\Prob^{\star},2\sep)\cdot 2\sep\right\} ;
\label{eq:interpretableclusteringriskbound1}\\
d_{\mathtt{k-means}}(\Prob^{\star},\HypClass)
&\leq
    \min\left\{\bar{W}(\Prob^{\star},4\sep) \cdot \left(d^2(\hyp^{\star},\HypClass_{k,2\sep,R^{\star}})
    +4CR^{\star} d^2(\hyp^{\star},\HypClass_{k,2\sep,R^{\star}})\right) , \right.\notag\\
    &\left. \qquad \bar{W}(\Prob^{\star},2\sep)\cdot \left(4\sep^2+8CR^{\star} \sep\right)\right\}.
    \label{eq:interpretableclusteringriskbound2}
\end{align}
  \end{lemma}
The form of~\eqref{eq:interpretableclusteringriskbound1}-\eqref{eq:interpretableclusteringriskbound2} %
  illustrates that, when restricting possible clustering hypotheses $\hyp$ to be separated, the best restricted risk $\Risk(\Prob^\star,\hyp)$ can be smaller if the unrestricted optimal clustering $\Prob^\star$ has
  centroids globally well-approximated by a set of separated centroids, or if $\Prob^\star$  puts
   large weight on isolated centroids, with both effects possibly compounding.

\subsection{Learning algorithm ?}

\label{sec:diracs_illustration}

For compressive clustering, learning in the sketched domain means \rev{addressing the minimization of the proxy $\proxyRisk_{\mathtt{clust.}}(\vy,\hyp)$ over $\hyp \in \HypClass$}, %
which is analogous to the classical finite-dimensional least squares problem under a sparsity constraint. The latter is NP-hard, yet, under RIP conditions, provably good and computationally efficient algorithms (either greedy or based on convex relaxations) have been derived \citep{FouRau13}. \rev{Remark that the classic (non-compressed) $k$-means problem by minimization of the empirical risk is also known to be NP-hard \citep{Garey:1982ks,Aloise2009} and that guarantees for approaches such as K-means++ \citep{Arthur2007} are only in expectation and with a logarithmic sub-optimality factor.}

It was shown practically in \citep{Keriven2016} that a heuristic based on orthogonal matching pursuit (which neglects the separation and boundedness constraint associated to the class $\HypClass_{k,2\sep,R}$) is empirically able to recover sums of Diracs from sketches of the appropriate size. It must be noted that recovering sums of Diracs from Fourier observations has been studied in the case of regular low frequency  measurements. In this problem, called super-resolution, it was shown that a convex proxy (convexity in the space of distributions using total variation regularization) for the non-convex optimization \rev{of the proxy $\proxyRisk_{\mathtt{clust.}}(\vy,\hyp)$} is able to recover sufficiently separated Diracs~\citep{Candes_2013,Castro_2015,Duval_2015}. 
In dimension one, an algorithmic approach to address the resulting convex optimization problem relies on semi-definite relaxation of dual optimization followed by root finding. Extension to dimension $d$ and weighted random Fourier measurements is not straightforward. Frank-Wolfe algorithms \citep{Bredies:2013wf} are more flexible for higher dimensions, and a promising direction for future research around practical sketched learning.

\subsection{Improved sketch size guarantees?}\label{sec:discussionKM}
Although Theorem~\ref{thm:mainkmeansthm} only provides guarantees when $\nMeasures$ \rev{is of the order of} $\PCAdim^2\sampleDim$ (up to logarithmic factors), the observed empirical phase transition pattern  \citep{Keriven2016} hints that $\nMeasures$ \rev{of the order of} $\PCAdim\sampleDim$ is in fact sufficient. This is intuitively what one would expect the ``dimensionality'' of the problem to be, since this is the number of parameters of the model $\rev{\ModelCT(\HypClass)}$. In fact, as the parameters live in the cartesian product of $k$ balls of radius $R$ in $\RR^{\sampleDim}$ and the ``resolution'' associated to the separation assumption is $\sep$, a naive approach to address the problem would consist in discretizing the parameter space into $N= \order{(R/\sep)^{\sampleDim}}$ bins. Standard intuition from compressive sensing would suggest a sufficient number of measures  $\nMeasures$ \rev{of the order of} $k \log N = \order{k\sampleDim \log \frac{R}{\sep}}$.
We leave a possible refinement of our analysis, trying to capture the empirically observed phase transition, for future work. 

\section{Guarantees for Compressive Gaussian Mixture Modeling}\label{sec:gmm}

We consider Gaussian Mixture Modeling on the sample space $\SampleSpace = \RR^d$, with $k$ Gaussian components with \emph{fixed, known invertible covariance} matrix $\covar \in \RR^d$ . Denoting $\Prob_c=\mathcal{N}(c,\covar)$, an hypothesis $\hyp=(c_1,...,c_k,\alpha_1,...,\alpha_k)$ contains the means and weights of the components of a Gaussian Mixture Model (GMM) denoted $\Prob_\hyp=\sum_{l=1}^k\alpha_l\Prob_{c_l}$, with $c_l \in \RR^d$ and $\boldmath{\alpha} \in \Simplex_{k-1}$. The loss function for a density fitting problem is the negative log-likelihood:
$\ell(\sample,\hyp)=-\log \Prob_\hyp(\sample)$,
and correspondingly the risk is $\Risk_{\mathtt{GMM}}(\Prob,\hyp) = \Exp_{\Sample \sim \Prob} (-\log \Prob_\hyp(\Sample))$.
\rev{When $\Prob$ has a density with respect to the Lebesgue measure, and if this density admits a well-defined differential entropy,} 
\begin{equation}\label{eq:DefEntropy}
\Entropy{\Prob} := \Exp_{\Sample \sim \Prob} -\log \Prob(\Sample),
\end{equation}
the risk can be written $\Risk_{\mathtt{GMM}}(\Prob,\hyp) = \KLdiv{\Prob}{\Prob_{\hyp}}+\Entropy{\Prob}$ with $\KLdiv{\Prob}{\Prob'} := \Exp_{\Sample \sim \Prob} \log \frac{\Prob(\Sample)}{\Prob'(\Sample)}$ 
the Kullback-Leibler divergence, see, e.g., \citep[Chapter 9]{CT91}.
\rev{For any distribution $\Prob$ with integrable GMM loss class, there exists an unconstrained global GMM risk minimizer $\hyp^\star \in \RR^{dk} \times \Simplex_{k-1} $ of $\Risk_{\mathtt{GMM}}(\Prob,\hyp)$
(see Section~\ref{sec:prexistencegmmmin} for a proof).}

\paragraph{Model set $\rev{\ModelML(\HypClass)}$ and best hypothesis for $\Prob \in \rev{\ModelML(\HypClass)}$.} A natural model set for density fitting maximum log likelihood is precisely the model of all parametric densities:
\begin{equation}
\label{eq:ModelGMM}
\rev{\ModelML(\HypClass)} :=\set{\Prob_h: h\in\HypClass}.
\end{equation}
A fundamental property of the Kullback-Leibler divergence is that $\KLdiv{\Prob}{\Prob'} \geq 0$ with equality if, and only if $\Prob = \Prob'$. Hence, for any distribution $\estProb = \Prob_{\hyp_{0}}$ in the model set $\rev{\ModelML(\HypClass)}$, the optimum of minimization~\rev{\eqref{eq:ThmLearn2} (with $\eps'=0$)}  is $\hat{\hyp} = \hyp_{0}$ as it corresponds (up to an offset \rev{independent of $\hyp$}) to minimizing $\KLdiv{\estProb}{\Prob_{h}}$.

\paragraph{Separation assumption.} Similar to the compressive clustering framework case of Section~\ref{sec:clustering}, we enforce a bounded domain and a minimum separation between the means of the components of a GMM. We denote
\begin{equation}
\label{eq:HypClassGMM}
\HypClass_{k,\sep,R}=\set{(c_1,...,c_k,\alpha_1,...,\alpha_k): c_l \in \RR^\sampleDim,~\normmah{c_l}{\covar} \leq R,~\min_{\rev{c_l \neq c_{l'}}}\normmah{c_l-c_{l'}}{\covar}\geq \sep,~(\alpha_{1},\ldots,\alpha_{k}) \in \Simplex_{k-1}},
\end{equation}
where 
\begin{equation}\label{eq:DefMahalanobisNorm}
\normmah{c}{\covar} := \sqrt{c^T\covar^{-1}c}
\end{equation} 
is the Mahalanobis norm associated to the known covariance $\covar$. 

\paragraph{Choice of feature function: random Fourier features.} Compressive learning of GMMs with random Fourier features has been recently studied \citep{bourrier:hal-00799896,Keriven2015}. Unlike compressive clustering we do not need to define a reweighted version of the Fourier features, and we directly sample $\nMeasures$ frequencies $\freq_1,\ldots,\freq_\nMeasures$ in $\RR^d$ i.i.d from the distribution with density
\begin{equation}\label{eq:freqdistGMM}
\freqdist 
= \freqdist_s 
=\mathcal{N}(0,s^{-2}\covar^{-1}),
\end{equation}
with scale parameter $s>0$. Define the associated feature function $\SketchingOperator: \RR^\sampleDim \rightarrow \CC^\nMeasures$:
  \begin{equation}\label{eq:fourierfeatureGMM}
  \SketchingOperator(\sample) := \frac{1}{\sqrt{\nMeasures}} \brac{e^{\jmath \freq_j^T \sample} }_{j=1,\ldots,\nMeasures}.
  \end{equation}

\paragraph{Learning from a sketch by minimizing a proxy for the risk.} Given sample points $\sample_1,\ldots,\sample_\nSamples$ in $\RR^\sampleDim$, a sketch $\vy$ can be computed as in~\eqref{eq:GenericSketching}, i.e., as a sampling of the conjugate of the empirical characteristic function \citep{Feuerverger:1977bc} of the distribution $\Prob$ of $\Sample$. The characteristic function of a Gaussian $\Prob_c = \mathcal{N}(c,\covar)$ has a closed form expression hence, with the operator $\SketchingOperatorProb$ defined in \eqref{eq:SketchingOperatorProbDef}, we have
\[
\Psi(c) := \Exp_{X \sim \mathcal{N}(c,\covar)} \SketchingOperator(X) = \SketchingOperatorProb(\Prob_c) = \frac{1}{\sqrt{\nMeasures}}\brac{e^{\jmath\freq_j^Tc}e^{-\frac12 \freq_j^T\covar\freq_j}}_{j=1,\ldots,\nMeasures}\,.
\]
Then, given a sketch vector $\vy$ and a hypothesis class $\HypClass$, finding near minimizers of \eqref{eq:ThmDecoder2} and \eqref{eq:ThmLearn2} in \rev{Theorem~\ref{thm:LRIPsuff_excess}} corresponds to finding a (near) minimizer $\hat{\hyp} = (\hat{c}_{1},\ldots,\hat{c}_{k},\hat{\alpha}_{1},\ldots,\hat{\alpha}_{k})\in \HypClass$ of the proxy
\begin{equation}\label{eq:DefRiskProxyGMM}
\proxyRisk_{\mathtt{GMM}}(\vy,\hyp) :=  \norm{\vy-\sum_{i=1}^{k}\alpha_{i} \Psi(c_{i})}_{2}.
\end{equation}
The following guarantees are proved in Appendix~\ref{subsec:clustfinalproof} jointly with Theorem~\ref{thm:mainkmeansthm}. 
\begin{theorem}
  \label{thm:maingmmthm}
  Consider $\SketchingOperator$ as in~\eqref{eq:fourierfeatureGMM} where the $\freq_{j}$ are drawn according to~\eqref{eq:freqdistGMM} with scale $s\rev{\geq 1}$. Given an integer $k \geq 1$ define $\sep:= \rev{4}\sqrt{(2+s^{2})\rev{\log(ek)}}$ %
  and consider $R \geq \sep$.   
\begin{enumerate}
\item  There is a universal constant $C>0$ such that, for any $\probLevel,\coveps \in (0,1)$, when the sketch size satisfies 
  \begin{align}
\label{eq:sksizeGMM}
\rev{m \geq C \coveps^{-2} 
     \cdot k \cdot} & \rev{\left[kd \cdot \left(\frac{d}{s^{2}}+ 1+\log(kRs) + \log(1/\coveps)\right) + \log(1/\zeta) \right]} \notag\\
     & \rev{\cdot  \min(\log^{2}(ek),s^2\log(ek)) \cdot (1+2/s^2)^{d/2}.}
  \end{align}
   with probability at least $1-\probLevel$ on the draw of $(\freq_{j})_{j=1}^{m}$ the operator $\SketchingOperatorProb$ induced by $\SketchingOperator$ satisfies 
\begin{equation}\label{eq:mainRIPKernelGMM}
1-\delta \leq 
\frac{\norm{\SketchingOperatorProb(\mProb)-\SketchingOperatorProb(\mProb')}_{2}^{2}}{\norm{\mProb-\mProb'}_{\kappa}^{2}}
\leq 1+\delta, \qquad \forall \mProb,\mProb' \in \ModelML(\HypClass_{k,2\sep,R}).
\end{equation}
\item
Consider any samples $\sample_{i} \in \RR^{\sampleDim}$,  $1 \leq i \leq \nSamples$ (represented by the empirical distribution $\empProb$) and any probability distribution $\Prob$ on $\RR^{\sampleDim}$ \rev{that is both $\LossClass(\HypClass)$-integrable and $\set{\SketchingOperator}$-integrable}.  \\
 Consider a constrained class $\HypClass \subseteq \HypClass_{k,2\sep,R} \subseteq \HypClassRef := (\RR^{\sampleDim})^{k} \times \Simplex_{k-1}$. Denote $\hyp^{\star} \in \arg\min_{\hyp \in \HypClassRef} \Risk_{\mathtt{GMM}}(\Prob,\hyp)$, $\Prob^{\star} := \Prob_{\hyp^{\star}}$, %
 and consider 
 $\hat{h} \in \HypClass$ and $\nu,\nu'>0$ such that
\begin{equation}\label{eq:DefQuasiOptimumGMMViaProxy}
\proxyRisk_{\mathtt{GMM}}(\vy,\hat{\hyp}) \leq (1+\nu) \inf_{\hyp \in \HypClass} \proxyRisk_{\mathtt{GMM}}(\vy,\hyp) + \nu'.
\end{equation}
with the proxy $\proxyRisk_{\mathtt{GMM}}(\vy,\cdot)$ defined in~\eqref{eq:DefRiskProxyGMM} and the sketch vector $\vy := \frac{1}{n} \sum_{i=1}^{\nSamples}\SketchingOperator(\sample_i) = \SketchingOperatorProb(\empProb)$.

If~\eqref{eq:mainRIPKernelGMM} holds then the excess risk of $\hat{\hyp}$ with respect to $\hyp^{\star}$ satisfies
\begin{align}\label{eq:MainBoundGMMThm}
  \KLdiv{\Prob}{\Prob_{\hat{\hyp}}}-\KLdiv{\Prob}{\Prob_{\hyp^{\star}}}
  =
  \drisk_{\hyp^{\star}}(\Prob,\hat{\hyp})
    \leq & 
    (2+\nu)C_{\SketchingOperatorProb}
\norm{\SketchingOperatorProb(\Prob)-\SketchingOperatorProb(\empProb)}_{2} 
+(2+\nu)C_{\SketchingOperatorProb}
\norm{\SketchingOperatorProb(\Prob)-\SketchingOperatorProb(\Prob^{\star})}_{2}\notag\\
& + \divp_{\hyp^{\star}}^{\HypClassRef}(\Prob\|\Prob^{\star}) + d(\Prob^{\star},\HypClass)
+ C_{\SketchingOperatorProb}\nu'
\end{align}
where $C_{\SketchingOperatorProb} \leq 46\sqrt{k/(1-\delta)} R^{2} \left(1+2/s^2\right)^{\sampleDim/4}$ and
\begin{equation}\label{eq:BiasGMMMainThm}
d(\Prob^{\star},\HypClass)
  := \inf_{\mProb \in \ModelML(\HypClass)} 
\left\{
\sup_{\hyp \in \HypClass} 
\left(\KLdiv{\Prob^{\star}}{\Prob_{\hyp}}-\KLdiv{\mProb}{\Prob_{\hyp}}\right)
+ (2+\nu)C_{\SketchingOperatorProb}\norm{\SketchingOperatorProb(\Prob^{\star})-\SketchingOperatorProb(\mProb)}_{2}\right\}.
\end{equation}

\end{enumerate}
\end{theorem}
\begin{remark}
Note that this holds with the sample space $\SampleSpace = \RR^{\sampleDim}$, i.e., we only restrict the means of the GMM, \emph{not the data},  to the ball of radius $R$, $\Ball_{\RR^\sampleDim,\normmah{\cdot}{\covar}}(0,R)$.
\end{remark}
The first term in the bound~\eqref{eq:MainBoundGMMThm} is a statistical error term that is easy to control since \eqref{eq:fourierfeatureGMM} implies $\norm{\SketchingOperator(\sample)}_2 = 1$ for each $\sample$. By  the vectorial Hoeffding's inequality \citep{Pinelis92}, for i.i.d. samples $\sample_{i}$ drawn according to $\Prob$, with high probability w.r.t. data sampling it holds that $C_{\SketchingOperatorProb}\norm{\SketchingOperatorProb(\Prob)-\SketchingOperatorProb(\empProb)}_{2}$ is of the order of at most $\left(1+2/s^2\right)^{\sampleDim/4}\sqrt{k}R^{2}/ \sqrt{n}$. To reach a given precision $\xi>0$ we thus need $\nSamples \gtrsim \xi^{-2} \left(1+2/s^2\right)^{\sampleDim/2} k R^{4}$ training samples. Notice that when $s^{2}$ is of the order of $\sampleDim$ this is of the order of $\xi^{-2} k R^{4}$. However $\left(1+2/s^2\right)^{\sampleDim/2} \leq e^{\rev{d/s^{2}}}$ can grow exponentially with $\sampleDim$ when ${s^{2}}$ is of order one, potentially requiring $\nSamples$ to grow exponentially with $\sampleDim$ to have a small statistical error.

The second term $\norm{\SketchingOperatorProb(\Prob)-\SketchingOperatorProb(\Prob^{\star})}_{2}$ and the third one $\divp_{\hyp^{\star}}^{\HypClassRef}(\Prob\|\Prob^{\star})$ measure a modeling error, as they vanish when $\Prob$ belongs to the considered family of Gaussian mixtures. The second term can be controlled using Pinsker's inequality $\normTV{\Prob-\Prob'} \leq \sqrt{2\KLdiv{\Prob}{\Prob'}}$ \citep{Fedotov2003}. Considering $\SketchingOperator_{\vu}(\sample) := \inner{\SketchingOperator(\sample),\vu}$ where $\vu \in \RR^{\nMeasures}$ satisfies $\norm{\vu}_{2}\leq 1$, we have $\abs{\SketchingOperator_{\vu}(\sample)} \leq \norm{\SketchingOperator(\sample)}_{2}=1$ for all $\sample$. By definition of the total variation norm it follows that 
\begin{align*}
\norm{\SketchingOperatorProb(\Prob)-\SketchingOperatorProb(\Prob^{\star})}_{2}
= 
\sup_{\vu \in \RR^{\nMeasures},  \norm{\vu}_{2}\leq 1}
\inner{\SketchingOperatorProb(\Prob)-\SketchingOperatorProb(\Prob^{\star}),\vu}
&=
\sup_{\vu \in \RR^{\nMeasures},  \norm{\vu}_{2}\leq 1}
\Exp_{\Sample \sim \Prob}\SketchingOperator_{\vu}(\sample)-
\Exp_{\Sample \sim \Prob^{\star}}\SketchingOperator_{\vu}(\sample)\\
& \leq \normTV{\Prob-\Prob^{\star}}\leq \sqrt{2\KLdiv{\Prob}{\Prob^{\star}}}.
 \end{align*}
As $\Risk_{\mathtt{GMM}}(\Prob,\hyp^{\star})$ is, up to an additive offset, equal to $\KLdiv{\Prob}{\Prob^{\star}}$, this is reminiscent of the type of distribution free control obtained for clustering using \citeppartone{Lemma~\ref{P1-lem:LemmaBiasTermBis}} %
\rev{Whether $\divp_{\hyp^{\star}}^{\HypClassRef}(\Prob\|\Prob^{\star})$ vanishes as in compressive clustering (cf \citeppartone{Lemma~\ref{P1-lem:LemmaBiasTerm}}  and Lemma~\ref{lem:BiasKMeans} in the appendix) is an interesting question left to further work.}

\rev{As in compressive clustering the fourth term, $d(\Prob^{\star},\HypClass)$, is a measure of distance of the best (unconstrained) gaussian mixture model to the considered constrained hypothesis class. Controlling this term as was done for compressive clustering in Lemma~\ref{lem:bounddisttosepclust} would require further investigations.}

\paragraph{Separation assumption.} Given the scale parameter $s\rev{\geq 1}$ and the number of Gaussians $k$, Theorem~\ref{thm:maingmmthm} sets a separation condition $\sep$ sufficient to ensure compressive statistical learning guarantees with the proposed sketching procedure, as well as a sketch size driven by $M_{s}$. 
Contrary to the case of Compressive Clustering, one cannot target an arbitrary small separation as {for any value of $s$} we have %
\rev{$\sep \geq 4\sqrt{2 \log(ek)}$}.
Reaching guarantees for \rev{a level of separation $\order{\sqrt{\log(ek)}}$} requires choosing $s$ of the order of one. As we have just seen, this may require exponentially many training samples to reach a small estimation error, which is not necessarily surprising as such a level of separation is smaller than generally found in the literature \citep[see e.g.][]{Achlioptas2005,Dasgupta2000,Vempala2004}. For larger values of the scale parameter $s$, the separation required for our results to hold is larger.

\paragraph{Sketch size.} %

Contrary to the case of Compressive Clustering (cf Theorem~\ref{thm:mainkmeansthm}), the choice of the scale parameter $s$ also impacts the sketch size required for the guarantees of Theorem~\ref{thm:maingmmthm} to hold.
   Choosing $s^2=2$ we get $\sep^2$ of the order of $\log(ek)$ , and ~\eqref{eq:sksizeGMM} holds as soon as \rev{(with a universal numerical constant $C$ that may vary from line to line below)}
   \[
       m \geq C \coveps^{-2} \cdot 2^{d/2} 
     \cdot k \cdot \{kd \cdot [d+\log k + \log R + \log(1/\coveps)] + \log(1/\zeta) \} \cdot  \log(ek).
   \]
   Choosing $s^2=d$ we get $\sep^2$ of the order of  $d \log(ek)$, and~\eqref{eq:sksizeGMM} holds as soon as
    \[
  m \geq C \coveps^{-2}  
     \cdot k \cdot \{kd \cdot [1 + \log(kd) +\log(R) + \log(1/\coveps)] + \log(1/\zeta) \} \cdot \log(ek)\min(\log(ek),d).
   \]
   Choosing $s^2=d/\log(ek)$ we get $\sep^2$ of the order of $d + \log k$, and~\eqref{eq:sksizeGMM} holds as soon as
    \[
  m \geq C \coveps^{-2} 
     \cdot k^2 \cdot \{kd \cdot [1 +\log(kd) +\log(R) + \log(1/\coveps)] + \log(1/\zeta) \} \cdot \min(\log^2(ek),d).
   \]
Choosing $s^{2} \gg \sampleDim$ does not seem to pay off.

\paragraph{Tradeoffs.}
Overall we observe a tradeoff between the required sketch size, the required separation of the means in the considered class of GMMs, and the sample complexity. When the scale parameter $s$ decreases, higher frequencies are sampled (or, equivalently, the %
\rev{spatial kernel is more localized}), and the required separation of means decreases. As a price, a larger number of sampled frequencies is required, and the sketch size increases as well as the factor $C_\SketchingOperatorProb$. 
\begin{table}[htbp]
\begin{center}
\begin{tabular}{|c||c|c|c|c|}
\hline
Scale  & Separation & Estimation error & 
Sketch size\\
$s^{2}$ & $\sep$ & factor \rev{$C_\SketchingOperatorProb$} & %
$\rev{m}$ \\ %
\hline
\hline
$\sampleDim$ & $\sqrt{\sampleDim \log (ek)}$ & \rev{$\sqrt{k}R^{2}$} & 
$k^{2}\sampleDim \cdot \rev{\log(ek \sampleDim R) \log^{2}(ek)}$\\
\hline
$\frac{\sampleDim}{\log (ek)}$ & $\sqrt{\sampleDim+\log (ek)}$ & $\rev{k\sqrt{k}R^{2}}$ & 
$k^{3}\sampleDim \cdot \rev{\log (ek\sampleDim R)\log^{2}(ek)}$\\
\hline
$2$ & $\sqrt{\log (ek)}$ & $\rev{2^{\sampleDim/2} \sqrt{k}R^{2}}$ & 
$k^2\sampleDim^{\rev{2}} \cdot 2^{\sampleDim/2} \cdot \rev{(1+\log(k R)/\sampleDim)\log(ek)}$\\
\hline
\end{tabular}
\end{center}

\caption{\label{tab:gmmsep}Some tradeoffs between separation assumption, estimation error factor, and sketch size guarantees obtained using Theorem~\ref{thm:maingmmthm} for various values of the scale parameter $s^{2}$ of the frequency distribution~\eqref{eq:freqdistGMM}. \rev{Each expression gives an order of magnitude up to universal numerical factors and factors depending only on $\coveps$ and $\zeta$.} }
\end{table}
We give some particular values for $s$ in Table \ref{tab:gmmsep}. The regime $s^{2} = 2$ may be useful to resolve close Gaussians in moderate dimensions (typically $\sampleDim \leq 10$) where the factor $2^{\sampleDim/2}$ in sample complexity and sketch size remains tractable. %

\paragraph{Learning algorithm and improved sketch size guarantees?}

Again, although Theorem~\ref{thm:maingmmthm} only provides guarantees when the sketch size $\nMeasures$ exceeds the order of $\PCAdim^2\sampleDim$ (up to logarithmic factors, and for the most favorable choice of scale parameter $s$ with the strongest separation constraints), the observed empirical phase transition pattern  \citep{keriven:hal-01329195} (using an algorithm to adress the optimization of~\eqref{eq:DefRiskProxyGMM} with a greedy heuristic) suggests that $\nMeasures$ of the order of $\PCAdim\sampleDim$, i.e. of the order of the number of unknown parameters, is in fact sufficient. 
Also, while Theorem~\ref{thm:maingmmthm} only handles mixtures of Gaussians with fixed known covariance matrix, the same algorithm has been observed to behave well for mixtures of Gaussians with unknown diagonal covariance.

\section{Establishing the RIP for general mixture models}\label{sec:general_mixtures}

\rev{To establish the main results of the previous sections, Theorem~\ref{thm:mainkmeansthm} and Theorem~\ref{thm:maingmmthm}, we will prove that the main assumption~\eqref{eq:lowerDRIP} of Theorem~\ref{thm:LRIPsuff_excess} holds with high probability.}
\rev{As recalled in Section~\ref{sec:ingredientsLRIP} below (see Theorem \ref{thm:mainLRIP}), this can be achieved using the general approach described in \citeppartone{Section~\ref{P1-sec:ChoiceSketch}} relating random features and \emph{kernel mean embeddings} of probability distributions, and using the notion of a \emph{normalized secant set}. As the models sets appearing in Theorem~\ref{thm:mainkmeansthm} and Theorem~\ref{thm:maingmmthm} are mixture models (mixtures of $k$ Dirac, or mixtures of $k$ Gaussians), we develop in Section~\ref{sec:separatedmixtures} tools for generic mixture models, introducing the notion of \emph{(separated) dipole} and that of \emph{mutual coherence} of separated dipoles.}

\rev{\subsection{Ingredients to establish the LRIP for randomized sketching} \label{sec:ingredientsLRIP}}

  Considering a parameterized family of (real- or complex-valued) measurable functions $\FClass := \{\rfeat\}_{\freq \in \freqSpace}$ \rev{over $\SampleSpace$} and a probability distribution $\freqdist$ over the parameter set $\freqSpace$ (often $\freqSpace \subseteq \RR^\sampleDim$), the random feature functions we consider are defined by drawing $\freq_j$, $1\leq j\leq \nMeasures$, $i.i.d.$ from the distribution $\freqdist$ and defining 
\begin{equation}\label{eq:DefRandomFeatureFunction}
\SketchingOperator(\sample) := \tfrac{1}{\sqrt{m}} \left(\rfeatj(\sample)\right)_{j=1,m}.
\end{equation}
The expectation of $\langle \SketchingOperator(\sample),\SketchingOperator(\sample')\rangle = \tfrac{1}{\nMeasures}\sum_{j=1}^{m} \rfeatj(\sample)\overline{\rfeatj(\sample')}$ defines a kernel 
\begin{equation}\label{eq:DefIntegralRepresentation}
 \kernel(\sample,\sample') := \Exp_{\freq\sim\freqdist}\rfeat(\sample)\overline{\rfeat(\sample')}
\end{equation}
as well as the corresponding {\em mean embedding} kernel \citep{Sriperumbudur2010} for probability distributions,
\begin{equation}
\label{eq:DefMeanMapEmbedding}
\kappa(\Prob,\Prob') := \Exp_{\Sample \sim \Prob} \Exp_{\Sample' \sim \Prob'} \kernel(\Sample,\Sample'),
\end{equation}
and the associated Maximum Mean Discrepancy (MMD) metric
\begin{equation}\label{eq:DefMMD}
\norm{\Prob-\Prob'}_{\kappa} := \sqrt{\kappa(\Prob,\Prob)-2\rev{\operatorname{Re}(}\kappa(\Prob,\Prob')\rev{)}+\kappa(\Prob',\Prob')}.
\end{equation}
By construction $\normkern{\Prob-\Prob'}^{2}$ is the expectation (with respect to the draw of $\freq_{j}$, $1 \leq j \leq \nMeasures$) of 
\[
\norm{\SketchingOperatorProb(\Prob)-\SketchingOperatorProb(\Prob')}_2^{2} = \frac{1}{\nMeasures} \sum_{j=1}^{\nMeasures} \abs{\Exp_{\Sample \sim \Prob} \rfeatj(\Sample)-
\Exp_{\Sample' \sim \Prob'} \rfeatj(\Sample')}^{2}.
\]
  A quantity of interest, given a model set $\Model$, is a {\em concentration function} $t \mapsto \ConcFn(t) \in (0,\infty]$ such that 
\begin{equation}\label{eq:PointwiseConcentrationFn}
\mathbb{P}\left(\abs{\frac{\norm{\SketchingOperatorProb(\mProb)-\SketchingOperatorProb(\mProb')}_{2}^{2}}{\norm{\mProb-\mProb'}_{\kappa}^{2}}-1} \geq t\right) \leq 2 \exp\left(-\frac{\nMeasures}{\ConcFn(t)} \right),\qquad \forall \mProb,\mProb' \in \Model,\quad \forall t > 0,\qquad \forall \nMeasures \geq 1.
\end{equation}
The \emph{normalized secant set} of the model set $\Model$ with respect to a kernel $\kernel$ is the following subset of the set of finite signed measures (see \citeppartone{Appendix~\ref{P1-sec:FiniteSignedMeasures}}):%
\begin{equation}\label{eq:DefNormalizedSecantSet}
\secant_{\kernel} = \secant_{\kernel}(\Model) := \set{\frac{\mProb-\mProb'}{\normkern{\mProb-\mProb'}}: \mProb,\mProb'\in\Model, \normkern{\mProb-\mProb'}>0}.
\end{equation} 
Given a function class $\GClass$ of measurable %
functions $g: \SampleSpace \to \RR\ \text{or}\ \CC$, the \emph{radius} of a subset $\mathcal{E}$ of finite signed measures is denoted
\begin{equation}\label{eq:DefSetRadius}
  \rev{\normfclass{\mathcal{E}}{G} := \sup_{\HH \in \mathcal{E}} \normfclass{\HH}{G}
    = \sup_{\HH \in \mathcal{E}} \sup_{g \in \GClass} \abs{\int g d\HH}.}
\end{equation}
Of particular interest will be $\dnormloss{\secant_{\kernel}}{}$ and $\normfclass{\secant_{\kernel}}{F}$.
The \emph{covering number} $\covnum{d(\cdot,\cdot)}{S}{\coveps}$ of a set $S$ with respect to a (pseudo)metric\footnote{Further reminders on metrics, pseudometrics, and covering numbers are given in \citeppartone{Appendix~\ref{P1-sec:notations_definitions}}.} $d(\cdot,\cdot)$ is the minimum number of closed balls of radius $\coveps$ with respect to $d(\cdot,\cdot)$ with centers in $S$ needed to cover $S$. 
We can now recall \citeppartone{Theorem~\ref{P1-thm:mainLRIP}}: %
\rev{\begin{theorem}\label{thm:mainLRIP}
Consider $\FClass := \set{\rfeat}_{\freq \in \Omega}$ a family of real or complex-valued functions on $\SampleSpace$, $\freqdist$ a probability distribution on $\Omega$, $\SketchingOperator$ the associated random feature function and $\kernel$ the corresponding kernel. Consider the pseudometric on $\FClass$-integrable probability distributions\footnote{In fact, we consider the extension of $d_{\rev{\FClass}}$ to finite signed measures, see Appendix~\ref{P1-sec:FiniteSignedMeasures} in \citepartone.} 
\begin{equation}\label{eq:DefYetAnotherMetric}
d_\FClass(\Prob,\Prob'):=\sup_{\freq\in \freqSpace}\abs{\abs{\Exp_{\Sample \sim \Prob}\rfeat(\Sample)}^2 - \abs{\Exp_{\Sample' \sim \Prob'}\rfeat(\Sample')}^2}.
\end{equation} 
Consider a model set $\Model$ and $\secant_{\kernel}$ its normalized secant set.
Assume the pointwise concentration function $\ConcFn(\delta)$ satisfying~\eqref{eq:PointwiseConcentrationFn} exists. 
For $0<\delta,\zeta<1$, if
\begin{equation}
\nMeasures \geq \ConcFn(\delta/2) \cdot
\log\Big(2\covnum{d_\FClass}{\secant_\kernel}{\delta/2}/\probLevel\Big),
\end{equation}%
then, with probability at least $1-\probLevel$ on the draw of $(\freq_{j})_{j=1}^{m}$, the operator $\SketchingOperatorProb$ induced by $\SketchingOperator$ (cf \eqref{eq:SketchingOperatorProbDef}) satisfies 
\begin{equation}\label{eq:mainLRIPPureKernel}
1-\delta \leq 
\frac{\norm{\SketchingOperatorProb(\mProb)-\SketchingOperatorProb(\mProb')}_{2}^{2}}{\norm{\mProb-\mProb'}_{\kappa}^{2}}
\leq 1+\delta, \qquad \forall \mProb,\mProb' \in \Model.
\end{equation}
When~\eqref{eq:mainLRIPPureKernel} holds, the LRIP~\eqref{eq:lowerDRIP} holds with constant 
$C_\SketchingOperatorProb :=\frac{\dnormloss{\secant_{\kernel}}{}}{\sqrt{1-\delta}}$ and $\eta = 0$.
\end{theorem}
}
\subsection{\rev{Separated mixtures models, dipoles, and mutual coherence}}\label{sec:separatedmixtures}

\rev{In Theorems~\ref{thm:mainkmeansthm} and~\ref{thm:maingmmthm}, the random feature map $\SketchingOperator$  is made of (weighted) random Fourier features, leading to a shift-invariant kernel $\kernel$, and the considered model set is a mixture of Diracs (resp. of Gaussians) satisfying a certain separation condition. To prove these theorems using Theorem \ref{thm:mainLRIP}, our main goal is to bound the radius of the normalized secant set, $\dnormloss{\secant_{\kernel}}{}$, as well as the concentration function $\ConcFn(t)$ (see \eqref{eq:PointwiseConcentrationFn})} and the covering numbers of $\secant_{\kernel}$ (see~\eqref{eq:DefNormalizedSecantSet}) \rev{with respect to the pseudometric~\eqref{eq:DefYetAnotherMetric}}.
As the distance $\normkern{\mProb - \mProb'}$ is the denominator of all these expressions,  most difficulties arise when $\normkern{\mProb - \mProb'}$ is small ($\mProb,\mProb' \in \Model$ get ``close'' to each other) and we primarily have to control the ratio $\norm{\mProb-\mProb'}/\normkern{\mProb-\mProb'}$ for various norms when $\normkern{\mProb-\mProb'} \rightarrow 0$.
In this section, we develop a framework to control these quantities when the model $\Model$ is a mixture model, which covers both mixtures of Diracs and mixtures of Gaussians.

\rev{We consider a given parametrized family of base distributions $\BasicSet = (\ParamSpace,\metricParam,\embd)$ where $\ParamSpace$ is a parameter set (typically a subset of a finite-dimensional vector space), $\metricParam$ is a metric on $\ParamSpace$, and $\embd: \Param \in \ParamSpace \mapsto \embd(\Param) = \Prob_{\Param}$ is an \emph{injective} map defining a family of probability distributions}
(e.g. a family of Diracs or of Gaussians). \rev{In statistical terms, $\BasicSet$ is an identifiable statistical
  model whose parameter space is equipped with a metric, \rev{and $\Prob_{\Param}$ is a (probability) measure on the sample space $\SampleSpace$}.}
We define $\rev{2}$-separated \rev{$k$-}mixtures \rev{from $\BasicSet$} as 
\begin{equation}\label{eq:DefMixSetSep}
\rev{\MixSetSep{k}}
:=\set{\mProb = \sum_{l=1}^\ell \alpha_l\Prob_{\Param_l}: \ell \leq k,~\alpha_l> 0,~\sum_{l=1}^\ell \alpha_l=1,~\rev{\Param_l\in \ParamSpace},~\metricParam(\Param_l,\Param_{l'}) \geq \rev{2}
 ~\forall l\neq l' \leq \ell}.
\end{equation}
\begin{remark}
\rev{In the case of Diracs $\Prob_{\Param}=\delta_{\Param}$, with $\metricParam(\Param,\Param') = \norm{\Param-\Param'}_2/\sep$, $\MixSetSep{k}$ is the set of mixtures of $k$ pairwise $2\sep$-separated Diracs considered in Section~\ref{sec:clustering}.
For Gaussians $\Prob_{\Param} = \mathcal{N}(\Param,\covar)$, $\metricParam(\Param,\Param') := \normmah{\Param-\Param'}{\covar}/\sep$, we obtain the set of $2\sep$-separated Gaussian mixtures considered in Section~\ref{sec:gmm}.}
\end{remark}

The notion of \emph{dipoles} will turn out to be particularly useful in our analysis. %
\begin{definition}[Dipoles, separation]\label{def:Dipole}
A finite signed measure\footnote{See Appendix~\ref{P1-sec:FiniteSignedMeasures} in \citepartone} $\nu$ is a {\bf dipole} with respect to  \rev{$\BasicSet= (\ParamSpace,\metricParam,\embd)$} if it admits a decomposition as $\nu=\alpha_{\rev{1}} \Prob_{\Param_{\rev{1}}} - \alpha_{\rev{2}} \Prob_{\Param_{\rev{2}}}$ where $\Param_{\rev{1}}, \Param_{\rev{2}} \in \ParamSpace$, $\metricParam(\Param_{\rev{1}},\Param_{\rev{2}})\leq 1$ and \rev{$\alpha_{i} \geq 0$} %
for $i=1,2$.
The coefficients $\alpha_i$'s are not necessarily normalized to $1$, and any of them can be put to $0$ to yield a \emph{monopole} as a special case.
Two dipoles $\nu,\nu'$ are {\bf 1-separated} if they admit a decomposition $\nu=\alpha_1 \Prob_{\Param_1} - \alpha_2 \Prob_{\Param_2}$, $\nu'=\alpha'_1 \Prob_{\Param'_1} - \alpha'_2 \Prob_{\Param'_2}$ \rev{as above} such that
$\metricParam(\Param_i,\Param'_j)\geq 1$ for all $i,j\in\set{1,2}$.
\end{definition}
\rev{The relevance of the notion of separated dipoles to handle the secant of
  separated mixtures is captured
in the following decomposition lemma:}
\begin{lemma}\label{lem:decompDipoles}
\rev{If $\mProb,\mProb' \in \MixSetSep{k}$, then there exists $\ell \leq 2k$ nonzero dipoles $(\nu_l)_{1\leq l \leq \ell}$ that are pairwise $1$-separated and satisfy} %
\(
\mProb-\mProb'=\sum_{l=1}^{\ell} \nu_l.
\)
\end{lemma}

\begin{proof}
Using the $2$-separation in $\mProb$ and $\mProb'$ and the triangle inequality, for the metric $\metricParam$ each parameter $\Param_i$ in $\mProb$ is $1$-close to \emph{at most} one parameter $\Param'_j$ in $\mProb'$, and $1$-separated from all other components in both $\mProb$ and $\mProb'$. Hence $\mProb-\mProb'$ can be decomposed into a sum of (at most) $2k$ dipoles (some of which may also be monopoles). %
\end{proof}

\rev{As announced previously, we are interested in RIP inequalities with the kernel
  norm in the denominator. Correspondingly, it is natural to introduce the notion of \emph{normalized} monopoles and dipoles, given a kernel $\kernel$ and the associated mean map embedding. It will be convenient to make some basic assumptions on this kernel.
  For the following definitions, we only assume $\kernel$ is a positive semi-definite (psd) kernel  on $\SampleSpace$ with the associated kernel mean embedding defined by~\eqref{eq:DefMeanMapEmbedding}; the explicit representation in terms of random features is not
needed.}
\begin{definition}[Locally characteristic kernel, normalized kernel]\label{def:loccharkern}
  \rev{A psd  kernel $\kernel$ on $\SampleSpace$ (extended to probability distributions
    on $\SampleSpace$ via the kernel mean embedding~\eqref{eq:DefMeanMapEmbedding}) is \emph{locally characteristic} with respect to $\BasicSet = (\ParamSpace,\metricParam,\embd)$ if it satisfies the following two conditions:
\begin{enumerate}
\item \label{it:nondegeneratekernel} $\normkern{\Prob_{\Param}}>0$  for each $\Param \in \ParamSpace$; 
\item $\abs{\kernel(\Prob_{\Param},\Prob_{\Param'})} < \normkern{\Prob_{\Param}}\normkern{\Prob_{\Param'}}$ for each $\Param \neq \Param' \in \ParamSpace$ such that $\metricParam(\Param,\Param') \leq 1$. 
\end{enumerate}
Note that if $\kernel$ is locally characteristic, then $\normkern{\nu}>0$ for any nonzero dipole.\\
}
\end{definition}

\begin{definition}[Normalized monopoles, normalized dipoles]
The set of {\bf normalized dipoles} \rev{induced by the base family $\BasicSet$ with respect to a locally characteristic kernel $\kernel$} is denoted \rev{by}
\begin{equation}
\label{eq:NormalizedDipoleSet}
\dipoleSet = \dipoleSet_{\kernel}(\BasicSet) := \set{\frac{\nu}{\normkern{\nu}}:\text{$\nu$ is a nonzero dipole}}.
\end{equation}
\rev{It contains as a particular subset the set of {\bf normalized monopoles}
\begin{equation}
\label{eq:NormalizedMonopoleSet}
\monopoleSet = \monopoleSet_{\kernel}(\BasicSet) := \set{\nu_{\Param}:=\frac{\Prob_{\Param}}{\normkern{\Prob_{\Param}}}:\Param \in \ParamSpace}.
\end{equation}}
\end{definition}
\rev{Equipped with these notions we can define the mutual coherence and $\ell$-coherence of a kernel.
\begin{definition}\label{def:Coherence}
A  psd kernel $\kernel$ on $\SampleSpace$ has mutual coherence $M$ with respect to $\BasicSet$ if: (a) it is locally characteristic with respect to $\BasicSet$; and (b)
 for each pair of normalized dipoles $\mu,\mu' \rev{\in \dipoleSet_{\kernel}(\BasicSet)}$  that are $1$-separated from each other, we have\footnote{We properly define in Appendix~\ref{P1-sec:FiniteSignedMeasures}  of \citepartone the extension of the Mean Map Embedding to finite signed measures, to make sense of the notation $\kernel(\nu,\nu')$.}
\begin{equation}
\label{eq:MutualCoherence}
\abs{\kernel(\mu,\mu')}
 \leq \MutualCoherence.
\end{equation}
Given an integer $\ell>0$ and a number $\zeta\in[0,1]$, we say that a kernel $\kappa$ has its { $\ell$-coherence
    with respect to $\BasicSet$ bounded by $\zeta$} if, for any dipoles $(\nu_l)_{1\leq l \leq \ell}$ that are pairwise $1$-separated and such that $\sum_{l=1}^\ell \normkern{\nu_l}^2>0$, it holds
  \begin{equation}
\label{eq:quasiOrthogonality}
1-\zeta \leq \frac{\normkern{\sum_{l=1}^\ell \nu_l}^2}{\sum_{l=1}^\ell \normkern{\nu_l}^2} \leq 1+\zeta.
\end{equation}
\end{definition}
\rev{A crucial step in the analysis to come is the reduction from differences of $k$-mixtures to
  individual dipoles. To this end, the representation of Lemma~\ref{lem:decompDipoles} combined with the quasi-Pythagorean identity~\eqref{eq:quasiOrthogonality}
  will play a central role. 
}
The following result is a direct
  consequence of Gershgorin's disc lemma \cite[see e.g.][Theorem 5.3]{FouRau13} and
  establishes the link between mutual coherence and $\ell$-coherence.} 
\begin{lemma}\label{lem:Gershgorin}
  Consider a kernel $\kernel$ with mutual coherence $M$ with respect to $\BasicSet$.
  Then $\kernel$ has $\ell$-coherence bounded by $M(\ell-1)$.
\end{lemma}
\begin{remark}
The reader familiar with sparse recovery will find this lemma highly reminiscent of the classical link between the coherence of a dictionary and its restricted isometry property \citep[see e.g.][Theorem 5.13]{FouRau13}. To handle incoherence in a continuous ``off the grid'' setting (such as mixtures of separated Diracs in Section~\ref{sec:clustering}, which also appear in super-resolution imaging scenarios~\citep{Candes_2013,Castro_2015,Duval_2015}), the apparently new trick is to consider incoherence \emph{between dipoles} rather than between monopoles.
\end{remark}
\rev{Conditions such that $\kernel$ has low mutual coherence with respect to $\BasicSet$ will be given in Theorem~\ref{thm:MutualCoherenceRBF1}.
  }

\subsection{{\rev{From separated $k$-mixtures to dipoles}}}\label{sec:secantmixture}

\rev{We turn to the ingredients delineated in Section~\ref{sec:ingredientsLRIP}
in order to establish the RIP for (separated) $k$-mixture models.
Using the notions introduced in Section~\ref{sec:separatedmixtures}, the following results allow 
to control the various key quantities in terms of related notions defined by replacing the normalized secant set of $k$-mixtures with the simpler set $\dipoleSet$ of normalized dipoles.\\
In the sequel we will generically assume to have fixed a base distribution family
$\BasicSet$, a kernel $\kernel$, the associated
normalized dipole and monopole sets $\dipoleSet = \dipoleSet_{\kernel}(\BasicSet),
\monopoleSet = \monopoleSet_{\kernel}(\BasicSet)$, an integer $k\geq 1$, the separated $k$-mixture
model $\Model=\MixSetSep{k}$ and its normalized secant $\secant_{\kernel}=\secant_{\kernel}(\Model)$ as introduced in the previous section. Our first result relates the radius of the normalized secant set with respect to any function family $\GClass$
 to the corresponding radius of the set of dipoles.}

\begin{theorem}\label{thm:CompCstFromDipoles}
 \rev{Assume the kernel $\kernel$ has its $2k$-coherence with respect
  to $\BasicSet$ bounded by $\zeta \leq 3/4$.
  Let $\GClass$ be a real or complex-valued measurable function class over $\SampleSpace$. We have }
 \begin{equation}\label{eq:CompCstFromDipoles}
\rev{ \normfclass{\secant_{\kernel}}{G} \leq \sqrt{8\PCAdim} \cdot \normfclass{\dipoleSet}{G}.}
\end{equation}
\end{theorem}
\begin{proof}%
\rev{First, by definition of $\normfclass{\dipoleSet}{G}$, we have $\normfclass{\nu}{G} \leq \normfclass{\dipoleSet}{G} \cdot \normkern{\nu}$ for any dipole $\nu$. }
Let $\mProb,~\mProb' \in \rev{\MixSetSep{\PCAdim}}$. Using Lemma \ref{lem:decompDipoles} we write
\(
\mProb-\mProb' = \sum_{i=1}^{\ell} \nu_i
\)
where $\ell \leq 2k$ and the $\nu_i$'s are dipoles that are pairwise $1$-separated. By the triangle inequality \rev{and the Cauchy-Schwarz inequality} we have
\begin{align*}
\norm{\mProb-\mProb'}_{\GClass}
\leq \sum_{i=1}^{\ell}\norm{\nu_i}_{\GClass}
&\leq \normfclass{\dipoleSet}{G} \cdot \sum_{i=1}^{\ell}\normkern{\nu_i}
\leq \normfclass{\dipoleSet}{G} \cdot \sqrt{\ell} \paren{\sum_{i=1}^{\ell}\normkern{\nu_i}^2}^\frac12
\end{align*}
\rev{By our assumption on the bounded $2k$-coherence of $\kappa$ and  since $\ell \leq 2k$ and $\zeta \leq 3/4$, we have} 
\[
\norm{\mProb-\mProb'}_{\GClass}\leq \frac{\normfclass{\dipoleSet}{G}}{\sqrt{1-\zeta}}\sqrt{\ell} \normkern{\sum_{i=1}^{\ell}\nu_i} \leq 2\sqrt{2\PCAdim} \cdot \normfclass{\dipoleSet}{G} \cdot \normkern{\mProb-\mProb'}. \qedhere
\]
\end{proof}

We now consider the random sketching operator: consider a family of functions $\FClass := \{\rfeat\}_{\freq \in \Omega}$, $\nMeasures$ parameters $(\freq_j)_{j=1}^{m}$ drawn i.i.d. according to \rev{some distribution} $\Lambda$ on $\Omega$, $\SketchingOperatorProb$ the operator induced (see~\eqref{eq:SketchingOperatorProbDef}) by the feature function
\(
\SketchingOperator(\sample) := \tfrac{1}{\sqrt{m}} \left(\rfeatj(\sample)\right)_{j=1}^{m},
\)
\rev{and finally $\kernel$ the associated average kernel, given by~\eqref{eq:DefIntegralRepresentation}. For short, we call $(\FClass,\Lambda)$ a random
  feature family, and $\SketchingOperatorProb,\kernel$ the induced (random) sketching
operator and kernel.}

Concerning the pointwise concentration function for this random sketching operator, by \citeppartone{Lemma~\ref{P1-le:PointwiseConcentrationLemma}}, %
we have $\rev{\normfclass{\secant_{\kernel}}{F} \geq 1}$, and the concentration function satisfies
\begin{equation}\label{eq:ConcFnFromConcCst}
\ConcFn(t) \leq %
2t^{-2}(1+t/3) \cdot \rev{\normfclass{\secant_{\kernel}}{F}^{2}},\qquad \forall t>0.
\end{equation}
\rev{ Observe that this is based on a supremum control over the class $\FClass$, using the radius $\normfclass{\secant_{\kernel}}{F}$, and as such is independent of the choice of the distribution $\Lambda$
  over its index set.}
In settings such as Compressive Clustering with $\sampleDim \gtrsim \log k$, sharper bounds on the concentration function can be obtained for mixture models when the considered kernel has low mutual coherence. \rev{In this situation, thanks to the separation assumption, {\em it is sufficient to properly control the moments wrt. $\Lambda$ of normalized dipoles}, for which sharper bounds may be available.\\
  For notational brevity, we extend by linearity the operator $\SketchingOperatorProb$
  to finite signed measures (in particular for normalized dipoles), and
  for any finite signed measure $\HH$, we denote
  by $\inner{\HH,f}=\int f d\HH$ for an integrable function $f$.
Proofs of the remaining results in this section are in Appendix~\ref{sec:ProofMixtures}.}
\begin{theorem}\label{lem:ConcFnMixturesFromDipole}
  \rev{Consider a random feature family $(\{\rfeat\}_{\freq \in \Omega},\Lambda)$ and the
    induced random sketching operator $\SketchingOperatorProb$ and kernel $\kappa$.
    Assume $\kernel$ has its $2k$-coherence with respect
  to $\BasicSet$ bounded by $\zeta \leq 3/4$.}

Assume there are \rev{$\gamma > 0, \lambda \geq 1$}
such that, for \rev{each} normalized dipole $\HH \in \mathcal{D}$:
\begin{equation}\label{eq:MomentControlDipole}
\Exp_{\omega \sim \freqdist} \brac{\abs{\inner{\HH,\rfeat}}^{2q}  }
\leq  
\frac{q!}{2} \rev{\lambda \gamma^{q-1}}, \qquad \text{for each integer } q \geq 2.
\end{equation}
{Set $V := 16ek\gamma \log^{2}(4ek
  \rev{\lambda})$.} For %
any $\HH \in \secant_\kernel(\MixSetSep{k})$ we have
\begin{equation}\label{eq:MainConcentrationPointwise}
\mathbb{P}
\left(
\abs{%
\norm{\SketchingOperatorProb(\HH)}^2-1}
    \geq t\right)
\leq
2\exp\left(-\frac{m t^{2}}{2\rev{V}(1+t/3)}\right),\qquad \text{for each } t > 0.
\end{equation}
\end{theorem}%
Specific estimates of $\gamma$ such that the moment bounds~\eqref{eq:MomentControlDipole} hold for normalized dipoles will be given in Section~\ref{sec:RBFKernel} (Lemma~\ref{lem:RFFMoments}) and completed in Section~\ref{sec:separatedornot} where we gather all ingredients to prove Theorems~\ref{thm:mainkmeansthm} and~\ref{thm:maingmmthm} for Compressive Clustering and Compressive GMM. 

\rev{Finally, the covering numbers (for $d_{\rev{\FClass}}$) of the normalized secant set are also controlled by those (for $\normrff{\cdot}$) of normalized dipoles.} 
\begin{theorem}\label{thm:covnumSecant}
\rev{Consider a random feature family $(\FClass,\Lambda)$ and the
  induced random sketching operator $\SketchingOperatorProb$ and average kernel $\kernel$.}
    \rev{Assume that $\kernel$ is locally characteristic with respect to $\BasicSet = (\ParamSpace,\metricParam,\embd)$.}
\begin{itemize}
\item  We have $\normfclass{\dipoleSet}{F} \geq 1$, and for each $\Param,\Param' \in \ParamSpace$ such that $\metricParam(\Param,\Param') \leq 1$ and $\alpha,\alpha' \geq 0$
\begin{equation}\label{eq:TangentProof1}
\normkern{\alpha\Prob_{\Param}-\alpha'\Prob_{\Param'}} \leq \normfclass{\alpha\Prob_{\Param}-\alpha'\Prob_{\Param'}}{\FClass} \leq \normfclass{\dipoleSet}{F} \normkern{\alpha\Prob_{\Param}-\alpha'\Prob_{\Param'}}.
\end{equation}
\item
\rev{Assume the kernel $\kernel$ has its $2k$-coherence with respect
  to $\BasicSet$ bounded by $\zeta \leq 3/4$, and consider $d_{\rev{\FClass}}$ the pseudo-metric defined in~\eqref{eq:DefYetAnotherMetric}.
  Then %
  we have for each $\delta>0$:
\begin{equation}
\covnum{d_{\rev{\FClass}}}{\secant_\kernel}{\coveps} 
\leq 
\left[
\covnum{\normfclass{\cdot}{\rev{\FClass}}}{\dipoleSet}{\tfrac{\coveps}{\rev{64k \normfclass{\dipoleSet}{F}}}}
\cdot
\max\left(1,\tfrac{256k\normfclass{\dipoleSet}{F}^{2}}{\coveps}\right)
\right]^{2k}.
\end{equation}
}
\end{itemize}
\end{theorem}

Gathering all the ingredients above \rev{together with the general Theorem~\ref{thm:mainLRIP}}
and Lemma~\ref{lem:Gershgorin} we obtain the following result.
\begin{theorem}\label{thm:LRIPGenericMixture}
Consider $\FClass := \{\rfeat\}_{\freq \in \Omega}$, $\freqdist$ a probability distribution on $\Omega$, and $\kernel$ the induced average kernel.
Assume that $\kernel$ has mutual coherence $M$ with respect to $\BasicSet$ and consider $k \geq 1$ such that $M(2k-1) \leq 3/4$.
Assume that there are $C \geq 1$, $r>0$ such that $\covnum{\normfclass{\cdot}{F}}{\dipoleSet}{\coveps}  \leq \rev{2}(C/\coveps)^{r}$  for each $0<\coveps<1$ and that \rev{there are $\gamma>0,\lambda\geq 1$ such that}
\begin{align}
\sup_{\HH \in \dipoleSet} \Exp_{\omega \sim \freqdist} \brac{\abs{\inner{\HH,\rfeat}}^{2q}  }
&\leq  
\frac{q!}{2} \rev{\lambda \gamma^{q-1}},\label{eq:LRIPGenericMixtureMoments}
\end{align}
for every integer $q \geq 2$. For $0<\delta,\zeta<1$, if $(\freq_j)_{j=1}^{m}$ are drawn i.i.d. according to $\freqdist$  and
\begin{equation}\label{eq:DefNMeasuresGenericMixture}
\nMeasures \geq 
80 \cdot \coveps^{-2} \cdot \min\left(2e\gamma \log^{2}(4ek\rev{\lambda}),\normfclass{\dipoleSet}{F}^2\right) \cdot k \cdot  \left\{
2k (r+1) \left[ \log (kC\normfclass{\dipoleSet}{F}^{2}) + \log(\rev{1024}/\coveps)\right]
+\log(2/\zeta)\right\},
\end{equation}
then, with probability at least $1-\probLevel$ on the draw of $(\freq_{j})_{j=1}^{m}$, we have
\begin{equation}\label{eq:mainLRIPKernelMixture}
1-\delta \leq 
\frac{\norm{\SketchingOperatorProb(\Prob)-\SketchingOperatorProb(\Prob')}_{2}^{2}}{\norm{\Prob-\Prob'}_{\kappa}^{2}}
\leq 1+\delta, \qquad \forall \Prob,\Prob' \in \MixSetSep{k}.
\end{equation}
where $\SketchingOperatorProb$ is the operator induced by 
\(
\SketchingOperator(\sample) := \tfrac{1}{\sqrt{m}} \left(\rfeatj(\sample)\right)_{j=1}^{m}.
\)

When~\eqref{eq:mainLRIPKernelMixture} holds, the LRIP~\eqref{eq:lowerDRIP} holds with 
$C_\SketchingOperatorProb :=\frac{8\sqrt{2k}\dnormloss{\dipoleSet}{}}{\sqrt{1-\delta}}$ and $\eta = 0$
for each loss class $\LossClass$. 
\end{theorem}

\subsection{Strongly characteristic kernels and associated controls}

\rev{%
  Theorem~\ref{thm:LRIPGenericMixture} notably involves two important quantities: the coherence $M$ of the kernel, and the radiuses $\normfclass{\dipoleSet}{G}$, where $\GClass \in \{\DLossClass,\FClass\}$ and $\dipoleSet$ is the set of normalized dipoles. These quantities can be controlled under some assumptions on $\BasicSet$ and $\kernel$ which are essentially captured by a normalized version of the kernel, which
  we introduce now.}

\begin{definition}\label{def:nkernel}
  Let $\BasicSet = (\ParamSpace,\metricParam,\embd)$  be a
  family of base distributions, and $\kernel$ be a psd kernel on $\SampleSpace$
  such that $\normkern{\Prob_{\Param}}>0$  for each $\Param \in \ParamSpace$.
  We define the $\BasicSet$-normalized kernel $\nkernel$ on the parameter space $\ParamSpace$ as
  \begin{equation}
    \label{eq:defnkernel}
\nkernel(\Param,\Param') := \frac{\kernel(\Prob_{\Param},\Prob_{\Param'})}{{\normkern{\Prob_{\Param}}\normkern{\Prob_{\Param'}}}},\quad \Param,\Param' \in \ParamSpace.
\end{equation}
It holds that $\nkernel(\Param,\Param) = 1$ for each $\Param \in \ParamSpace$,
$\abs{\nkernel(\Param,\Param')} \leq 1$ for every $\Param,\Param'$, and $\kernel$ is locally characteristic iff $\abs{\nkernel(\Param,\Param')} < 1$ when $0< \metricParam(\Param,\Param') \leq 1$.

Given $c\in(0,2]$ we say that the kernel $\kernel$ is {\em $c$-strongly locally characteristic} if \rev{it is real-valued and}
\begin{equation}\label{eq:DefCStrongly}
1-\nkernel(\Param,\Param') \geq \frac{c}{2} \metricParam^{2}(\Param,\Param'),\qquad \forall \Param,\Param' \in \ParamSpace\ \text{such that}\ \metricParam(\Param,\Param') \leq 1.
\end{equation}
\end{definition}
\rev{We note that kernels that locally decrease quadratically also appear naturally in sparse spikes recovery \citep{Poon2020}, where infinite-dimensional convex relaxations are employed to estimate sums of Diracs; like we do here for $k$-means/medians however through the non-convex problem \eqref{eq:DefQuasiOptimumClusteringViaProxy}.}
\rev{Concrete examples of such kernels will be given in Section~\ref{sec:RBFKernel}, where typically $\metricParam(\cdot,\cdot)$ is a simple Euclidean distance and $\kappa$ is a Gaussian kernel.}

Our first result relates $\GClass$-radiuses of the set of normalized dipoles $\dipoleSet$ to those of the set of the normalized monopoles $\monopoleSet$.

\begin{theorem}\label{thm:RadiusGeneric}
Consider a kernel $\kernel$ that is $c$-strongly locally characteristic with respect to $\BasicSet$. 
\rev{For any function class $\GClass$ we have
\begin{equation}
\label{eq:admissibilityDipole}
\normfclass{\monopoleSet}{\GClass} \leq \normfclass{\dipoleSet}{G} \leq \normfclass{\monopoleSet}{G}+L_{\GClass}/\sqrt{c},
\end{equation}
with $L_{\GClass}$ the Lipschitz constant of $\Param \mapsto \nu_{\Param} := \Prob_{\Param}/\normkern{\Prob_{\Param}}$ with respect to the metrics $\metricParam$ and $\normfclass{\cdot}{G}$.}
\end{theorem}
The proof is in Appendix~\ref{sec:proofAdmissibilityDipole}.
The second result gives a concrete criterion to establish quantitatively that $\kernel$ is
$c$-strongly characteristic and has bounded mutual coherence.

\begin{theorem}\label{thm:MutualCoherenceRBF1}
  Consider $\BasicSet = (\ParamSpace,\metricParam,\psi)$ a family of base distributions, and $\kernel$ a  psd kernel on $\SampleSpace$. Assume that $\normkern{\Prob_{\Param}}>0$ for each $\Param \in \ParamSpace$ and that the normalized
  kernek $\nkernel$ is of the form
\begin{equation}
\label{eq:DefKernelF}
\rev{\nkernel(\Param,\Param')} = K(\metricParam(\Param,\Param')),\quad \forall \Param,\Param' \in \ParamSpace,
\end{equation}
for a function $K: \RR_+ \rightarrow \RR_+$ such that $K(0) = 1$ and $0 \leq K(u) \leq 1  - \tfrac{c u^2}{2}, \forall u \in [0,1]$, with $0<c \rev{\leq} 2$. Then:
\begin{enumerate}
\item The kernel $\kernel$ is $c$-strongly locally characteristic with respect to $\BasicSet$.
\item If $K$ is bounded and differentiable with bounded and Lipschitz derivative on $[1,\infty)$, and if there exists a mapping
$\psi:\ParamSpace \mapsto \mathcal{H}$, with $\mathcal{H}$ some \rev{Hilbert} space, such that
$\metricParam(\Param,\Param') := \norm{\psi(\Param)-\psi(\Param')}_{\mathcal{H}}$, 
then the kernel $\kernel$ has mutual coherence with respect to $\BasicSet$ bounded by
\begin{equation}
\label{eq:MutualCoherenceGenericMain}
\MutualCoherence \leq \tfrac{\rev{4}C}{\min(c,1)},
\end{equation}
with%
\begin{equation}
\label{eq:MutualCoherenceRBF1}
C = C(K) := \max(K_{\max},\rev{(2K'_{\max}+K''_{\max})}),
\end{equation}
where
\(
K_{\max} := \sup_{u \geq 1} |K(u)|,\ K'_{\max} := \sup_{u \geq 1}\abs{K'(u)},\ 
K''_{\max} := \sup_{\stackrel{u, v \geq 1}{u \neq v}} \frac{\abs{K'(u)-K'(v)}}{\abs{u-v}}.
\)
\end{enumerate}
\end{theorem}
The proof is in Appendix~\ref{sec:proofmutualcoherence}.

\section{\rev{Random Fourier Sketching with location-based mixtures}}
\label{sec:RBFKernel}

Given the prominent role of Random Fourier Features for Compressive Clustering and Compressive Gaussian Mixture Modeling, we now focus on this specific setting. Mixtures of Diracs / Gaussians belong to what we call \emph{location-based mixture models}. Combined with a shift-invariant kernel on samples they yield a shift-invariant mean embedding, and we show that the assumptions of Theorem~\ref{thm:LRIPGenericMixture} and Theorem~\ref{thm:MutualCoherenceRBF1} are satisfied. \rev{We note that in \citep{Poon2020}, non-translation-invariant embeddings are treated with the same techniques as translation-invariant ones through a Riemannian geometry framework, which is an interesting path for future extensions.}

\subsection{Location-based mixtures and shift-invariant kernels}\label{sec:shiftmodels}
\rev{Much like the introduced notion of family of parametrized base distributions $\BasicSet = (\ParamSpace,\metricParam,\embd)$ is, in statistical terminology, an
  identifiable statistical model, the following definition specializes it to the case where the distributions in that collection are obtained by translation of
a single reference distribution, which is generally called a location family.}

\begin{definition}[Location family, location-based mixtures]\label{def:shiftmodel}
Consider $\norm{\cdot}$ a norm on $\RR^{\sampleDim}$, $\Prob_{0}$ a probability distribution on $\SampleSpace= \RR^{\sampleDim}$. For $\Param \in \RR^{\sampleDim}$, denote $\Prob_{\Param}$ the distribution of $\Param+\Sample$ when $\Sample \sim \Prob_{0}$ and consider the mapping $\embd: \Param \mapsto \Prob_{\Param}$.  In statistical terms, given $\ParamSpace \subset \RR^{\sampleDim}$, $\BasicSet = (\ParamSpace,\norm{\cdot},\embd)$ is a \emph{location family}. We call $\MixSetSep{k}$, where $k \geq 1$, a \emph{location-based mixture model}. 
\end{definition}
\begin{proposition}\label{prop:DefKShift}
Consider $\Prob_{0}$ a probability distribution and $\kernel$ a shift-invariant kernel on $\RR^{\sampleDim}$ such that $\normkern{\Prob_{0}} > 0$. Let $\BasicSet$ be a location family based on $\Prob_{0}$. For each $\Param \in \ParamSpace$ we have $\normkern{\Prob_{\Param}} = \normkern{\Prob_{0}}$. There exists $\mathtt{K}: \RR^{\sampleDim} \mapsto \RR$ such that $\abs{\mathtt{K}(\Param)} \leq \mathtt{K}(\rev{0}) = 1$ for every $\Param \in \RR^{\sampleDim}$ and
\begin{equation}\label{eq:KShiftBased}
\nkernel(\Param,\Param') 
= \mathtt{K}(\Param-\Param'),\quad \forall \Param,\Param' \in \ParamSpace.
\end{equation}
By a standard abuse of notation we also denote $\nkernel$ the function  $\mathtt{K}$, so that $\nkernel(\Param,\Param') = \nkernel(\Param-\Param')$.
\end{proposition}
\begin{proof} 
Since $\kernel$ is shift-invariant, there is $g$ such that $\kernel(\sample,\sample') = g(\sample-\sample')$ for each $\sample,\sample' \in \SampleSpace$, hence
\begin{align*}
\kernel(\Prob_{\Param},\Prob_{\Param'})
&= \Exp_{\Sample \sim \Prob_{\Param}}\Exp_{\Sample' \sim \Prob_{\Param'}} \kernel(\Sample,\Sample') 
= \Exp_{\Sample \sim \Prob_{0}}\Exp_{\Sample' \sim \Prob_{0}} \kernel(\Sample+\Param,\Sample'+\Param')\\
&=  \Exp_{\Sample \sim \Prob_{0}}\Exp_{\Sample' \sim \Prob_{0}} g(\Param-\Param'+\Sample-\Sample')
\end{align*}
only depends on $\Param-\Param'$. As a result, there is a function $G: \RR^{\sampleDim}\to\RR$ such that $\kernel(\Prob_{\Param},\Prob_{\Param'}) = G(\Param-\Param')$. In particular, $\normkern{\Prob_{\Param}}^{2} = \kernel(\Prob_{\Param},\Prob_{\Param})  = G(0)>0$ for any $\Param$, \rev{and $\nkernel(\Param,\Param')=
G(\Param-\Param')/G(0) =: \mathtt{K}(\Param-\Param') \leq 1$.}
\end{proof}
Shift-invariant kernels are intimately connected with random Fourier features \rev{via Bochner's theorem \citep{Rahimi2007}.} %
For technical reasons, we consider weighted variants of random Fourier features. \rev{In fact, observe that the integral kernel
given by~\eqref{eq:DefIntegralRepresentation} is invariant if we rescale the features by a weight function
$w(\omega)^{-1}$ and their distribution $\Lambda$ by $w^2(\omega)$. This additional freedom in the
design of random features corresponding to a given kernel is convenient to obtain appropriate control of the
moments of $\Lambda$ involving powers of the frequency, as will be needed below.}
\begin{definition}[Weighted random Fourier features]\label{def:weightedRFF}
  Consider $\Omega = \SampleSpace = \RR^{\sampleDim}$, $w:  \Omega \to \RR$ a function such that $\inf_{\freq} w(\freq) = w(0) = 1$, and $\FClass = \set{\rfeat}_{\freq \in \Omega}$,
  \rev{with} %
\begin{equation}
\label{eq:fourierfeaturecomponent}
\rfeat(\sample) := \frac{e^{\jmath \inner{\freq,\sample}}}{w(\freq)} \qquad \forall \sample \in \RR^{\sampleDim}.
\end{equation}
Given an arbitrary probability distribution $\freqdist$ on the vector $\freq \in \RR^{\sampleDim}$,
the corresponding kernel given by~\eqref{eq:DefIntegralRepresentation} is shift invariant.
\end{definition}

\subsection{Ingredients to apply Theorem~\ref{thm:LRIPGenericMixture}}

To exploit Theorem~\ref{thm:LRIPGenericMixture}, a first ingredient is to control $\normfclass{\dipoleSet}{F}$ via Theorem~\ref{thm:RadiusGeneric} and a characterization of the quantities $\normfclass{\monopoleSet}{F}$ and $L_{\FClass}$. 
\begin{lemma}\label{lem:ShiftBasedRF}
Consider $\BasicSet = (\ParamSpace,\norm{\cdot},\embd)$ a location family built from a probability distribution $\Prob_{0}$ and denote $\norm{\cdot}_{\star}$ the dual norm defined for $u \in \RR^{\sampleDim}$ by
\begin{equation}\label{eq:DefDualNorm}
\norm{u}_{\star} := \sup_{v: \norm{v} \leq 1} u^{T}v.
\end{equation}
Consider a shift-invariant kernel $\kernel$ on $\RR^{\sampleDim}$ such that $\normkern{\Prob_{0}}>0$, $\monopoleSet = \monopoleSet_{\kernel}(\BasicSet)$, and $\psi: \Param \mapsto \Prob_{\Param}/\normkern{\Prob_{\Param}}$.

Consider $w$ a weight function and $\FClass = \set{\rfeat}_{\freq \in \Omega}$ as in Definition~\ref{def:weightedRFF}, and let $L_{\FClass}$ be the Lipschitz constant of $\psi$ with respect to $\norm{\cdot}$ and $\normfclass{\cdot}{F}$. Denoting $\FClass' := \set{ \norm{\freq}_{\star} \rfeat}_{\freq \in \Omega}$, we have 
\begin{align*}
\normfclass{\monopoleSet}{F} &=  \normkern{\Prob_{0}}^{-1} \cdot \normrff{\Prob_{0}} =  \normkern{\Prob_{0}}^{-1};\\
  L_{\FClass} %
                             & =  \normkern{\Prob_{0}}^{-1} \cdot 
\sup_{\freq} \abs{\inner{\Prob_{0},\norm{\freq}_\star\rfeat}}
= 
\normkern{\Prob_{0}}^{-1} \cdot \normrffd{\Prob_{0}}.
\end{align*}
 \end{lemma}
 \begin{proof}
For  $\Param \in \ParamSpace$ and $\freq \in \RR^{\sampleDim}$ we have
\begin{equation}\label{eq:InnerProdRFFPTheta}
\inner{\Prob_{\Param},\rfeat} = \Exp_{\Sample \sim \Prob_{\Param}} e^{\jmath \inner{\freq,\Sample}}/w(\freq)
= \Exp_{\Sample \sim \Prob_{0}} e^{\jmath \inner{\freq,\Param+\Sample}}/w(\freq)
= \inner{\Prob_{0},\rfeat} e^{\jmath \inner{\freq,\Param}} 
\end{equation}
hence $\normrff{\Prob_{\Param}} = \sup_{\freq} \abs{\inner{\Prob_{0}, \rfeat}} = \sup_{\freq} \abs{\Exp_{\Sample \sim \Prob_{0}} e^{\jmath \inner{\freq,\Sample}}/w(\freq)} \leq 1$ since $w(\freq) \geq 1$. The bound is achieved for $\freq = 0$ since $w(0)=1$. Now, by definition, any $\HH \in \monopoleSet$ can be written as $\HH = \psi(\Param) = \Prob_{\Param}/\normkern{\Prob_{\Param}}$. Since $\kernel$ is shift-invariant we have $\normkern{\Prob_{\Param}} = \normkern{\Prob_{0}}$ hence $\normrff{\HH} = \normrff{\Prob_{\Param}}/\normkern{\Prob_{\Param}} = \normkern{\Prob_{0}}^{-1}$ is independent of $\Param$ and $\normrff{\monopoleSet} := \sup_{\HH \in \monopoleSet} \normrff{\HH} = \normkern{\Prob_{0}}^{-1}$. 
Another consequence of~\eqref{eq:InnerProdRFFPTheta} is that $\inner{\psi(\Param),\rfeat} =  \normkern{\Prob_{0}}^{-1}\inner{\Prob_{0},\rfeat} e^{\jmath \inner{\freq,\Param}}$. For $a \leq b$, $\abs{e^{\jmath a}-e^{\jmath b}} = \abs{\int_{a}^{b} \jmath e^{\jmath u} du} \leq \int_{a}^{b} \abs{\jmath e^{\jmath u}} du = b-a$, hence
\begin{align*}
\normrff{\psi(\Param')-\psi(\Param)}
&=  \normkern{\Prob_{0}}^{-1} \cdot 
\sup_{\freq}  \set{ \abs{\inner{\Prob_{0},\rfeat}} \cdot \abs{e^{\jmath\inner{\freq,\Param'}}-e^{\jmath\inner{\freq,\Param}}}}
 \leq
 \normkern{\Prob_{0}}^{-1} \cdot \sup_{\freq}  \set{\abs{\inner{\Prob_{0},\rfeat}} \cdot \abs{\inner{\freq,\Param'-\Param}}}\\
&  \leq 
 \normkern{\Prob_{0}}^{-1} \cdot \sup_{\freq} \set{ \abs{\inner{\Prob_{0},\rfeat}} \cdot \norm{\freq}_\star  }\cdot \norm{\Param'-\Param} 
=
 \normkern{\Prob_{0}}^{-1} \cdot \sup_{\freq} \abs{\inner{\Prob_{0},\norm{\freq}_\star\rfeat}}    \cdot \norm{\Param'-\Param} \\
& =  \normkern{\Prob_{0}}^{-1} \cdot\normrffd{\Prob_{0}} \cdot \norm{\Param'-\Param}.
 \end{align*}
To conclude we show that the bound is tight. When $\normrffd{\Prob_{0}}$ is finite (resp. infinite), for each integer $n \geq 1$ there is $\freq_{n}\rev{\neq 0}$ such that $\abs{\inner{\Prob_{0},\norm{\freq_{n}}_\star\rfeatn}} \geq \normrffd{\Prob_{0}}-1/n$ (resp. $ \geq n$). By compactness of the unit ball of $\norm{\cdot}$ in $\RR^{\sampleDim}$ there is $u_{n}$ such that $\norm{u_{n}} = 1$ and $\inner{\freq_{n},u_{n}} = \norm{\freq_{n}}_{\star}$. Setting $\Param'_{n} = \frac{u_{n}}{n \norm{\freq_{n}}_{\star}}$ and $\Param=0$ we get $\inner{\freq_{n},\Param'_{n}} = 1/n$ and $\inner{\freq_{n},\Param}=0$, \rev{so that $\abs{e^{\jmath\inner{\freq_{n},\Param'_n}}-e^{\jmath\inner{\freq_{n},\Param}}} \stackrel{n\rightarrow \infty}{\sim} 1/n$}. Straightforward arguments then show that $\lim_{n \to \infty }\normrff{\psi(\Param'_{n})-\psi(\Param)}/\norm{\Param'_{n}-\Param} \geq 
\normkern{\Prob_{0}}^{-1} \cdot\normrffd{\Prob_{0}}$.
 \end{proof}

 In order to leverage Theorem~\ref{thm:LRIPGenericMixture}, we now exhibit $\lambda,\gamma$ such that~\eqref{eq:LRIPGenericMixtureMoments} holds (Lemma~\ref{lem:RFFMoments} below), and more concrete estimates for the covering numbers  $\covnum{\normrff{\cdot}}{\dipoleSet}{\coveps}$ (Lemma~\ref{lem:TangentLocationBased} below) which are
 pivotal to determine the required number of measurements.

 \begin{lemma}\label{lem:RFFMoments}
 Consider $\BasicSet = (\ParamSpace,\norm{\cdot},\embd)$ a location family built from a probability distribution $\Prob_{0}$.
 Consider $w$ a weight function, $\FClass = \set{\rfeat}_{\freq \in \Omega}$, $\freqdist$ a probability distribution on $\Omega$ and $\kernel$ the associated shift-invariant kernel as in Definition~\ref{def:weightedRFF}. 
 
 Assume that $\normkern{\Prob_{0}}>0$
  and let $\nkernel: \RR^{\sampleDim} \to \RR$ be the function associated to the $\BasicSet$-normalized version of $\kernel$ as in Proposition~\ref{prop:DefKShift}.  
 Assume there exists $a>0, b \geq 1/2$ such that 
 \begin{equation}\label{eq:MainAssumptionDipoleShift}
1-\nkernel(x) \geq \min(1,(\norm{x}/a)^{2})/b,\ \forall x\ \text{s.t.}\ \norm{x} \leq 1,
\end{equation}
and $\lambda_{0}>0$ such that for each $u \in \RR^{\sampleDim}$ such that $\norm{u} = 1$ and each integer $q \geq 2$ we have
\begin{equation}
\label{eq:MomentAssumptionFreqDist}
\Exp_{\freq \sim \freqdist} \left\{\abs{\inner{\Prob_{0},\rfeat}}^{2q} \cdot \inner{\freq,u}^{2q}\right\} \leq \normkern{\Prob_{0}}^{2}\tfrac{q!}{2} \lambda_{0}^{q}.
\end{equation}
Then for each integer $q \geq 2$ we have
\begin{equation}\label{eq:ConcentrationKMeansGMMFinal}
\sup_{\HH \in \dipoleSet}
\Exp_{\omega \sim \freqdist} \abs{\langle \HH,\rfeat\rangle}^{2q}  
\leq  
\frac{q!}{2} [\lambda \normkern{\Prob_{0}}^{-2}]^{q-1} \cdot \lambda
\qquad
\text{with}\ \lambda := \max(2b,1+ba^{2}\lambda_{0}/2) \geq 1.
\end{equation}
 \end{lemma}
The proof is in Appendix~\ref{app:MomentsDipoles}.

 \begin{remark}\label{rk:MomentsImprovedPossibility}  If $\kernel$ is $c$-strongly locally characteristic with respect to $\BasicSet$ with $0<c\leq 2$ then~\eqref{eq:MainAssumptionDipoleShift} holds with $b=2/c \geq 1$ and $a=1$. 
With specific choices of $\Prob_{0}$, $\freqdist$ discussed in Section~\ref{sec:DiracGaussian} we obtain finer estimates 
and provide concrete bounds for $\lambda_{0}$, $\normkern{\Prob_{0}}$, $\normrffd{\Prob_{0}}$, $a$ and $b$.
 \end{remark}

\begin{lemma}\label{lem:TangentLocationBased}
Let $\BasicSet$ be a location family based on a probability distribution $\Prob_{0}$ on $\RR^{\sampleDim}$ and a norm $\norm{\cdot}$. %
\rev{  Assume that the covering numbers of the base parameter space
  $\ParamSpace \subseteq \RR^\sampleDim$ satisfy, for some $C_\BasicSet\geq 1$,}
  \begin{equation}
    \label{eq:covparamspace2}
\rev{    \covnum{\norm{\cdot}}{\ParamSpace}{\delta}
    \leq \max\paren{1,\tfrac{C_\BasicSet}{\delta}}^{d}, \qquad \delta>0.}
  \end{equation}
Consider $\FClass$ a \rev{weighted random Fourier} feature family as in Definition~\ref{def:weightedRFF}, $\kernel$ the induced shift-invariant kernel, and $\dipoleSet = \dipoleSet_{\kernel}(\BasicSet)$ the induced set of normalized dipoles. 
Denote $\FClass' := \set{ \norm{\freq}_{\star} \rfeat}_{\freq \in \Omega}$ and $\FClass'' := \set{ \norm{\freq}^{2}_{\star} \rfeat}_{\freq \in \Omega}$. 
If $\kernel$ is $1$-strongly\footnote{the result is easily adjusted if $\kernel$ is $c$-strongly locally characteristic with $c<1$.} locally characteristic on $\BasicSet$ then, defining $D := \normrff{\dipoleSet} \geq 1$ and
 \begin{align*}
   C_{\FClass} %
   & =  \normkern{\Prob_{0}}^{-1} \normrff{\Prob_{0}} = \normkern{\Prob_{0}}^{-1};\\
   C'_{\FClass} %
   & =   \normkern{\Prob_{0}}^{-1} \normrffd{\Prob_{0}} = \normkern{\Prob_{0}}^{-1} \sup_{\freq} \set{\abs{\inner{\Prob_{0},\rfeat}}\norm{\freq}_{\star}};\\
   C''_{\FClass} %
   & =  \normkern{\Prob_{0}}^{-1} \normrffdd{\Prob_{0}} = \normkern{\Prob_{0}}^{-1} \sup_{\freq} \set{\abs{\inner{\Prob_{0},\rfeat}}\norm{\freq}^{2}_{\star}}
 \end{align*}
 it holds \begin{equation}
   \label{eq:maindipcovestimate}
   \covnum{\normrff{\cdot}}{\dipoleSet}{\coveps} \leq
     \rev{2\max\paren{1,\tfrac{64 C_\BasicSet(D C''_{\FClass}+C'_\FClass +C_\FClass)}{\delta} }^{4(d+1)}}, \qquad \delta>0.
   \end{equation}

\end{lemma}
The proof is in Appendix~\ref{app:TangentLocationBased}.

\subsection{Random Fourier sketching with a Gaussian kernel} %
\label{sec:DiracGaussian}
For the two scenarios of Sections~\ref{sec:clustering}-\ref{sec:gmm}, clustering and compressive GMM, the natural model set $\ModelCT(\HypClass)$ (resp. $\ModelML(\HypClass)$) is location-based, either built with Diracs ($\Prob_{0} = \delta_{0}$) or Gaussians ($\Prob_{0} = \mathcal{N}(0,\covar)$). For these scenarios, the following distribution $\freqdist$ of random frequencies is specifically
designed to lead to a Gaussian kernel %
when matched with weighted random Fourier features (Definition~\ref{def:weightedRFF}) using the same weight function. 

\begin{definition}[Frequency distribution]\label{def:freqdist}
Let $\mGamma \in \RR^{\sampleDim \times \sampleDim}$ be positive definite, and denote $p_{\mathcal{N}(0,\mGamma)}(\freq)$ the probability density function (pdf) of the centered Gaussian with covariance $\mGamma$. Given a weight function $w$ as in Definition~\ref{def:weightedRFF}, define a probability distribution $\freqdist$ on the frequency $\freq$ through the pdf
   \begin{equation}\label{eq:randomfeaturesamplingdensity}
      \freqdist(\freq) := C^{-2}_{\freqdist} w^{2}(\freq)\ p_{\mathcal{N}(0,\mGamma)}(\freq)\,,
      \end{equation}
where
\begin{equation}\label{eq:DefCfreqdist}
C_{\freqdist}  := \sqrt{\Exp_{\freq \sim \mathcal{N}\left(0,\mGamma\right)} w^{2}(\freq)},
\end{equation}
Since $\inf_{\freq}w(\freq)=1$ we have $C_\freqdist \geq 1$. When using the unit weight function $w \equiv 1$ we get $C_{\freqdist}=1$.
\end{definition}
From Definitions~\ref{def:weightedRFF} and~\ref{def:freqdist} one can build a random feature map $\SketchingOperator$ as in Section~\ref{sec:ingredientsLRIP}.
Its properties depend on the choice of the weight function $w$ (which will always be chosen identical in the definition of $\FClass$ and $\freqdist$) and of the covariance matrix $\mGamma$. Before discussing the choice of these parameters, one can immediately observe that the associated kernel is Gaussian: for any $\sample,\sample' \in \SampleSpace$ we have
\begin{align}
\kernel(\sample,\sample')
&=\Exp_{\freq\sim\freqdist} \rfeat(\sample)\overline{\rfeat(\sample')} 
= \int_{\freq \in \RR^{\sampleDim}} 
w^{-2}(\freq)
 e^{\jmath \freq^{T}(\sample-\sample')} 
  C_{\freqdist}^{-2} w^{2}(\freq) 
 p_{\mathcal{N}(0,\mGamma)}(\freq)
d\freq\notag 
\\
&=C_{\freqdist}^{-2} \cdot
\Exp_{\freq \sim \mathcal{N}\left(0,\mGamma\right)} 
e^{\jmath {\freq}^{T}(\sample-\sample')}
\stackrel{(*)}{=}
C_{\freqdist}^{-2} \cdot \exp\left(-\tfrac{\normmah{\sample-\sample'}{\mGamma^{-1}}^{2}}{2}\right),\label{eq:GaussianKernelRFF}
\end{align}
where (*) follows from the expression of the characteristic function of the Gaussian and $\normmah{c}{\mA} := \sqrt{c^T\mA^{-1}c}$ is the Mahalanobis norm \eqref{eq:DefMahalanobisNorm} given a positive definite matrix $\mA$. We focus on two scenarios.
\begin{definition}\label{def:DiracGaussian}
Consider $\sep>0$ some separation, $s>0$ some scale, and $\ParamSpace \subseteq \RR^{\sampleDim}$ some parameter space. We consider the following setting using Definitions~\ref{def:shiftmodel},\ref{def:weightedRFF},\ref{def:freqdist}.
\begin{itemize}
\item {\bf for mixtures of Diracs}: $\BasicSet_{\mathtt{Dirac}}$ is defined with $\Prob_{0} = \delta_{0}$, and $\norm{\cdot} = \norm{\cdot}_{2}/\sep$;
 $\FClass_{\mathtt{Dirac}}$ is defined with a weight function $w(\cdot)$ to be discussed; $\freqdist_{\mathtt{Dirac}}$ is defined with the same $w(\cdot)$ and $\mGamma = s^{-2}\mI_{\sampleDim}$.
In several places we focus more specifically on $\ParamSpace = \ParamSpace_R  := \Ball_{\RR^\sampleDim,\norm{\cdot}_{2}}(0,R)$, where $R \geq \sep$. 
\item {\bf for mixtures of Gaussians}: define $\BasicSet_{\mathtt{Gauss}}$ with $\Prob_{0} = \mathcal{N}(0,\covar)$ for some chosen positive definite $\covar \in \RR^{\sampleDim \times \sampleDim}$ and $\norm{\cdot} = \normmah{\cdot}{\covar}/\sep$; $\FClass_{\mathtt{Gauss}}$ is defined with $w(\cdot)\equiv 1$; $\freqdist_{\mathtt{Gauss}}$ is defined with the same $w(\cdot)$ and  $\mGamma = s^{-2} \covar^{-1}$. Again
for some results we will focus on $\ParamSpace_R  := \Ball_{\RR^\sampleDim,\normmah{\cdot}{\covar}}(0,R)$, where $R \geq \sep$. 
\end{itemize}
Observe that in both cases we have the identity $\ParamSpace_R  = \Ball_{\RR^\sampleDim,\norm{\cdot}}(0,R/\sep)$, and that $\norm{\Param-\Param'} = \norm{\psi(\Param)-\psi(\Param')}_{2}$ with $\psi_{\mathtt{Dirac}}(\Param) = \Param/\sep$, while $\psi_{\mathtt{Gauss}}(\Param) = \covar^{-1/2}\Param/\sep$. %
\end{definition}

\subsubsection{Properties of the kernel mean embedding}\label{sec:mmddiracgauss}
Before we state the main theorem of this section, we make a few observations. For the Dirac scenario, since $\mGamma = s^{-2}\mI_{\sampleDim}$, the kernel mean embedding associated to~\eqref{eq:GaussianKernelRFF} satisfies
\begin{equation*}
\kernel(\Prob_{\Param},\Prob_{\Param'}) = \kernel(\Param,\Param') = 
C_{\freqdist}^{-2} \cdot \exp\left(-\tfrac{1}{2}\cdot\tfrac{\norm{\Param-\Param'}_{2}^{2}}{s^{2}}\right)
=
C_{\freqdist}^{-2} \cdot \exp\left(-\tfrac{1}{2}\cdot \tfrac{\norm{\Param-\Param'}^{2}}{(s/\sep)^{2}}\right)
\end{equation*}
For Gaussians, as $w \equiv 1$ we have $C_{\freqdist}=1$. By Lemma \ref{lem:GaussKernelMeanEmbedding} (see Appendix~\ref{app:KernelGaussian})  with $\mR = \mGamma^{-1} = s^{2} \covar$ we obtain
\begin{align*}
\kernel(\Prob_\Param,\Prob_{\Param'}) 
=& 
\tfrac{\sqrt{\det{s^2\covar}}}{\sqrt{\det{(2+s^2)\covar}}}
\exp\left( -\tfrac12  \normmah{\Param-\Param'}{(2+s^2)\covar}^2\right)\notag\\
=&
\left(\tfrac{s^2}{2+s^2}\right)^{\sampleDim/2} 
\exp\left( -\tfrac{1}{2} \cdot \tfrac{\normmah{\Param-\Param'}{\covar}^2}{2+s^2}\right)
=
\left(\tfrac{1}{1+2/s^2}\right)^{\sampleDim/2} 
\exp\left( -\tfrac{1}{2} \cdot \tfrac{\norm{\Param-\Param'}^{2}}{(2+s^2)/\sep^{2}}\right).
\end{align*}
This yields
\begin{equation}\label{eq:P0NormInvSq}
\normkern{\Prob_{0}} = 
\begin{cases}
C_{\freqdist}^{-1} = \left[\Exp_{\freq \sim \mathcal{N}(0,s^{-2} \mI_{\sampleDim})} w^{2}(\freq)\right]^{-1/2}, & \text{for Diracs},\\
\left(1+2/s^2\right)^{-\sampleDim/4},& \text{for Gaussians}.
\end{cases}
\end{equation}
In both cases we have $\normkern{\Prob_{0}} \leq 1$, and we also get
\begin{equation}
\nkernel(\Param-\Param') := 
\tfrac{\kernel(\Prob_{\Param},\Prob_{\Param'})}{\normkern{\Prob_{\Param}}\normkern{\Prob_{\Param'}}}
=
K_{\sigma(s)/\sep}(\norm{\Param-\Param'}),\quad \forall \Param,\Param' \in \RR^{\sampleDim},
\label{eq:RBFLikeKernel}
\end{equation}
with
 \begin{equation}
   \label{eq:DefDiracGaussianVariance}
   K_{\sigma}(u) := e^{-\frac{u^2}{2\sigma^2}},\quad u \geq 0; \qquad \text{ and } \qquad
\sigma(s) := \begin{cases}
s, & \text{for Diracs}\\
\sqrt{2+s^2}, & \text{for Gaussians.}
\end{cases}
\end{equation}
The following properties of Gaussian kernels are proved in Appendix~\ref{sec:proofLemmaSigmaK}. \rev{The first property allows to use Lemma~\ref{lem:RFFMoments}. The second shows that the kernel is $1$-strongly locally characteristic. The third one shows that its $2k$-coherence is below $3/4$
  provided  $\sigma \lesssim 1/\sqrt{\log k}$.}
\begin{lemma}\label{lem:GaussianKernel}
Consider $\sigma>0$.
\begin{enumerate}
\item \label{it:GaussCurvature} We have $1-K_{\sigma}(u) \geq \min(1,(u/\sigma)^{2})/3$
  for all $u\geq 0$;
\item \label{it:GaussStrongly} If $\sigma \leq 1/\sqrt{2}$ then $0 \leq K_{\sigma}(u) \leq 1-u^{2}/2$ for $u \in [0,1]$; %
\item \label{it:GaussCoherence}  \rev{ Given an integer $k \geq 1$,
  for any $\sigma \leq \sigma^\star_k:= (4 \sqrt{\log(ek)})^{-1}$,
 it holds $C(K_{\sigma}) \leq \frac{3}{\rev{16}(2k-1)}$, where $C(K)$ is defined in~\eqref{eq:MutualCoherenceRBF1}. }
\end{enumerate}
\end{lemma}
Using the generic tools of Section~\ref{sec:general_mixtures} with the model set $\MixSetSep{k}$ we can establish the following result.
\begin{theorem}\label{thm:LRIPDiracGaussMixture}
Consider $\FClass_{\mathtt{Dirac}},\freqdist_{\mathtt{Dirac}},\BasicSet_{\mathtt{Dirac}}$ (resp. $\FClass_{\mathtt{Gauss}},\freqdist_{\mathtt{Gauss}},\BasicSet_{\mathtt{Gauss}}$) as in Definition~\ref{def:DiracGaussian} with separation $\sep$, scale $s$ and weight $w$
such that 
\begin{equation}
\sep \rev{=} \begin{cases}
s / \sigma^{\star}_{k}, & \text{for Diracs};\\
\sqrt{2+s^{2}} / \sigma^{\star}_{k}, & \text{for Gaussians}.
\end{cases}\label{eq:MainSeparationAssumption}
\end{equation}
where $k \geq 1$ is an integer and \rev{$\sigma_{k}^{\star} := (4 \sqrt{\log(ek)})^{-1}$.} Then, the associated kernel $\kernel$ is 1-strongly characteristic and has its $2k$-coherence bounded by $3/4$. \\
Further assume that there exists\footnote{If needed, consider $C'_{\ParamSpace} = \max(C_{\ParamSpace},\sep)$.} $C_\ParamSpace \geq \rev{\sep}$ such that the parameter space $\ParamSpace \subseteq \RR^{\sampleDim}$ satisfies
\begin{equation}
  \label{eq:covparamspace1}
  \max\paren{1,\tfrac{C_\ParamSpace}{\delta}}^{\sampleDim} \geq
  \begin{cases}
  \covnum{\norm{\cdot}_2}{\ParamSpace}{\delta},  & \text{for Diracs},\\
  \covnum{\normmah{\cdot}{\Cov}}{\ParamSpace}{\delta},  & \text{for Gaussians},\\
\end{cases},\qquad \forall\coveps>0,
\end{equation}   
and denote
\begin{equation}
A := \begin{cases} 
 \Exp_{\freq \sim \mathcal{N}(0,s^{-2}\mI_{\sampleDim})} w^{2}(\freq), & \text{for Diracs,}\\
 (1+2/s^{2})^{\sampleDim/2}, & \text{for Gaussians,}
\end{cases} 
\qquad \text{and}\ 
B :=
\begin{cases}
\rev{1}+\sep^{2} \left(\sup_{\freq} \tfrac{\norm{\freq}_{2}}{w(\freq)}\right)^{2}, & \text{for Diracs,}\\
\rev{1}+\sep^{2}, & \text{ for Gaussians;}
\end{cases}\label{eq:DefAB}
\end{equation}
\begin{equation}
  \label{eq:DefD}
  \rev{C}:=
  \begin{cases}
    64 \rev{A \sqrt{2B}} C_\ParamSpace \sep^{-1}  \paren{1 + \sep \sup_\freq \tfrac{\norm{\freq}_{2}}{w(\freq)} +  \sep^2
      \sup_\freq \tfrac{\norm{\freq}^2_{2}}{w(\freq)}}, & \text{for Diracs,} \\
    64 A \sqrt{\rev{2}B} C_\ParamSpace \sep^{-1}  (1 + \sep + \sep^2), & \text{for Gaussians.}
    \end{cases}
\end{equation}
Consider $0<\delta,\zeta<1$ and
\(
\SketchingOperator(\sample) := \tfrac{1}{\sqrt{m}} \left(\rfeatj(\sample)\right)_{j=1}^{m},
\)
where $(\freq_j)_{j=1}^{m}$ are drawn i.i.d. according to $\Lambda$. If
\begin{equation}\label{eq:NMeasuresDiracGauss}
\nMeasures \geq 
80 \cdot \coveps^{-2} \cdot A \cdot \min(12e\log^{2}(\rev{24ek}),\rev{2}B)
\cdot k \cdot \left\{
2k (4d+5) \left[ \log (k\rev{C}AB) + \log(\rev{2048}/\coveps)\right]
+\log(2/\zeta)\right\},
\end{equation}
then
with probability at least $1-\probLevel$ on the draw of $(\freq_{j})_{j=1}^{m}$ the operator $\SketchingOperatorProb$ induced by $\SketchingOperator$ satisfies 
\begin{equation}\label{eq:mainLRIPKernelDiracGauss}
1-\delta \leq 
\frac{\norm{\SketchingOperatorProb(\Prob)-\SketchingOperatorProb(\Prob')}_{2}^{2}}{\norm{\Prob-\Prob'}_{\kappa}^{2}}
\leq 1+\delta, \qquad \forall \Prob,\Prob' \in \MixSetSep{k}.
\end{equation}
When~\eqref{eq:mainLRIPKernelDiracGauss} holds, the LRIP~\eqref{eq:lowerDRIP} holds with 
$C_\SketchingOperatorProb :=\frac{8\sqrt{2k}\dnormloss{\dipoleSet}{}}{\sqrt{1-\delta}}$ and $\eta = 0$
for each loss class $\LossClass$. 
\end{theorem}
\rev{\begin{remark}
In light of the separation assumption~\eqref{eq:MainSeparationAssumption} for Gaussians, which implies $\sep \geq \sqrt{2}/\sigma^{\star}_{k} =4\sqrt{2\log(ek)}$, the dominant term of the rightmost factor in $C$ (cf~\eqref{eq:DefD}) is $\sep^{2}$ for Gaussians. In contrast, for Diracs, the rightmost factor in $C$ can become arbitrarily close to $1$ by setting $\sep$ (and $s$) small enough. 
\end{remark}}
\begin{proof}
  The proof consists in checking the assumptions of the generic Theorem~\ref{thm:LRIPGenericMixture}
  using the more concrete estimates obtained previously in this section in the mixture of Dirac and Gaussian settings.
  
{\bf Step 1: control of the coherence.}
Assumption~\eqref{eq:MainSeparationAssumption} means that $\sigma(s)/\sep \rev{=} \sigma^{\star}_{k}$ with $\sigma(s)$ as in~\eqref{eq:DefDiracGaussianVariance}. By Lemma~\ref{lem:GaussianKernel} it follows that $K_{\sigma(s)/\sep}$ satisfies $0 \leq K_{\sigma(s)/\sep}(u) \leq 1-u^{2}/2$ for $0 \leq u \leq 1$ and that $(2k-1)C(K_{\sigma/\sep}) \leq 3/\rev{16}$ with $C(K)$ defined in~\eqref{eq:MutualCoherenceRBF1}. 
Since $\norm{\Param-\Param'} = \norm{\psi(\Param)-\psi(\Param')}_{2}$ with $\psi(\Param) = \Param/\sep$ for Diracs (resp. $\psi(\Param) := \covar^{-1/2}\Param/\sep$ for Gaussians), by property~\eqref{eq:RBFLikeKernel} and Theorem~\ref{thm:MutualCoherenceRBF1} we obtain that $\kernel$ is $1$-strongly locally characteristic with respect to  $\BasicSet$, and that
$\kernel$ has coherence $M$ with respect to $\BasicSet$, where $(2k-1)M \leq (2k-1)\rev{4}C(K_{\sigma/\sep}) \leq 3/4$. 

{\bf Step 2: control of $\normrff{\dipoleSet},\normkern{\Prob_0},\normfclass{\Prob_0}{G}, \mathcal{G} \in
  \set{\FClass,\FClass',\FClass''}$.} 
In light of property~\eqref{eq:P0NormInvSq}, assumption~\eqref{eq:DefAB} and Lemma~\ref{lem:ShiftBasedRF} we have $\normrff{\monopoleSet}^{2} = \normrff{\Prob_{0}}^{2} \cdot \normkern{\Prob_{0}}^{-2}= \normkern{\Prob_{0}}^{-2} = A$. To control $L_{\FClass}=C_\FClass'$ \rev{and $C''_{\FClass}$} as in Lemmas~\ref{lem:ShiftBasedRF} and~\ref{lem:TangentLocationBased} we compute:
\begin{itemize}
\item for Diracs: as $\norm{\cdot} = \norm{\cdot}_{2}/\sep$ we have $\norm{\cdot}_{\star} = \sep \norm{\cdot}_{2}$, and since $\Prob_{0} = \delta_{0}$ we get
\begin{align*}
\normrffd{\Prob_{0}} & 
:= \sup_{\freq} \abs{\inner{\Prob_{0},\norm{\freq}_{\star}\rfeat}} 
= \sup_{\freq} \norm{\freq}_{\star}\abs{\rfeat(0)} 
                       = \sep \sup_{\freq} \frac{\norm{\freq}_{2}}{w(\freq)};\\
  \normfclass{\Prob_{0}}{\FClass''} & 
:= \sup_{\freq} \abs{\inner{\Prob_{0},\norm{\freq}^2_{\star}\rfeat}} 
= \sup_{\freq} \norm{\freq}^2_{\star}\abs{\rfeat(0)} 
= \sep^2 \sup_{\freq} \frac{\norm{\freq}^2_{2}}{w(\freq)};
\end{align*}
\item for Gaussians: since $\norm{\cdot} = \normmah{\cdot}{\covar}/\sep$ we have $\norm{\cdot}_{\star} = \sep \normmah{\cdot}{\covar^{-1}}$, and since $w (\freq)\equiv 1$, by the expression of the characteristic function of Gaussians we have $\abs{\inner{\Prob_{0},\rfeat}} = e^{-\normmah{\freq}{\covar^{-1}}^{2}/2}$,  hence
\begin{align*}
\normrffd{\Prob_{0}} & 
:= \sup_{\freq} \abs{\inner{\Prob_{0},\norm{\freq}_{\star}\rfeat}} 
= \sep \sup_{\freq} \normmah{\freq}{\covar^{-1}}e^{-\normmah{\freq}{\covar^{-1}}^{2}/2} 
                       = \sep \sup_{u \geq 0} u e^{-u^{2}/2} = \sep e^{-1/2} \leq \sep;\\
  \normfclass{\Prob_{0}}{\FClass''} & 
:= \sup_{\freq} \abs{\inner{\Prob_{0},\norm{\freq}^2_{\star}\rfeat}} 
= \sep^2 \sup_{\freq} \normmah{\freq}{\covar^{-1}}^2e^{-\normmah{\freq}{\covar^{-1}}^{2}/2} 
= \sep^2 \sup_{u \geq 0} u e^{-u/2} = \sep^2 \cdot 2/e \leq \sep^2.
\end{align*}
\end{itemize}
Since $\kernel$ is $1$-strongly locally characteristic, by Theorem \ref{thm:RadiusGeneric}, Lemma~\ref{lem:ShiftBasedRF},~\eqref{eq:DefAB}  and $(a+b)^{2} \leq 2(a^{2}+b^{2})$
it follows that 
\[
\normrff{\dipoleSet}^{2} \leq 2\normrff{\monopoleSet}^{2}+\rev{2}L_{\FClass}^{2} \rev{=} 2A\rev{(1+ \normrffd{\Prob_{0}}^{2})} 
 = \rev{2}AB.
\]

{\bf Step 3: control of normalized dipole moments.} 
We show in Appendix~\ref{app:MomentsDipolesDiracGauss} that in both settings, for each $u \in \RR^{\sampleDim}$ such that $\norm{u}=1$ and each integer $q \geq 2$, we have
\begin{eqnarray}\label{eq:MomentsDiracGaussianTechnical}
  \Exp_{\freq \sim \freqdist}\set{ \abs{\inner{\Prob_{0},\rfeat}}^{2q} \inner{\freq,u}^{2q}}
& \leq \normkern{\Prob_{0}}^{2}(2\sep^{2}/\sigma^{2}(s))^{q} \tfrac{q!}{2}.
\end{eqnarray}
This proves~\eqref{eq:MomentAssumptionFreqDist} with 
\rev{$\lambda_{0} = 2\sep^{2}/\sigma^{2}(s) = 2/(\sigma^{\star}_{k})^{2}$.}
By Lemma~\ref{lem:GaussianKernel}-\ref{it:GaussCurvature} applied to $K_{\sigma(s)/\sep}$ and property~\eqref{eq:RBFLikeKernel} we obtain that $\nkernel(x) = K_{\sigma(s)/\sep}(\norm{x})$ satisfies  \eqref{eq:MainAssumptionDipoleShift}
with $a = \sigma(s)/\sep=\rev{\sigma^{\star}_{k}}$ and $b=3$.  By Lemma~\ref{lem:RFFMoments}, since 
$1+ba^{2}\rev{\lambda}_{0}/2 = 4$ 
while $2b = 6$, we get  that 
\[
\sup_{\HH \in \dipoleSet} \Exp_{\omega \sim \freqdist} \abs{\inner{\HH,\rfeat}}^{2q}  
\leq  
\frac{q!}{2} [6\normkern{\Prob_{0}}^{-2}]^{q-1} \cdot 6.
\]
It follows that the assumption~\eqref{eq:LRIPGenericMixtureMoments} of Theorem~\ref{thm:LRIPGenericMixture} holds with $\gamma = 6\normkern{\Prob_{0}}^{-2} = 6A$, \rev{$\lambda=6$.}
\begin{remark}
Since $\kernel$ is $1$-strongly locally characteristic (cf Lemma~\ref{lem:GaussianKernel}-\ref{it:GaussStrongly}), the generic arguments in Remark~\ref{rk:MomentsImprovedPossibility} show that $\nkernel$ also satisfies \eqref{eq:MainAssumptionDipoleShift} with $a' = 1$ and $b' = 2$.  Notice that here since $\sigma(s)/\sep \rev{=} \sigma^{\star}_{k}$ we have $b a^{2} = 3 (\sigma^{\star}_{k})^{2} = \order{1/\log(ek)} \ll 2 = b'(a')^{2}$ for large $k$, showing that Lemma~\ref{lem:GaussianKernel}-\ref{it:GaussCurvature} indeed allows to improve over the generic result of Remark~\ref{rk:MomentsImprovedPossibility}.
\end{remark}
{\bf Step 4: covering numbers for $\dipoleSet$.}
To control those covering numbers, we apply Lemma~\ref{lem:TangentLocationBased}. The final estimate~\eqref{eq:maindipcovestimate}  from this lemma involves the quantity
$64  C_\BasicSet (C_\FClass + C'_\FClass + \normrff{\dipoleSet} C''_\FClass)$,
which we now show is bounded by $\rev{C}$ defined in~\eqref{eq:DefD} in both cases,
using the estimates from Step 2 for the various constants.
Since $\norm{\cdot} = \norm{\cdot}_2/\sep$ (for Diracs), resp. $\norm{\cdot} = \normmah{\cdot}{\covar}/\sep$ (for Gaussians), Assumption~\eqref{eq:covparamspace1} implies that Assumption~\eqref{eq:covparamspace2} of  Lemma~\ref{lem:TangentLocationBased}
holds with $C_\BasicSet:=\sep^{-1} C_\ParamSpace$.

{\em For Diracs:}
 From the estimates in Step 2 \rev{and the fact that $A \geq 1, B \geq 1$} we obtain
\begin{align*}
 64 C_\BasicSet (C_\FClass + C'_\FClass + \normrff{\dipoleSet} C''_\FClass)
  & = 64  C_\ParamSpace \sep^{-1}  \normkern{\Prob_0}^{-1} (\normrff{\Prob_0} + \normrffd{\Prob_0} + \normrff{\dipoleSet}  \normfclass{\Prob_0}{\FClass''} )\\
  & \leq 64 C_\ParamSpace \sep^{-1} \sqrt{A}  \paren{1 + \sep\sup_{\freq} \tfrac{\norm{\freq}_{2}}{w(\freq)} +\sqrt{ \rev{2}AB} \sep^2
  \sup_{\freq} \tfrac{\norm{\freq}^2_{2}}{w(\freq)}} \leq \rev{C}.
  \end{align*}

  {\em For Gaussians:}
 From the estimates in Step 2 we obtain
similarly to the Dirac case
\begin{align*}
   64  C_\BasicSet (C_\FClass + C'_\FClass + \normrff{\dipoleSet} C''_\FClass)
 \leq 64 C_\ParamSpace \sep^{-1}  A \sqrt{\rev{2}B}  (1 + \sep + \sep^2) \rev{=C}.
  \end{align*}

By Lemma~\ref{lem:TangentLocationBased} we obtain
 $\covnum{\normfclass{\cdot}{F}}{\dipoleSet}{\coveps}  \leq \rev{2}(\rev{C}/\coveps)^{r}$  for each $0<\coveps<1$, with $r=4(d+1)$.
  Since $\rev{C}\geq 1$ %
  ($A,B,C_\ParamSpace\rev{\sep}^{-1}$ being greater than 1), the  covering number assumption of Theorem~\ref{thm:LRIPGenericMixture} is satisfied.

{\bf Step 5: wrapping up.}  
Combining the above ingredients we can apply Theorem~\ref{thm:LRIPGenericMixture}.
We have
\[
 \min\left(2e\gamma \log^{2}(4ek\rev{\lambda}),\normfclass{\dipoleSet}{F}^2\right)
 \leq A \min(12e\log^{2}(24ek), \rev{2}B);
\]
combined with~\eqref{eq:NMeasuresDiracGauss}, the above estimates imply that~\eqref{eq:DefNMeasuresGenericMixture} holds. We conclude  using Theorem~\ref{thm:LRIPGenericMixture}.
\end{proof}

\subsection{Additional ingredients to establish Theorem~\ref{thm:mainkmeansthm} and Theorem~\ref{thm:maingmmthm}}
\rev{The detailed proofs of Theorem~\ref{thm:mainkmeansthm} (resp. of Theorem~\ref{thm:maingmmthm}) are given in Appendix~\ref{sec:separatedornot}. }
The proofs use as an intermediate tool a constrained hypothesis class $\HypClass$ such that the model set $\ModelCT(\HypClass)$ \rev{(resp. $\ModelML(\HypClass)$)} corresponds to $k$-separated mixtures of Diracs (resp. of Gaussians). As both compressive clustering (see Section~\ref{sec:clustering}) and compressive Gaussian Mixture Modeling (see Section~\ref{sec:gmm}) rely on (reweighted) random Fourier features to design the sketching function $\SketchingOperator$, this allows to leverage Theorem~\ref{thm:LRIPDiracGaussMixture} combined with Theorem~\ref{thm:LRIPsuff_excess} to establish intermediate results. 

Additional ingredients to complete the proofs include: 
\begin{itemize}
\item the constant $\norm{\dipoleSet}_{\DLossClass(\HypClass)}$ \rev{determining the LRIP constant $C_{\SketchingOperatorProb}$ established as the product of Theorem~\ref{thm:LRIPDiracGaussMixture}.
  This constant} is controlled in Appendix~\ref{subsec:dnormlossclustgmm};
\item the bias term from Theorem~\ref{thm:LRIPsuff_excess}, which is  bounded \rev{by a more explicit estimate} in Appendix~\ref{subsec:biasclustgmm};
\end{itemize}
The theorems are proved in Appendix~\ref{subsec:clustfinalproof}.

\section{Conclusion and perspectives}\label{sec:future}

The principle of compressive statistical learning is to learn from large-scale collections by first summarizing the collection into a sketch vector made of empirical (random) moments, before solving a nonlinear least squares problem. The main contribution of this paper is to %
demonstrate on %
\rev{two examples (}compressive clustering and compressive Gaussian mixture estimation --with fixed known covariance) that the excess risk of this procedure can be controlled, as well as the sketch size. 

\paragraph{Sharpened estimates?} Our demonstration of the validity of the compressive statistical learning framework for certain tasks is, in a sense, qualitative, and we expect that many bounds and constants are sub-optimal. 
This is the case for example of the estimated sketch sizes for which statistical learning guarantees have been established, 
and an immediate theoretical challenge is to sharpen these guarantees to match the empirical phase transitions observed empirically for compressive clustering and compressive GMM \citep{keriven:hal-01329195,Keriven2016}. 
For mixture models, as our proof technique involves Geshgorin's disc theorem, it is natural to wonder to what extent the involved constants can be tightened to get closer to sharp oracle inequalities, possibly at the price of larger sketch sizes. 
Overall, an important question to benchmark the quality of the established bounds (on achievable sketch sizes, on the separation assumptions used for $k$-mixtures, etc.) is of course to investigate corresponding lower-bounds.

\paragraph{Provably-good algorithms of bounded complexity?}
As the control of the excess risk relies on the \rev{(approximate)} minimizer of a nonlinear least-squares problem~\eqref{eq:ThmDecoder2}, the results in this paper are essentially information-theoretic. Can we go beyond the heuristic optimization algorithms derived for compressive $k$-means and compressive GMM \citep{keriven:hal-01329195,Keriven2016} and characterize provably good, computationally efficient algorithms to obtain this minimizer ? 

Promising directions revolve around recent advances in super-resolution imaging and low-rank matrix recovery. For compressive clustering (resp. compressive GMM), the similarity between  \rev{the minimization of~\eqref{eq:RiskProxyKMeans} (resp.~\eqref{eq:RiskProxyGMM})} %
and super-resolution imaging suggests to explore TV-norm minimization --a \emph{convex} problem-- techniques \citep{Candes_2013,Castro_2015,Duval_2015} and to seek generalized RIP guarantees \citep{Traonmilin2016}. Further, to circumvent the difficulties of optimization (convex or not) in the space of finite signed measures, it may also be possible to adapt the recent guarantees obtained for certain nonconvex problems that directly leverage a convex ``lifted'' problem  \citep[see e.g.][]{Li:2016vz} without incurring the cost of actually computing in the lifted domain.

Finally, the computational cost of sketching itself \rev{can} be further controlled \rev{\citep{chatalic:hal-01701121}} by replacing random Gaussian weights where possible with fast approximations \citep{Le2013,Choromanski:2016un,Bojarski:2016tv}. \rev{This results} in accelerations of the learning stage wherever matrix multiplications are exploited. To conduct the theoretical analysis of the resulting sketching procedure, one will need to analyze the kernels associated to these fast approximations.

\section*{Acknowledgements}

This work was supported in part by the European Research Council, PLEASE project (ERC-StG-2011-277906),
the german DFG (SFB-1294 “Data Assimilation”), the Franco-German University through the binational Doktorandenkolleg CDFA 01-18, and the ANR (ANR-19-CHIA-0021-01, project BISCOTTE; ANR-19-CHIA-0009, project AllegroAssai). 
R{\'e}mi Gribonval is very grateful to Michael E. Davies for many enlightening discussions around the idea of compressive statistical learning since this project started several years ago. The authors also wish to warmly thank Bernard Delyon and Adrien Saumard, as well as Gabriel Peyr{\'e} and Lorenzo Rosasco for their constructive feedback on early versions of this manuscript.

\pagebreak
\part*{Appendix}
We begin by introducing notations and useful results. We then provide general properties on covering numbers, followed by properties that are shared by any model of mixtures of distributions that are sufficiently separated $\Model=\MixSetSep{k}$. We then apply these results to mixtures of Diracs and both $k$-medians and $k$-means risks, and to Gaussian Mixture Models with fixed known covariance for maximum likelihood estimation. %

\appendix
\section{Generalities on covering numbers}
\label{app:covering}
In this section we formulate generic results on covering numbers.
\subsection{Basic properties}
The definition used in this paper is that of \emph{internal} covering numbers, meaning that the centers $z_{i}$ of the covering balls are required to be included in the set $Y$ being covered. Somehow counter-intuitively these covering numbers (for a fixed radius $\coveps$) are not necessarily increasing with the inclusion of sets: for instance, consider a set $A$ formed by two points, included in set $B$ which is a ball of radius $\coveps$. Suppose those two points diametrically opposed in $B$. We have $A\subset B$, but two balls of radius $\coveps$ are required to cover $A$ (since their centers have to be in $A$), while only one such ball is sufficient to cover $B$. Yet, as shown by the following lemma, the covering numbers of included sets still behave in a controlled manner.

\begin{lemma}\label{lem:covnumsub}
Let $A \subseteq B \subseteq X$ be subsets of a pseudometric set $(X,d)$, and $\coveps>0$. Then,
\begin{equation}
\label{eq:covnumsub}
\covnum{d}{A}{\coveps} \leq \covnum{d}{B}{\coveps/2}.
\end{equation}
\end{lemma}

\begin{proof}
  Let $(b_i)_{1 \leq i \leq N}$ be a $\rev{\coveps/2}$-covering of $B$. We construct a $\coveps$-covering
  $(a_i)_{i \in I}$ of $A$
  \rev{of cardinality at most $N$} in the following way. \rev{For each $i=1,\ldots,N$, consider $C_i:=\Ball_{X,d}(b_i,\coveps/2)\cap A$.
    If $C_i \neq \emptyset$, we replace $b_i$ by an arbitrary point $a_i \in C_i$,
    otherwise we discard $b_i$. Note that in the first case, by the triangle inequality
    $C_i \subset \Ball_{X,d}(b_i,\coveps/2) \subset \Ball_{X,d}(a_i,\coveps)$.
    On the other hand, by the covering property, $\bigcup_{1\leq i \leq N} C_i = A$.
  }Therefore the set of $a_i$s is a $\coveps$-covering of $A$. %
\end{proof}

\begin{lemma}\label{lem:covnumlipschitz}
Let $(X,d)$ and $(X',d')$ be two pseudometric sets, and $Y \subseteq X$, $Y' \subseteq X'$. If there exists a surjective function $f: Y \rightarrow Y'$ which is $L$-Lipschitz with $L>0$, i.e. such that
\begin{equation*}
\forall x,y\in Y,~d'(f(x),f(y)) \leq L d(x,y),
\end{equation*}
then for all $\coveps>0$ we have
\begin{equation}
\label{eq:covnumlipschitz}
\covnum{d'}{Y'}{\coveps} \leq \covnum{d}{Y}{\coveps/L}.
\end{equation}
\end{lemma}

\begin{proof}
Define $\coveps_2=\coveps/L$, denote $N=\covnum{d}{Y}{\coveps_2}$, and let $y_i \in Y$, $i=1,...,N$ be a $\coveps_2$-covering of $Y$. Consider $y' \in Y'$. There exists $y \in Y$ such that $f(y)=y'$ since $f$ is surjective. For some $1 \leq i \leq N$ we have $d(y,y_i) \leq \coveps_{2}$, hence we have
\begin{equation*}
d'(y',f(y_i)) = d'(f(y),f(y_i)) \leq Ld(y,y_i) \leq L\coveps_2 = \coveps.
\end{equation*}
Thus $\{f(y_i)\}_{i=1,...,N}$ is a $\coveps$-covering of $Y'$, and we have $\covnum{d'}{Y'}{\coveps} \leq N$.
\end{proof}

\begin{lemma}\label{lem:covnumClose}
Let $Y,Z$ be two subsets of a pseudometric set $(X,d)$ \rev{and $\epsilon \geq 0$} such that the following holds:
\begin{equation}\label{eq:covnumCloseHyp}
\forall z \in Z,~\exists y \in Y,~d(z,y) \leq \epsilon.
\end{equation}
Then for all $\coveps>0$
\begin{equation}\label{eq:covnumClose}
\covnum{d}{Z}{2(\coveps+\epsilon)} \leq \covnum{d}{Y}{\coveps}.
\end{equation}
\end{lemma}

\begin{proof}
Denote $N=\covnum{d}{Y}{\coveps}$ and let $y_1,...,y_N \in Y$ be a $\coveps$-covering of $Y$. For all $z\in Z$, by the assumption~\eqref{eq:covnumCloseHyp} there is $y\in Y$ such that $d(z,y)\leq\epsilon$, and subsequently there is an index $i$ such that
$d(z,y_i)\leq d(z,y) + d(y,y_i) \leq \coveps + \epsilon.$
This implies $Z\subset \bigcup_{i=1}^N \Ball_{X,d}(y_i,\coveps+\epsilon)$, hence by Lemma \ref{lem:covnumsub} 
\begin{equation*}
\covnum{d}{Z}{2(\coveps+\epsilon)} \leq \covnum{d}{\bigcup_{i=1}^N \Ball_{X,d}(y_i,\coveps+\epsilon)}{\coveps + \epsilon} \leq N.
\end{equation*}
\end{proof}

\begin{lemma}[\cite{Cucker2002}, Prop. 5]\label{lem:covnumball}
Let $(X,\norm{\cdot})$ be a Banach space of finite dimension $\sampleDim$. Then for any $x\in X$ and $R>0$ we have for any $\coveps>0$
\begin{equation}
\covnum{\norm{\cdot}}{\Ball_{X,\norm{\cdot}}(x,R)}{\coveps}\leq \max\left(1,\left(\frac{4R}{\coveps}\right)^\sampleDim\right)
\end{equation}
\end{lemma}
NB: The result in \cite[Prop. 5]{Cucker2002} does not include $\max(1,\cdot)$. This obviously cannot hold for $\coveps > 4R$ since the left hand side is at least one. The proof of \cite[Prop. 5]{Cucker2002} yields the result stated here.

\subsection{``Clipped'' Secant set}\label{sec:covnumExtrudedSecant}

\newcommand{\norma}[1]{\norm{#1}_a}
\newcommand{\normb}[1]{\norm{#1}_b}
\newcommand{\normaa}[1]{\norm{#1}_{a'}}
\newcommand{\normbb}[1]{\norm{#1}_{b'}}

To control the covering numbers of the normalized secant set~\eqref{eq:DefNormalizedSecantSet}, \rev{or those of the normalized dipole set $\dipoleSet$ \eqref{eq:NormalizedDipoleSet}} it will be convenient to control those of \rev{certain subsets of its} normalized secant set. 

\begin{lemma}\label{lem:covnumsecant}
\rev{Consider $X$ a vector space, $\norma{\cdot},\normb{\cdot}: X \to [0,+\infty]$ two semi-norms, $X_{a},X_{b} \subseteq X$ the subspaces where they are finite, and $Y \subseteq X_{a}$. Consider $\mathcal{Q} \subseteq Y^{2}$ and assume that for some constants $0<A \leq B<\infty$,
\begin{equation}
\label{eq:SecSetLemmaNorms}
\forall \rev{(y,y') \in \mathcal{Q}},~A\normb{y-y'}\leq \norma{y-y'} \leq B\normb{y-y'} <\infty.
\end{equation}
Given $\eta>0$, consider the following subset of the normalized secant of $Y$:
\begin{equation*}
\secant_{\rev{\eta}} := \left\lbrace\frac{y-y'}{\normb{y-y'}} ~\Big|~\rev{(y,y')\in \mathcal{Q} \subset Y^{2}},~\normb{y-y'}>\eta\right\rbrace.
\end{equation*}
For each $\coveps>0$ we have
\begin{equation}
\covnum{\norma{\cdot}}{\rev{\secant_\eta}}{\coveps}\leq \covnumsq{\norma{\cdot}}{Y}{\frac{\coveps\eta}{4(1+B/A)}}.
\end{equation}}
\end{lemma}
\begin{proof}
Define the (semi)norm on $Y^2$:
\(
\norma{(y_1,y_2)-(y'_1,y'_2)}=\norma{y_1-y'_1}+\norma{y_2-y'_2}
\)
and note that we have trivially $\covnum{\norma{\cdot}}{Y^2}{\coveps}\leq \covnumsq{\norma{\cdot}}{Y}{\coveps/2}$.
Consider the set $\rev{\mathcal{Q}' := \set{ (y_{1},y_{2}) \in \mathcal{Q}: \normb{y_{1}-y_{2}} > \eta}}$.
By definition the function $f:(\rev{\mathcal{Q}'},\norma{\cdot}) \rightarrow (\secant,\rev{\norma{\cdot}})$ such that $f(y_{1},y_{2})=\frac{y_{1}-y_{2}}{\normb{y_{1}-y_{2}}}$ is surjective. We show that $f$ is Lipschitz continuous, and conclude with Lemma \ref{lem:covnumlipschitz}.
For $(y_1,y_2),(y'_1,y'_2) \in \rev{\mathcal{Q}'}$, we have
\begin{align*}
\norma{f(y_1,y_2)-f(y'_1,y'_2)}=& \norma{ \frac{y_1-y_2}{\normb{y_1-y_2}} - \frac{y'_1-y'_2}{\normb{y'_1-y'_2}} } \\
\leq& \norma{ \frac{y_1-y_2}{\normb{y_1-y_2}} - \frac{y'_1-y'_2}{\normb{y_1-y_2}}} + \norma{ \frac{y'_1-y'_2}{\normb{y_1-y_2}} - \frac{y'_1-y'_2}{\normb{y'_1-y'_2}} }.
\end{align*}
Since $\normb{y_1-y_2}> \eta$, the first term is bounded by
\begin{align*}
\frac{1}{\eta}\Big(\norma{y_1-y'_1}+\norma{y_2-y'_2} \Big) 
= & \frac{1}{\eta} \norma{(y_1,y_2)-(y'_1,y'_2)},
\end{align*}
while the second term is bounded by
\begin{align*}
\norma{y'_1-y'_2} \abs{\frac{1}{\normb{y_1-y_2}}-\frac{1}{\normb{y'_1-y'_2}}} 
\stackrel{\eqref{eq:SecSetLemmaNorms}}{\leq}& 
B  \abs{\frac{\normb{y'_1-y'_2}}{\normb{y_1-y_2}}-1}
\leq
 \frac{B}{\eta}\abs{ \normb{y'_1-y'_2}-\normb{y_1-y_2}} \\
\leq& 
\frac{B}{\eta}\Big(\normb{y_1-y'_1}+\normb{y_2-y'_2}\Big), \\
\stackrel{\eqref{eq:SecSetLemmaNorms}}{\leq}& \frac{B}{A\eta}\left(\norma{y_1-y'_1}+\norma{y_2-y'_2}\right), \\
=& \frac{B}{A\eta}\norma{(y_1,y_2)-(y'_1,y'_2)}.
\end{align*}
Hence we have
\[
\norma{f(y_1,y_2)-f(y'_1,y'_2)} \leq \frac{1+B/A}{\eta}\norma{(y_1,y_2)-(y'_1,y'_2)}.
\]
The function $f$ is Lipschitz continuous with constant $L=(1+B/A)/\eta$, and therefore for all $\coveps>0$:
\begin{equation*}
\covnum{\norma{\cdot}}{\secant}{\coveps} \stackrel{\text{Lemma \ref{lem:covnumlipschitz}}}{\leq} \covnum{\norma{\cdot}}{\rev{\mathcal{Q}'}}{\coveps/L} \stackrel{\text{Lemma \ref{lem:covnumsub}}}{\leq} \covnum{\norma{\cdot}}{Y^2}{\frac{\coveps}{2L}}\leq \covnumsq{\norma{\cdot}}{Y}{\frac{\coveps}{4L}}.
\end{equation*}
\end{proof}

\subsection{Mixture set}
Let $(X,\norm{\cdot})$ be a vector space over $\RR$ and $Y\subset X,~Y\neq \emptyset$. Let $k>0$ and $\weightSet\subset\RR^k$.
For $k>0$ and a bounded set $\weightSet\subset \RR^k,~\weightSet\neq \emptyset$, denote
the mixture set
\begin{equation}\label{eq:DefMixSet}
[Y]_{k,\weightSet}=\left\lbrace \sum_{i=1}^k \alpha_i y_i :\alpha \in \weightSet,~ y_i \in Y\right\rbrace.
\end{equation}
The \textbf{radius} of a subset $Y$ of a semi-normed vector space $(X,\norm{\cdot})$ is denoted $\norm{Y} :=\sup_{x\in Y}\norm{x}$. 
\begin{lemma}\label{lem:covnummix}
For all $\coveps>0$ the set $[Y]_{k,\weightSet}$ satisfies
\begin{equation}
\label{eq:covnummix}
\covnum{\norm{\cdot}}{[Y]_{k,\weightSet}}{\coveps} \leq \min_{\tau \in ]0,~1[} 
\covnum{\norm{\cdot}_1}{\weightSet}{\frac{(1-\tau)\coveps}{\norm{Y}}}
\cdot 
\covnumpow{k}{\norm{\cdot}}{Y}{\frac{\tau\coveps}{\norm{\weightSet}_{1}}}.
\end{equation}
If the semi-norm $\norm{\cdot}$ is indeed a norm and $Y$ and $\weightSet$ are compact, then $Y_{k,\weightSet}$ is also compact.
\end{lemma}

\begin{proof}
Let $\coveps>0$ and $\tau \in ]0;1[$. Denote $\coveps_1=\tau\coveps/\norm{\weightSet}_1$ and $\coveps_2=(1-\tau)\coveps/\norm{Y}$. Also denote $N_1=\covnum{\norm{\cdot}}{Y}{\coveps_1}$ and let $\enet_1=\{x_1,...,x_{N_1}\}$ be a $\coveps_1$-covering of $Y$. Similarly, denote $N_2=\covnum{\norm{\cdot}_1}{\weightSet}{\coveps_2}$, let $\enet_2=\{\alpha_1,...,\alpha_{N_2}\}$ be a $\coveps_2$-covering of $\weightSet$. The cardinality of the set
\begin{equation}
Z=\left\lbrace \sum_{j=1}^k \alpha_j x_j : x_{j} \in \enet_1,~ \alpha \in \enet_2\right\rbrace
\end{equation}
is $\abs{Z} \leq N_1^k N_2$. We will show that $Z$ is a $\coveps$-covering of $Y_{k,\weightSet}$. 

Consider $y=\sum_{j=1}^k \alpha_jy_j \in Y_{k,\weightSet}$. By definition, there is $\bar\alpha \in \enet_2$ so that $\norm{\alpha-\bar \alpha}_1\leq \coveps_2$, and for all $j=1...k$, there is $\bar{y}_j \in \enet_1$ so that $\norm{y_j-\bar y_j}\leq \coveps_1$. Denote $\bar{y}=\sum_{j=1}^k\bar{\alpha}_j\bar{y}_j \in Z$. We have
\begin{align}
\norm{y-\bar{y}}
=& 
\norm{ \sum_{j=1}^k \alpha_j y_j - \sum_{j=1}^k \bar\alpha_j \bar{y}_j }
\leq
\norm{ \sum_{j=1}^k \alpha_j y_j- \sum_{j=1}^k \alpha_j \bar{y}_j } + \norm{ \sum_{j=1}^k \alpha_j \bar{y}_j - \sum_{j=1}^k \bar\alpha_j \bar{y}_j } \notag \\
\leq& \sum_{j=1}^k \abs{\alpha_j} \norm{ y_j - \bar{y}_j } + \sum_{j=1}^k \abs{\alpha_j-\bar\alpha_j} \norm{ \bar{y}_j } \label{eq:lipgmm} \\
\leq& \norm{\alpha}_1\coveps_1  + \norm{\alpha-\bar\alpha}_1\norm{Y} \leq \norm{\weightSet}_1\coveps_1 + \coveps_2\norm{Y} = \coveps \notag ,
\end{align}
and $Z$ is indeed a $\coveps$-covering of $Y_{k,\weightSet}$. Therefore, we have the bound (for all $\tau$)
\begin{equation*}
\covnum{\norm{\cdot}}{Y_{k,\weightSet}}{\coveps} \leq \left\lvert Z\right\rvert \leq N_1^k N_2.
\end{equation*}
Furthermore, in equation \eqref{eq:lipgmm}, we have shown in particular that the embedding $(y_1,...,y_k,\alpha)\rightarrow \sum_{j=1}^k\alpha_jy_j$ from $Y^k\times \weightSet$ to $Y_{k,\weightSet}$ is continuous. Hence if $Y$ and $\weightSet$ are compact $Y_{k,\weightSet}$ is the continuous image of a compact set and is compact.
\end{proof}

\section{Proofs on mixtures of distributions}\label{sec:ProofMixtures}

We gather here all proofs related to results stated in Section~\ref{sec:general_mixtures}.

\subsection{Proof of %
Theorem~\ref{lem:ConcFnMixturesFromDipole}}\label{app:ProofPointwiseConcentration}

We start with the following lemma.
\begin{lemma}\label{lem:SecantMix}
Consider \rev{a kernel $\kernel$ with %
  and an integer $k \geq 1$ such that $\kernel$ has its $k$-coherence with respect to $\BasicSet$ bounded by $\zeta<1$.
  Then the normalized secant of the model set \rev{$\MixSetSep{k}$} of $2$-separated mixtures is made of mixtures of $2k$ normalized dipoles:}
\begin{equation}\label{eq:SecantMixCharacterization}
\secant_\kernel
\subset 
\set{ \sum_{l=1}^{2k} \alpha_l \mu_{l} : \ %
  (1+\zeta)^{-1} \leq  \sum_{l=1}^{2k} \alpha_{l}^{2} \leq %
  (1- \zeta)^{-1},\ {\alpha_{l}\geq 0},\ \mu_{l} \in \dipoleSet, \rev{l=1,\ldots,2k}},
\end{equation}
where the normalized dipoles $\rev{(\mu_{l})_{1\leq l \leq 2k}}$ \rev{associated to nonzero coefficients $\alpha_{l}$} are pairwise $1$-separated.
\end{lemma}
\begin{proof}
By definition any $\HH \in \secant_\kernel$
can be written as $\HH = (\mProb-\mProb')/\normkern{\mProb-\mProb'}$ with $\mProb,\mProb' \in \rev{\MixSetSep{k}}$ \rev{and $\normkern{\mProb-\mProb'}>0$}. By Lemma \ref{lem:decompDipoles} we have 
\(
\mProb-\mProb'=\sum_{l=1}^{\ell} \nu_l
\)
where the $\nu_l$ are non-zero dipoles that are $1$-separated from one another and $\ell \leq 2k$. With $\alpha_l:=\frac{\normkern{\nu_l}}{\normkern{\sum_{i=1}^{l} \nu_i}} \rev{>} 0$, $\HH_{l}:=\nu_{l}/\normkern{\nu_{l}}$ \rev{(note that by assumption $\kernel$ is locally characteristic with respect to $\BasicSet$ hence $\normkern{\nu_{l}}>0$)} we can write
\[
\HH=\frac{\sum_{l=1}^{\ell} \nu_l}{\normkern{\sum_{l=1}^{\ell} \nu_l}}=\sum_{l=1}^\ell \frac{\normkern{\nu_l}}{\normkern{\sum_{l=1}^{\ell} \nu_l}}\cdot \frac{\nu_l}{\normkern{\nu_l}} = \sum_{l=1}^\ell \alpha_l \cdot \HH_{l}
\]
By construction $\HH_{l} \in \dipoleSet$, and %
by definition of $2k$-coherence %
we have 
\[
 \sum_{i=1}^\ell \alpha_i^2=\frac{\sum_{i=1}^\ell\normkern{\nu_i}^2}{\normkern{\sum_{i=1}^\ell \nu_i}^2}
 \in
 \left[ (1+\zeta)^{-1}, (1-\zeta)^{-1} \right].
\]
If needed, we iteratively add to $\HH$ arbitrary normalized dipoles $\HH_{l}$ \rev{(not necessarily $1$-separated from another)} with $\alpha_l=0$ for $l=\ell+1 \ldots 2k$.
\end{proof}

To prove Theorem~\ref{lem:ConcFnMixturesFromDipole} we use the following version of Bernstein's inequality.
\begin{proposition}[{\citealt[Corollary 2.10]{Massart:2007book}}]
\label{pr:bernsteinmassart}
Let $X_i$, $i=1,\ldots,N$ be i.i.d. real-valued random variables. Denote $(\cdot)_{+} = \max(\cdot,0)$ and assume there exists positive numbers $\sigma^2$ and $u$ such that
\begin{align*}
\Exp(X^2) &\leq \sigma^2,\\
\Exp((X)^q_+) &\leq \frac{q!}{2} \sigma^2 u^{q-2},\quad \text{for all integers}\ q \geq 3.
\end{align*}
Then for any $t>0$ we have
\[
\PP\brac{ \frac{1}{N} \sum_{i=1}^N X_i \geq \Exp(X) + t }\leq \exp\paren{\frac{-Nt^2}{2(\sigma^2 + ut)}}.
\]
\end{proposition}
For both lemmas we start from the observation that
\[
\frac{\norm{\SketchingOperatorProb(\mProb-\mProb')}_{2}^{2}}{\normkern{\mProb-\mProb'}^{2}}-1
=
\tfrac{1}{m} \sum_{j=1}^{m} Z(\omega_{j})
\qquad \text{with} \qquad
Z(\omega)
:= 
\frac{
\abs{ \inner{\mProb,\rfeat}-\inner{\mProb',\rfeat}}^{2}
}{
\normkern{\mProb-\mProb'}^{2}
}-1
\]

\begin{proof}[Proof of Theorem~\ref{lem:ConcFnMixturesFromDipole}]
  We will use Proposition~\ref{pr:bernsteinmassart} with $X = Y(\freq) := \abs{\inner{\mProb-\mProb',\rfeat}}^{2}/\normkern{\mProb-\mProb'}^{2}$, hence we need to control the moments of $Y$.

  Denoting $\secant_\kernel$ the normalized secant set of $\MixSetSep{k}$, since $\HH := (\mProb-\mProb')/\normkern{\mProb-\mProb'} \in \secant_\kernel$ \rev{and $\kernel$ has
  its $2k$-coherence bounded by $\zeta\leq 3/4$,} %
  we can apply Lemma~\ref{lem:SecantMix}  to write $\HH$ as a mixture $\HH = \sum_{i=1}^{2k} \alpha_{i} \nu_{i}$ of normalized dipoles $\nu_{i} \in \dipoleSet$ with 
\(
\|\alpha\|_{2}^{2} \leq (1-\rev{\zeta})^{-1} \leq 4
\) 
\rev{and $\alpha_{l} \geq 0$.} 
We have $Y(\freq) = \frac{\abs{\inner{\mProb-\mProb',\rfeat}}^{2}}{\normkern{\mProb-\mProb'}^{2}} = \abs{\inner{\HH,\rfeat}}^{2} \geq 0$ and $\Exp_{\freq \sim \freqdist} Y(\freq) = 1$. With $\beta_{i} := \alpha_{i}/\norm{\alpha}_{1} \in [0,1]$ we have $\sum_{i=1}^{2k} \beta_{i}=1$. By convexity of $z \in \CC \mapsto \abs{z}^{2q}$ we get for $q \geq 2$:
\begin{align}
  Y^{q}(\freq) = \abs{\inner{\HH,\rfeat}}^{2q} 
  &=
    \abs{\sum_{i} \alpha_{i}\inner{\nu_{i},\rfeat}}^{2q}
    =
    \norm{\alpha}_{1}^{2q} \cdot \abs{\sum_{i} \beta_{i}\inner{\nu_{i},\rfeat}}^{2q}
    \leq
    \norm{\alpha}_{1}^{2q} \cdot \sum_{i} \beta_{i} \abs{\inner{\nu_{i},\rfeat}}^{2q};
    \notag\\
  \Exp_{\freq \sim \freqdist} \brac{Y(\freq)^{q}}
  & \leq 
    \norm{\alpha}_{1}^{2q} \cdot \sum_{i} \beta_{i}
    \Exp_{\freq \sim \freqdist} \brac{\abs{\inner{\nu_{i},\rfeat}}^{2q}}
    \leq 
    \frac{q!}{2} \cdot
    \rev{
       \norm{\alpha}^{2q}_{1}\lambda\gamma^{q-1}.
    }
    \label{eq:MomentControlMixtureNaive}
\end{align}
A direct use of Proposition~\ref{pr:bernsteinmassart} with $u := \norm{\alpha}_{1}^{2}\gamma $ and \rev{ $\sigma^{2} := \norm{\alpha}_{1}^{4}\lambda \gamma$} would lead to a concentration function \rev{$\ConcFn(t) = \order{t^{-2}\norm{\alpha}_{1}^{4}\lambda \gamma}$} for $t\leq 1$. Since $\norm{\alpha}_{1}^{2}\leq 2k \norm{\alpha}_{2}^{2} \leq 8k$ this would yield \rev{$\ConcFn(t)  = \order{t^{-2}k^{2}\lambda\gamma}$} for $t \leq 1$. This is however suboptimal: since for $q=1$ we have $\Exp_{\freq \sim \freqdist} \abs{Y(\freq)}^{q} = \normkern{\mProb-\mProb'}^{2}/\normkern{\mProb-\mProb'}^{2} = 1 = (\norm{\alpha}_{1}^{2}\gamma)^{q-1}$, interpolation allows to replace $\rev{(\norm{\alpha}_{1}^{2})^q}$ in~\eqref{eq:MomentControlMixtureNaive} by $\rev{(\norm{\alpha}_{1}^{2})^{q-1}}$ (up to log factors), as summarized by the following lemma \rev{whose} proof is slightly postponed.
\begin{lemma}\label{lem:MomentInterpolation}
  Assume that the random variable $X \geq 0$ satisfies $\Exp(X) = 1$ and \rev{$\Exp \brac{X^{q}} \leq \tfrac{q!}{2}a w^{q-1}$ for any integer $q \geq 2$, where $a\geq 2e$, $w > 0$.} %
  Then for any integer $q \geq 2$ we have
\begin{equation}
\Exp \brac{X^{q}} \leq \tfrac{q!}{2} \sigma^{2} u^{q-2},
\end{equation}
with {$u := w \log (e\rev{a} /2)$ and $\sigma^{2} := 2e \cdot w \log^{2}(e\rev{a}/2)$.}
\end{lemma}
\rev{As $\lambda \geq 1$ and $\norm{\alpha}_{1}^{2}\leq 8k$, putting  $a:= 8k\lambda \geq \max(\norm{\alpha}_{1}^{2}\lambda,2e)$, $w:=8k\gamma \geq \gamma \norm{\alpha}_1^2$, by~\eqref{eq:MomentControlMixtureNaive} the assumptions of Lemma~\ref{lem:MomentInterpolation} are satisfied hence}
\begin{equation}\label{eq:MomentControlMixture}
\Exp_{\freq \sim \freqdist} \brac{Y(\freq)^{q}}
\leq 
\frac{q!}{2} \sigma^{2} u^{q-2},
\end{equation}
with \rev{$u:=w \log(ea/2)$ and 
$\sigma^{2} :=2ew \log^{2}(ea/2)$. }
Since $a \geq 2e$, we have $\log(ea/2) \geq 2$ hence $u/\sigma^{2} = \paren{2e \log(ea/2)}^{-1} \leq \tfrac{1}{4e} \leq \tfrac{1}{3}$ (we chose the factor 1/3 for unification with the classical form of
Bernstein's inequality, see e.g. Lemma~\ref{P1-le:PointwiseConcentrationLemma} in \citepartone). 
Applying Proposition~\ref{pr:bernsteinmassart} to $X = Y(\freq)$ and to $X = -Y(\freq)$, we get 
\[
\mathbb{P}
\left(
\abs{\frac{\norm{\SketchingOperatorProb(\mProb-\mProb')}_{2}^{2}}{\normkern{\mProb-\mProb'}^{2}}-1} \geq t\right)
\leq
2\exp\left(-\frac{m t^{2}}{2\sigma^{2}(1+t\tfrac{u}{\sigma^{2}})}\right)
\leq
2\exp\left(-\frac{m t^{2}}{2\sigma^{2}(1+t/3)}\right),\qquad \text{for each } t > 0.
\]
Finally, we have \rev{$V := \sigma^{2} = 16e\gamma k \log^{2}(4ek\lambda)$,} which
yields~\eqref{eq:MainConcentrationPointwise} . 
\end{proof}

\begin{proof}[Proof of Lemma~\ref{lem:MomentInterpolation}]
Consider arbitrary integers $q \geq 2$, \rev{$p \geq 2$} and a real number $1 < p'$ such that $1/p'+1/p=1$. By H{\"o}lder's inequality, as the integer $r:=p(q-1)+1 = p(q-1+1/p) = p(q-1/p')$ satisfies $r \geq 2$, leveraging the assumptions yields
\begin{align*}
\Exp (X^{q})
= 
\Exp (X^{1/p'} X^{q-1/p'})
&\leq
\left(
\Exp (X^{1/p'})^{p'}
\right)^{1/p'}
\left(
\Exp (X^{q-1/p'})^{p}
\right)^{1/p}
=
\left(
\Exp X^{r}
\right)^{1/p}\\
&\leq
\left(
\tfrac{r!}{2} a w^{r-1}
\right)^{1/p}
= 
\rev{\left(
\tfrac{r!}{2}
\right)^{1/p}
\cdot a^{1/p} w^{q-1}.
}
\end{align*}
As $p \geq \rev{1}$ we have $r = pq-p+1 \leq pq$, hence
\[
  r! \leq (pq)! = \prod_{i=1}^{pq} i = \prod_{i=1}^{q} \prod_{j=1}^p (p(i-1)+j) \leq \prod_{i=1}^q(pi)^p = (p^q q!)^p .
\]
Combining the above we obtain
\[
\Exp (X^{q})
\leq q! \cdot  p^{q} \cdot (\rev{a}/2)^{1/p} \cdot w^{q-1} .
\]
\rev{If} $a>2e$, setting $p := \lceil \log (a/2) \rceil > 1$ yields $\log (a/2) \leq p < 1+\log (a/2) = \log (ea/2)$ and
 \[
p^{q} \cdot (a/2)^{1/p} = p^{q} \cdot e^{\tfrac{\log (a/2)}{p}}  \leq  (\log (e a/2))^{q} \cdot e.
\]
We conclude that for any $q \geq 2$
\[%
\Exp (X^{q}) 
\leq q! \cdot (\log(ea/2))^{q} \cdot e \cdot w^{q-1} 
= \tfrac{q!}{2}  \cdot [2e \cdot w\log^{2} (ea/2)] [w \log(ea/2)]^{q-2}.
\]
\rev{If $a=2e$ we establish the same bounds with arbitrary $a'>a$ and take their infimum.}
\end{proof}

\subsection{\rev{Proof of Theorem~\ref{thm:covnumSecant}}}
We prove that $\rev{\normfclass{\dipoleSet}{F}} \geq 1$ with a minor adaptation of the arguments showing that $\normfclass{\secant_{\kernel}}{F} \geq 1$ in  \citeppartone{Proof of Lemma~\ref{P1-le:PointwiseConcentrationLemma}}. 
For $\FClass$-integrable $\Prob,\Prob'$ and arbitrary $\alpha,\alpha' \in \RR$ the left hand side inequality in~\eqref{eq:TangentProof1} holds since
\begin{align*}
\normkern{\alpha\Prob-\alpha'\Prob'}^{2} 
&= \Exp_{\freq \sim \freqdist} 
\abs{\alpha\inner{\Prob,\rfeat}-\alpha'\inner{\Prob',\rfeat}}^{2}
\leq \sup_{\freq \sim \freqdist} 
\abs{\alpha\inner{\Prob,\rfeat}-\alpha'\inner{\Prob',\rfeat}}^{2} 
= \normfclass{\alpha\Prob-\alpha'\Prob'}{\FClass}^{2},
\end{align*}
and the r.h.s. inequality in~\eqref{eq:TangentProof1} holds by definition of $\normfclass{\dipoleSet}{F}$. Combined, they imply $\normfclass{\dipoleSet}{F} \geq 1$.

The rest of the proof relies on two lemmas. 
\begin{lemma}\label{lem:covnumWholeSecantNormd}
\rev{Consider a random feature family $(\{\rfeat\}_{\freq \in \Omega},\Lambda)$, the
  induced average kernel $\kernel$,
  and $d_{\rev{\FClass}}$ the pseudo-metric defined in~\eqref{eq:DefYetAnotherMetric}. Let $\Model$ be an arbitrary model set and $\secant_\kernel$ be its normalized secant set. For each $\coveps> 0$ we have}
\[
\covnum{d_{\rev{\FClass}}}{\secant_\kernel}{\coveps} \leq
\covnum{\rev{\normfclass{\cdot}{\FClass}}}{\secant_\kernel}{\rev{\frac{\coveps}{2\normfclass{\secant_\kernel}{F}}}}.
\]
\end{lemma}
\begin{remark}  \rev{Lemma~\ref{lem:covnumWholeSecantNormd} is valid beyond the case of mixture models.}\end{remark}
\begin{proof}
Consider $\HH_i=(\mProb_i-\mProb_i')/\normkern{\mProb_i-\mProb_i'}$, $i=1,2$ in $\secant_\kernel$.  
By definition of $\rev{\normfclass{\secant_\kernel}{F}}$, for all $\rfeat \in \rev{\FClass}$ we have 
\(
\abs{\inner{\HH_{i},\rfeat}}\leq \rev{\normfclass{\mProb_{i}-\mProb'_{i}}{\FClass}}/\normkern{\mProb_{i}-\mProb'_{i}}\leq \rev{\normfclass{\secant_\kernel}{F}}
\),
hence
\begin{align*}
d_{\rev{\FClass}}(\HH_1,\HH_2) 
= \sup_{\freq\in\Omega}\abs{\abs{\inner{\HH_1,\rfeat}}^{2} -\abs{\inner{\HH_2,\rfeat}}^{2}} 
&= \sup_{\freq  \in \Omega}
\Big|\abs{\inner{\HH_1,\rfeat}} + \abs{\inner{\HH_2,\rfeat}}\Big| 
\cdot
\Big|\abs{\inner{\HH_1,\rfeat}} -\abs{\inner{\HH_2,\rfeat}}\Big|\\
& \leq \sup_{\freq}  2  \rev{\normfclass{\secant_\kernel}{F}} \cdot  \abs{\inner{\HH_1-\HH_2,\rfeat}} 
= 2 \rev{\normfclass{\secant_\kernel}{F}} \rev{\normfclass{\HH_1-\HH_2}{\FClass}}.
\end{align*}
We conclude using Lemma \ref{lem:covnumlipschitz}.
\end{proof}
\begin{lemma}\label{lem:covnumSecantGeneric}%
  \rev{Consider \rev{a kernel $\kernel$ on $\SampleSpace$,
      $k \geq 1$ such that $\kernel$ has its $2k$-coherence bounded by $\zeta\leq 3/4$,} %
    $\secant_\kernel$ the normalized secant set of $\rev{\MixSetSep{k}}$, and a semi-norm $\norm{\cdot}$. Then for each $\delta>0$:}
\[
\covnum{\norm{\cdot}}{\secant_\kernel}{\coveps} 
\leq 
\rev{\left[
\covnum{\norm{\cdot}}{\dipoleSet}{\tfrac{\coveps}{8\sqrt{2k}}}
\cdot
\max\left(1,\tfrac{32\norm{\dipoleSet} \cdot \sqrt{2k}}{\coveps}\right)
\right]
^{2k}}.
\]
\end{lemma}
\begin{proof}
Denote $[Y]_{k,\weightSet}$ the set of $k$-mixtures of elements in $Y$ with weights in $\weightSet$ (see \eqref{eq:DefMixSet}). 
\rev{Any $\alpha \in \RR^{2k}$ such that $\norm{\alpha}_{2} \leq (1-\zeta)^{-1/2}\leq 2$} satisfies $\norm{\alpha}_{1} \leq \sqrt{2k} \norm{\alpha}_{2} \leq 2\sqrt{2k}$. As $\kernel$ has \rev{$2k$-}coherence
\rev{bounded by $\zeta\leq 3/4$}, it follows by Lemma~\ref{lem:SecantMix} that $\secant_\kernel \subset \brac{\dipoleSet}_{2k,\Ball}$ with $\dipoleSet$ the set of normalized dipoles and $\Ball:=\Ball_{\RR^{2k},\norm{\cdot}_1}(0,R)$ the closed $\ell^{1}$ ball of radius $R:=2\sqrt{2k}$ in $\RR^{2k}$. 

We use generic lemmas on covering numbers that can be found in Appendix~\ref{app:covering}.
Exploiting Lemma~\ref{lem:covnummix} will involve the following two quantities
\begin{equation}
\norm{\Ball}_1 = R = 2\sqrt{2k};
\qquad
D := \norm{\dipoleSet} 
= \sup_{\HH \in \dipoleSet} \norm{\HH}.\label{eq:radiigeneric}
\end{equation}
We get for each $\delta>0$,
\begin{eqnarray*}
\covnum{\norm{\cdot}}{\secant_\kernel}{\coveps}
& {\stackrel{{\text{[Lemma \ref{lem:covnumsub}]\qquad}}}{\leq}} &
\covnum{\norm{\cdot}}{\brac{\dipoleSet}_{2k,\Ball}}{\tfrac{\coveps}{2}}\notag\\
&{\stackrel{{\text{[Lemma \ref{lem:covnummix} with } \tau=\tfrac{1}{2}\&\eqref{eq:radiigeneric}]}}{\leq}} &
\covnum{\norm{\cdot}_1}{\Ball}{\tfrac{\coveps}{4D}} \cdot \covnumpow{2k}{\norm{\cdot}}{\dipoleSet}{\tfrac{\coveps}{4R}} \qquad  \notag \\
&{\stackrel{{\text{[Lemma~\ref{lem:covnumball}]\qquad}}}{\leq}} &
\left[\max\left(1,\tfrac{16RD}{\coveps}\right) \cdot \covnum{\norm{\cdot}}{\dipoleSet}{\tfrac{\coveps}{4R}}\right]^{2k}.\notag
\end{eqnarray*}
We conclude by \rev{replacing $R,D$ by their values from~\eqref{eq:radiigeneric}}.
\end{proof}

\rev{To wrap up the proof of Theorem~\ref{thm:covnumSecant}, we exploit  Lemmas~~\ref{lem:covnumWholeSecantNormd} and~\ref{lem:covnumSecantGeneric} with  $\delta' = \tfrac{\delta}{2\normfclass{\secant_\kernel}{F}}$ to get}
\[
\covnum{d_{\rev{\FClass}}}{\secant_\kernel}{\coveps} 
\leq 
\covnum{\rev{\normfclass{\cdot}{\FClass}}}{\secant_\kernel}{\rev{\frac{\coveps}{2\normfclass{\secant_\kernel}{F}}}}
\leq \rev{
\left[
\covnum{\normfclass{\cdot}{\rev{\FClass}}}{\dipoleSet}{\tfrac{\coveps}{16\normfclass{\secant_\kernel}{F}\sqrt{2k}}}
\cdot
\max\left(1,\tfrac{64 \normfclass{\dipoleSet}{F} \cdot \normfclass{\secant_\kernel}{F}\sqrt{2k}}{\coveps}\right)
\right]
^{2k}.}
\]
\rev{Combined with Theorem~\ref{thm:CompCstFromDipoles}, this yields the result.}

\subsection{Proof of Theorem~\ref{thm:LRIPGenericMixture}}

By Theorem~\ref{thm:CompCstFromDipoles}  the normalized secant $\secant_{\kernel}$ of $\MixSetSep{k}$ satisfies
\[
\normfclass{\secant_\kernel}{F} \leq \sqrt{8k} \normfclass{\dipoleSet}{F},\qquad
\dnormloss{\secant_\kernel}{} \leq \sqrt{8k} \dnormloss{\dipoleSet}{}.
\]
By \citeppartone{Lemma \ref{P1-le:PointwiseConcentrationLemma}} it follows that for each $t>0$
\[
\ConcFn(t) \leq 
\normfclass{\secant_\kernel}{F}^{2} \cdot \frac{2(1+t/3)}{t^{2}} \leq 
8k \cdot \normfclass{\dipoleSet}{F}^{2} \cdot \frac{2(1+t/3)}{t^{2}}.
\]
Combining with Lemma~\ref{lem:ConcFnMixturesFromDipole} and using that for $t=\coveps/2 \leq 1/2$ we have $2(1+t/3)/t^{2} \leq (7/3)/(\coveps/2)^{2} = (28/3)/\coveps^{2} \leq 10 \coveps^{-2}$ we obtain
\[
\ConcFn(\coveps/2) \leq 10 \cdot \delta^{-2} \cdot 8k \cdot \min\left(2e\gamma \log^{2}(4ek\rev{\lambda}),\normfclass{\dipoleSet}{F}^2\right).
\]
By \citeppartone{Lemma \ref{P1-le:PointwiseConcentrationLemma}} we have $\normfclass{\secant_\kernel}{F} \geq 1$ and by Theorem \ref{thm:covnumSecant} we have $\normfclass{\dipoleSet}{F} \geq 1$, hence
\(
1 \leq 64 k \normfclass{\dipoleSet}{F} \leq 256k \normfclass{\dipoleSet}{F}^{2}.
\)
For $0<\coveps<1$, since $\covnum{\normfclass{\cdot}{F}}{\dipoleSet}{\coveps}  \leq \rev{2}(C/\coveps)^{r}$ we obtain 
\begin{align*}
\max\left(1,\tfrac{256k\normfclass{\dipoleSet}{F}^{2}}{\coveps/2}\right) 
&=  \tfrac{512 k \normfclass{\dipoleSet}{F}^{2}}{\coveps}\\
\covnum{\normfclass{\cdot}{\rev{\FClass}}}{\dipoleSet}{\tfrac{\coveps/2}{64k \normfclass{\dipoleSet}{F}}}
&
\leq \rev{2}(128kC\normfclass{\dipoleSet}{F}/\coveps)^{r}.
\end{align*}
Taking the logarithms and using Theorem \ref{thm:covnumSecant} we obtain

\begin{align*}
\log \covnum{d_{\rev{\FClass}}}{\secant_\kernel}{\coveps/2} 
&\leq 
2k \cdot
\left[
\rev{\log(2) + } r \cdot \log (k C \normfclass{\dipoleSet}{F}) + r \cdot \log (128/\coveps) + 
\log (k \normfclass{\dipoleSet}{F}^{2}) + \log (512/\coveps)
\right]\notag\\
&\leq 2k (r+1) \left[ \log k + \log C + \log \normfclass{\dipoleSet}{F}^{2} + \log(\rev{1024}/\coveps)\right]
\end{align*}
We establish~\eqref{eq:mainLRIPPureKernel} and the LRIP~\eqref{eq:lowerDRIP} with $C_\SketchingOperatorProb :=\frac{8\sqrt{2k}\dnormloss{\dipoleSet}{}}{\sqrt{1-\delta}}$ and $\eta = 0$ using Theorem \ref{thm:mainLRIP}.

\subsection{Proof of Theorem~\ref{thm:RadiusGeneric}}
\label{sec:proofAdmissibilityDipole}
\rev{The fact that $\normfclass{\monopoleSet}{G} \leq \normfclass{\dipoleSet}{G}$ is a simple consequence of Definition~\eqref{eq:DefSetRadius} and the inclusion $\monopoleSet \subset \dipoleSet$. To prove the second inequality in~\eqref{eq:admissibilityDipole},
let $\nu = \alpha_1\Prob_{\Param_1} - \alpha_2 \Prob_{\Param_2}$ be a nonzero dipole}, \rev{which by definition means that $\Param_{1}\neq \Param_{2}$,} $\metricParam(\Param_{1},\Param_{2}) \leq 1$. Consider a fixed $f \in \GClass$.
We are interested in bounding $\abs{\inner{\nu,f}}/\normkern{\nu}$, which is invariant by rescaling $\nu$;
hence, replacing $\nu$ by $\tilde{\nu} := C^{-1}_\nu \nu$, with $C_\nu:=\max(\alpha_1 \normkern{\Prob_{\Param_1}}^{-1},
\alpha_2 \normkern{\Prob_{\Param_2}}^{-1})>0$, we can assume without loss of generality that $\nu$ takes the form
$\nu=s(\Prob_{\Param}/\normkern{\Prob_{\Param}} - \alpha \Prob_{\Param'}/\normkern{\Prob_{\Param'}})$,
with $s\in \set{-1,1}; \alpha\in[0,1]; \rho:=\metricParam(\Param,\Param')\leq 1$. With this representation, since $\rho \leq 1$ we get by~\eqref{eq:DefCStrongly}
\[
  \normkern{\nu}^2 = \nkernel(\Param,\Param) + \alpha^2 \nkernel(\Param',\Param') -2 \alpha
  \nkernel(\Param,\Param') \geq (1-\alpha)^2 + \alpha c \rho^2.
\]
Denoting the normalized monopole $\nu_\Param:= \Prob_{\Param}/\normkern{\Prob_{\Param}}$ (and similarly $\Param'$),
we have $\abs{\inner{\nu_\Param,f}} \leq \normfclass{\monopoleSet}{\GClass}$, and
$\abs{\inner{\nu_\Param - \nu_{\Param'},f}} \leq  L_{\GClass} \rho$, from the assumptions of the theorem. Thus,
\begin{align*}
  \abs{\inner{\nu,f}} = \abs{ \inner{\nu_\Param  - \alpha \nu_{\Param'},f}}
   = \abs{  (1-\alpha)\inner{\nu_\Param,f}  + \alpha \inner{\nu_\Param-\nu_{\Param'},f}}    
  & \leq  (1-\alpha)\abs{\inner{\nu_\Param,f}} +\alpha \abs{\inner{\nu_\Param- \nu_{\Param'},f}}\\
  & \leq (1-\alpha) \normfclass{\monopoleSet}{\GClass} +\alpha L_\GClass \rho.
\end{align*}
Gathering the two last displays, and using $c \rho^2 \leq 2$, we get 
\begin{align*}
  \frac{\abs{\inner{\nu,f}}}{\normkern{\nu}} = \frac{(1-\alpha) \normfclass{\monopoleSet}{\GClass} +\alpha L_\GClass \rho}{\sqrt{(1-\alpha)^2 + \alpha c \rho^2}}
 \leq \normfclass{\monopoleSet}{\GClass} 
 + \max_{\alpha \in[0,1]} \frac{\alpha L_\GClass \rho}{\sqrt{(1-\alpha)^2 + \alpha c \rho^2}}
 \leq  \normfclass{\monopoleSet}{\GClass} +
 L_\GClass/\sqrt{c},
\end{align*}
where the last inequality follows from 
an elementary study of the function $\alpha \mapsto \alpha/\sqrt{(1-\alpha)^2+a\alpha)}$
showing that it is nondecreasing on $[0,1]$ for $a \in [0,2]$, and therefore attains
its maximum at $\alpha=1$.

\subsection{Proof of Theorem~\ref{thm:MutualCoherenceRBF1}}\label{sec:proofmutualcoherence}
We start with the following intermediate result:

\begin{proposition}\label{prop:MutualCoherenceProp}
\rev{  Consider $\BasicSet = (\ParamSpace,\metricParam,\psi)$ a family of base distributions, $\kernel$ a psd kernel on $\SampleSpace$ such that $\normkern{\Prob_\Param}>0$ for all $\Param \in \ParamSpace$, and $\nkernel$ its associated normalized kernel on $\ParamSpace$, given
  by~\eqref{eq:defnkernel}. Assume that $\kernel$ is $c$-strongly characteristic, $c \in(0,2]$.}
Using the shorthand $d_{ij} := \metricParam(\Param_{i},\Param_{j})$ and $K_{ij} := \nkernel(\Param_{i},\Param_{j})$ for generic parameters $\Param_{i}, \Param_{j} \in \ParamSpace$, assume there is $C$ such that the following properties hold:
\begin{enumerate}
\item \label{it:kernelBound} if $d_{ij} \geq 1$ then $\abs{K_{ij}} \leq C$;
\item \label{it:kernelLip} if  $\min(d_{ij},d_{ik}) \geq 1$ then $\abs{K_{ij}-K_{ik}} \leq C d_{jk}$;
\item \label{it:kernelLip2} if $\max(d_{ij},d_{kl}) \leq 1$ and $\min(d_{ik},d_{il},d_{jk},d_{jl}) \geq 1$ then
$\abs{K_{ik}-K_{jk}-K_{il}+K_{jl}}  \leq C d_{ij}d_{kl}$.
\end{enumerate}
Then the kernel $\kernel$ has mutual coherence with respect to $\BasicSet$ bounded by
\begin{equation}
\label{eq:MutualCoherenceGeneric}
\MutualCoherence \leq \tfrac{\rev{4}C}{\min(c,1)}.
\end{equation}
\end{proposition}
\begin{proof}
Denote $\nu=\alpha_1\rev{\frac{\Prob_{\Param_1}}{\normkern{\Prob_{\Param_1}}}}-\alpha_2\rev{\frac{\Prob_{\Param_2}}{\normkern{\Prob_{\Param_2}}}}$ and $\nu'=\alpha_3\rev{\frac{\Prob_{\Param_3}}{\normkern{\Prob_{\Param_3}}}} - \alpha_4\rev{\frac{\Prob_{\Param_4}}{\normkern{\Prob_{\Param_4}}}}$ two dipoles that are $1$-separated, and without loss of generality suppose that $\alpha_1=\alpha_3=1$, $\alpha_2=a \in [0,1]$, $\alpha_4=b \in [0,1]$. Our goal is to bound $\frac{\abs{\kernel(\nu,\nu')}}{\normkern{\nu}\normkern{\nu'}}$. Recall that $d_{ij}=\metricParam(\Param_i,\Param_j)$ and $K_{ij}=\rev{\nkernel(\Param_i,\Param_j)}$. We have
\begin{align*}
\frac{\abs{\kernel(\nu,\nu')}}{\normkern{\nu}\normkern{\nu'}} =& \frac{\abs{K_{13} - aK_{23} -bK_{14} + ab K_{24}}}{\sqrt{1-2aK_{12}+a^2}\sqrt{1-2bK_{34}+b^2}} \\
\leq & \tfrac{\abs{K_{13}-K_{23}-K_{14}+K_{24}} + \abs{(1-a)(K_{23}-K_{24})} + \abs{(1-b)(K_{14}-K_{24})} + \abs{(a-1)(b-1)K_{24}}}{\sqrt{(1-a)^2+2a(1-K_{12})}\sqrt{(1-b)^2+2b(1-K_{34})}}
\end{align*}
By the assumptions on $\kernel$ we have
\begin{align*}
\abs{K_{13}-K_{23}-K_{14}+K_{24}} \leq& Cd_{12}d_{34} \quad \text{(since $d_{12}\leq 1, d_{34}\leq 1$, $\min(d_{13},d_{14},d_{23},d_{24}) \geq 1$)}\\
\abs{K_{23}-K_{24}} \leq&~ \rev{C}d_{34} \quad \text{(since $d_{23}\geq 1$ and $d_{24}\geq 1$)} \\
\abs{K_{14}-K_{24}} \leq&~ \rev{C}d_{12} \quad \text{(since $d_{14}\geq 1$ and $d_{24}\geq 1$)} \\
\abs{K_{24}} \leq&~ \rev{C} \qquad \text{(since $d_{24}\geq 1$)} \\
2(1-K_{12})\geq&~ c d_{12}^2 \quad \text{(since $d_{12}\leq 1$)}\\
2(1-K_{34})\geq&~ c d_{34}^2 \quad \text{(since $d_{34}\leq 1$)}\\
\end{align*}
Therefore, denoting $g(x,y):=\frac{x+y}{\sqrt{x^2+(1-x)y^2}}$ for $0 \leq x,y\leq 1$, we have
\begin{align*}
\frac{\abs{\kernel(\nu,\nu')}}{\normkern{\nu}\normkern{\nu'}} \leq&~ \rev{C}\cdot \frac{d_{12}d_{34} + (1-a)d_{34} + (1-b)d_{12} + (1-a)(1-b)}{\sqrt{(1-a)^2+ac d_{12}^2}\sqrt{(1-b)^2+bc d_{34}^2}} \\
=&~ \rev{C}\cdot \frac{d_{12} + 1-a}{\sqrt{(1-a)^2+ac d_{12}^2}} \cdot \frac{d_{34} + 1-b}{\sqrt{(1-b)^2+bc d_{34}^2}} \\
\leq&~ 
\rev{C}\cdot \frac{d_{12} + 1-a}{\sqrt{\min(c,1)}\sqrt{(1-a)^2+a d_{12}^2}} \cdot \frac{d_{34} + 1-b}{\sqrt{\min(c,1)}\sqrt{(1-b)^2+b d_{34}^2}}\\
=&~\frac{\rev{C}}{\min(c,1)} \cdot g(1-a,d_{12}) g(1-b,d_{34}).
\end{align*}
As we have for any $0 \leq x,y\leq 1$: $g(x,y)\leq 2$ (see~Lemma~\ref{lem:gxybound} for a proof),
gathering everything, we obtain
\begin{equation*}
\frac{\abs{\kernel(\nu,\nu')}}{\normkern{\nu}\normkern{\nu'}} \leq \frac{\rev{4C}}{\min(c,1)}. \qedhere
\end{equation*}
\end{proof}

We will establish that the assumptions of Theorem~\ref{thm:MutualCoherenceRBF1} allow us to use
the above proposition, for this we will also need the following technical lemmas.

\begin{lemma}\label{lem:SmallLipschitzLemma}
  Assume that $h: \mathbb{R}_{+} \to \mathbb{R}$ is differentiable and that $h'(t)$ is $C$-Lipschitz. Then
  \[
  \abs{h(0) - h(x) - h(y) + h(x + y)} \leq
  \rev{x y C},\quad \forall x,y \geq 0.
  \]
\end{lemma}
\begin{proof}
\rev{  Assume without loss of generality that $x = \min(x,y)$ and introduce $g(y) = h(y+x)-h(y)$. 
  Notice that the considered quantity is $|g(y)-g(0)|$. 
  \rev{By the mean value theorem,}
  $g(y) - g(0) = g'(c) y = (h'(c+x)-h'(c)) y$ for some $c \in [0, x]$, thus
    \[
    \abs{h(0) - h(x) - h(y) + h(x + y)}\\
    = \abs{y(h'(c+x) - h'(c))} \leq yCx\,.
    \]
    }
\end{proof}

\begin{lemma}\label{lem:IntermediaryResultOrthogonality}
\rev{Assume that $K$ is differentiable with Lipschitz derivative on $[1,\infty)$ and denote $K'_{\max},K''_{\max}$ as in the statement of Theorem~\ref{thm:MutualCoherenceRBF1}. }
Let $(\xi_i)_{1\leq i \leq 4}$ be 4 points in a Hilbert space $\mathcal{H}$; denote
$d_{ij}=\norm{\xi_i-\xi_j}_{\mathcal{H}}$ and assume $d_{ij}\geq 1$ for $(i,j) \in \set{(1,3);(1,4);(2,3);(2,4)}$. Then we have
\begin{equation}
\label{eq:IntermediaryResultOrthogonality}
\rev{\abs{K(d_{13}) - K(d_{23}) - K(d_{14}) + K(d_{24})}} \leq \rev{(2K'_{\max}+K''_{\max})}d_{12}d_{34}.
\end{equation}
\end{lemma}

\begin{proof}
Assume without loss of generality that $d_{13} = \min(d_{13},d_{23},d_{14},d_{24})$ and write
\begin{align}
\abs{K(d_{13}) - K(d_{23}) - K(d_{14}) + K(d_{24})} \leq&~ \abs{  K(d_{13}) - K(d_{23}) - K(d_{14}) + K(d_{23} + d_{14} - d_{13}) } \notag \\ &~+ \abs{K(d_{24}) - K(d_{23} + d_{14} - d_{13} )}. \label{eq:IntermediaryResultOrthogonalityProof1}
\end{align}
To bound the first term of the right hand side of~\eqref{eq:IntermediaryResultOrthogonalityProof1}, since we assumed without loss of generality that $d_{13} = \min(d_{13},d_{23},d_{14},d_{24})$, \rev{and since $d_{13} \geq 1$ by the $1$-separation assumption, } we can
 apply Lemma \ref{lem:SmallLipschitzLemma} with $h(t) := K(d_{13}+t)$, $x:= d_{23}-d_{13}\geq 0$, $y:= d_{14}-d_{13}\geq 0$, $C = K''_{\max}$, leading to
\[\abs{  K(d_{13}) - K(d_{23}) - K(d_{14}) + K(d_{23} + d_{14}
  - d_{13}) } \leq \rev{K''_{\max}} \abs{(d_{23} -d_{13})(d_{14} - d_{13})}\\
\leq 2 K''_{\max} d_{12} d_{34}\,.
\]
To bound the second term in \eqref{eq:IntermediaryResultOrthogonalityProof1}, let $g(u):=K(\sqrt{u})$ and note that $g'(u) = K'(\sqrt{u})/2\sqrt{u}$ hence $g'(u^2) \leq K'_{\max}/2 $ for $u \geq 1$.
By the separation assumption we have $1 \leq d_{23} \leq d_{23} + d_{14} - d_{13}$ and $1 \leq d_{24}$. We write
\begin{align*}
K(d_{24}) - K(d_{23} + d_{14} - d_{13} )
=g(d^2_{24}) - g( (d_{23} + d_{14} - d_{13} )^2)
\leq \tfrac{K'_{\max}}{2} \abs{d^2_{24} - (d_{23} + d_{14} - d_{13} )^2},
\end{align*}
where the last inequality follows from \rev{the mean value theorem.} %
Now, it holds
\[
d^2_{24} - (d_{23} + d_{14} - d_{13} )^2
= d^2_{24} - d^2_{23} - d^2_{14} + d^2_{13} - 2(d_{13}-d_{23})(d_{13}-d_{14})\,,
\]
and by the reversed triangle inequality $\abs{d_{ij}-d_{il}} \leq d_{jl}$ for any $i,j,l$ so that the last product is bounded in absolute value by $2d_{12}d_{34}$. It is also easy to check by expanding the squared norms $d_{ij}^2 = \rev{\norm{\xi_i - \xi_j}_{\mathcal{H}}^2}$
that
\[
\abs{d^2_{24} - d^2_{23} -  d^2_{14} + d^2_{13}} = 2\abs{\rev{\inner{\xi_1 - \xi_2,\xi_3-\xi_4}_{\mathcal{H}}}} \leq 2 d_{12}d_{34}\,.
\]
Gathering everything we get the desired result.
\end{proof}

We can now prove Theorem~\ref{thm:MutualCoherenceRBF1}.
\begin{proof}[Proof of Theorem~\ref{thm:MutualCoherenceRBF1}]
Since $K(0)=1$ and $K(u) \leq 1-cu^{2}/2$,  the kernel $\kernel$ is $c$-strongly locally characteristic with respect to $\BasicSet$.

We exhibit a constant $C$ allowing the use of Proposition~\ref{prop:MutualCoherenceProp}. 
Consider generic parameters $\Param_{i},\Param_{j}$,
and denote as before $d_{ij} = \metricParam(\Param_{i},\Param_{j})$, $K_{ij} = \rev{\nkernel({\Param_i},{\Param_j})} = K(d_{ij})$. 
Since $\abs{K(u)} \leq K_{\max}$ for $u \geq 1$ we get $\abs{K_{ij}} \leq K_{\max}$ if $d_{ij} \geq 1$. By the mean value theorem and the reversed triangle inequality, if $\min(d_{ij},d_{il}) \geq 1$ then as $\abs{K'(u)} \leq K'_{\max}$ for $u \geq 1$ we get $\abs{K_{ij}-K_{il}} = \abs{K(d_{ij})-K(d_{il})} \leq K'_{\max} \abs{d_{ij}-d_{il}} \leq K'_{\max} d_{jl}$. Applying Lemma \ref{lem:IntermediaryResultOrthogonality} we get if $d_{12} \leq 1$, $d_{34}\leq 1$ and $\min(d_{13},d_{14},d_{23},d_{24}) \geq 1$ that:
\begin{align*}
\abs{K_{13}-K_{23}-K_{14}+K_{24}} \leq& \rev{(2K'_{\max}+K''_{\max})}d_{12}d_{34}
\end{align*}
To conclude observe that $C := \max(K_{\max},K'_{\max},\rev{(2K'_{\max}+K''_{\max})}) = \max(K_{\max},\rev{(2K'_{\max}+K''_{\max})})$. 
\end{proof}

\section{Proofs for Section~\ref{sec:RBFKernel}} 

We will start with an elementary result introducing a ``canonical'' representation of normalized dipoles.
\begin{lemma}
  \label{lem:canrepdip}
  The set of normalized dipoles can be written as
  \[
    \dipoleSet = \set{\frac{\nu}{\normkern{\nu}}: \nu = \normkern{\Prob_0}^{-1} s (\Prob_{\Param'}-\alpha\Prob_{\Param}); s \in \set{-1,+1}; 0 \leq \alpha \leq 1; 0<\norm{\Param'-\Param} \leq 1}.
  \]
\end{lemma}
\begin{proof}
  Any element in $\dipoleSet$ can (by definition~\eqref{eq:NormalizedDipoleSet}) be written as $\nu/\normkern{\nu}$, where $\nu$ is a nonzero dipole of the form $\nu = \alpha_1 \Prob_{\Param_1} - \alpha_2
  \Prob_{\Param_2}$, with $\alpha_1,\alpha_2 \geq 0$ and $\norm{\Param_1-\Param_2}\leq 1$. Let $\zeta=\max(\alpha_1,\alpha_2)>0$ since $\nu$ is nonzero. Then $\nu'=(\zeta\normkern{\Prob_0})^{-1}\nu$ is such that
  $\nu'/\normkern{\nu'} = \nu/\normkern{\nu}$, and $\nu'= \normkern{\Prob_0}^{-1} s (\Prob_{\Param'}-\alpha\Prob_{\Param})$
  with $s \in \set{-1,+1}$; $0 \leq \alpha \leq 1$; $0<\norm{\Param'-\Param} \leq 1$.
\end{proof}

\subsection{Proof of Lemma~\ref{lem:RFFMoments}} \label{app:MomentsDipoles}

\rev{From Lemma~\ref{lem:canrepdip}, 
any normalized dipole can be written as $\HH = \nu/\normkern{\nu}$ with $\nu = s \normkern{\Prob_0}^{-1} (\Prob_{\Param}- \alpha \Prob_{\Param'})$, $s\in\set{-1,1}$,} $0 \leq \alpha \leq 1$ and $x := \Param-\Param' \neq 0$, $0<\norm{x} \leq 1$. Denote $u := x/\norm{x}$. Since $\normkern{\Prob_{\Param}} = \normkern{\Prob_{\Param'}} = \normkern{\Prob_{0}}$, reusing~\eqref{eq:InnerProdRFFPTheta} and the definition of $\nkernel$ we have
\begin{align*}
  \normkern{\nu}^{2} 
&=
\normkern{\Prob_{0}}^{-2}\left(\normkern{\Prob_{\Param}}^{2}+\alpha^{2}\normkern{\Prob_{\Param'}}^{2}
-2\alpha \kernel(\Prob_{\Param},\Prob_{\Param'})\right)
= 1+\alpha^{2}-2\alpha \nkernel(x)
= (1-\alpha)^{2} + 2\alpha (1-\nkernel(x));\\
\normkern{\Prob_0}^2 \abs{\inner{\nu,\rfeat}}^{2} 
&=
\abs{\inner{\Prob_0,\rfeat}}^{2} \cdot 
\abs{
e^{\jmath\inner{\freq,\Param}}
-
\alpha e^{\jmath\inner{\freq,\Param'}}
}^{2}
=
\abs{\inner{\Prob_0,\rfeat}}^{2} \cdot 
\left(
(1-\alpha)^{2}+2\alpha (1-\cos \inner{\freq,x})
\right)\\
& \leq 
\abs{\inner{\Prob_0,\rfeat}}^{2} \cdot 
\max\left(1,\tfrac{1-\cos \inner{\freq,x}}{1-\nkernel(x)}\right) \cdot 
\left((1-\alpha)^{2}+2\alpha (1-\nkernel(x))\right).
\end{align*}
\rev{Together, the last two inequalities imply
  \begin{equation}
    \label{eq:bounddip}
    \normkern{\Prob_{0}}^{2}\abs{\inner{\HH,\rfeat}}^{2}\leq  \abs{\inner{\Prob_0,\rfeat}}^{2} \cdot 
\max\left(1,\tfrac{1-\cos \inner{\freq,x}}{1-\nkernel(x)}\right). 
    \end{equation}}
By~\eqref{eq:MainAssumptionDipoleShift} we have $1-\nkernel(x) \geq b^{-1} \min(1,(\norm{x}/a)^{2})$ hence:
\begin{itemize}
\item if $\norm{x} \geq a$ then $0 \leq 1-\cos \inner{\freq,x} \leq 2 = 2\min(1,(\norm{x}/a)^{2})) \leq 2b(1-\nkernel(x))$ hence, since we assumed $b \geq 1/2$,
\[
\max(1,\tfrac{1-\cos \inner{\freq,x}}{1-\nkernel(x)}) \leq \max(1,2b) = 2b;
\]
\item if $\norm{x} \leq a$ then, since $\sin^{2} t \leq t^{2}$ for each $t \in \RR$ we have $2\sin^{2}(t/2) \leq \frac{t^{2}}{2}$ and
\[
0 \leq 1-\cos \inner{\freq,x} = 2\sin^{2} \tfrac{\inner{\freq,x}}{2} \leq \tfrac{1}{2} \inner{\freq,x}^{2} = \tfrac{a^{2}}{2} \inner{\freq,u}^{2} \cdot (\norm{x}/a)^{2}
\leq \tfrac{a^{2}}{2} \inner{\freq,u}^{2} \cdot b(1-\nkernel(x)),
\]
\[
\text{implying} \qquad \max(1,\tfrac{1-\cos \inner{\freq,x}}{1-\nkernel(x)}) \leq \max(1,\tfrac{ba^{2}}{2}\inner{\freq,u}^{2}).
\]
\end{itemize}
Since $w(\freq)\geq 1$ we have $\abs{\inner{\Prob_0,\rfeat}} \leq \abs{\Exp_{\Sample \sim \Prob_{0}} e^{\jmath \inner{\freq,\Sample}}} \leq 1$, hence for any integer $q \geq 1$ we have
\begin{equation}\label{eq:TmpMomentControl}
\Exp_{\freq \sim \freqdist} \abs{\inner{\Prob_0,\rfeat}}^{2q} \leq 
\Exp_{\freq \sim \freqdist} \abs{\inner{\Prob_0,\rfeat}}^{2} = \normkern{\Prob_{0}}^{2}.
\end{equation}
Denoting $Y(\freq) := \normkern{\Prob_{0}}^{2}\abs{\inner{\HH,\rfeat}}^{2}$, we obtain
by~\eqref{eq:bounddip} and the previous cases, for any integer $q \geq 2$: 
\begin{itemize}
\item if $\norm{x} \geq a$ then $\abs{Y(\freq)}^{q} \leq \abs{\inner{\Prob_0,\rfeat}}^{2q}\cdot (2b)^{2q}$; by~\eqref{eq:TmpMomentControl} we get $\Exp_{\freq \sim \freqdist} \brac{\abs{Y(\freq)}^{q}} \leq \normkern{\Prob_{0}}^{2} \cdot (2b)^{q}$;
\item if $\norm{x} \leq a$ then
$\abs{Y(\freq)}^{q} \leq 
\abs{\inner{\Prob_0,\rfeat}}^{2q}
\cdot 
\max\left(1,(ba^{2}/2)^{q} \inner{\freq,u}^{2q}\right)
\leq 
\abs{\inner{\Prob_0,\rfeat}}^{2q}
\cdot \left(1+(ba^{2}/2)^{q} \inner{\freq,u}^{2q}\right)$; using assumption~\eqref{eq:MomentAssumptionFreqDist} and~\eqref{eq:TmpMomentControl}%
we get
\begin{align*}
\Exp_{\freq \sim \freqdist} \brac{\abs{Y(\freq)}^{q}} 
&\leq 
\normkern{\Prob_{0}}^{2}
+
(ba^{2}/2)^{q} \cdot
\Exp_{\freq \sim \freqdist} \brac{\abs{\inner{\Prob_0,\rfeat}}^{2q} \cdot \inner{\freq,u}^{2q}} 
\leq 
\normkern{\Prob_{0}}^{2}+ (ba^{2}/2)^{q} \normkern{\Prob_{0}}^{2} \frac{q!}{2} \lambda_{0}^{q}\\
& \stackrel{q \geq 2}{\leq}
\normkern{\Prob_{0}}^{2} \frac{q!}{2} \left(1^{q}+ (ba^{2}\lambda_{0}/2)^{q}\right)  
\leq 
\normkern{\Prob_{0}}^{2} \frac{q!}{2} (1+ ba^{2}\lambda_{0}/2)^{q}. 
\end{align*}
\end{itemize}
Combining both cases we obtain $
\Exp_{\freq \sim \freqdist} \brac{\abs{Y(\freq)}^{q} }
\leq \normkern{\Prob_{0}}^{2} \frac{q!}{2} \max(2b, 1+ ba^{2}\lambda_{0}/2)^{q}$.
Finally, %
we get
\[
  \Exp_{\freq \sim \freqdist} \brac{\abs{\inner{\HH,\rfeat}}^{2q} }
  = (\normkern{\Prob_{0}}^{-2})^{q} \Exp_{\freq \sim \freqdist} \brac{\abs{Y(\freq)}^{q}}
\leq (\normkern{\Prob_{0}}^{-2})^{q-1} \frac{q!}{2} \max(2b, 1+ ba^{2}\lambda_{0}/2)^{q}.
\]

\subsection{Proof of Lemma~\ref{lem:TangentLocationBased}}\label{app:TangentLocationBased}

The following bound will be useful.%
\begin{lemma}\label{lem:gxybound}
For each $0 \leq x,y\leq 1$, $(x,y) \neq (0,0)$ we have $g(x,y):=\frac{x+y}{\sqrt{x^2+(1-x)y^2}} \leq 2$.
\end{lemma}
\begin{proof} Since $2xy \leq x^{2}+y^{2}$,  we have
  \begin{align*}    
    g^{2}(x,y)=\frac{(x+y)^2}{x^2+(1-x)y^2}
    & = 1 + \frac{2xy+xy^2}{x^2 + y^2 -xy^2}
 \leq 1 +  \frac{2xy+xy^2}{2xy -xy^2} = 1 + \frac{2+y}{2-y} = \frac{4}{2-y} \leq 4.
\end{align*}
\end{proof}

\begin{proof}[Proof of Lemma~\ref{lem:TangentLocationBased}]
  The argument relies on the decomposition (straightforward from the ``canonical''
  dipole representation introduced in Lemma~\ref{lem:canrepdip})
  \(
    \dipoleSet = \dipoleSet_\eta \cup \overline{\dipoleSet}_\eta, 
  \)
  where (for $\eta>0$ to be soon specified)
  \begin{align*}
    \rev{\dipoleSet_\eta} & := \set{\frac{\nu}{\normkern{\nu}}: \nu = \normkern{\Prob_{0}}^{-1} s (\Prob_{\Param'}-\alpha\Prob_{\Param}), \Param,\Param' \in \ParamSpace, \norm{\Param-\Param'} \leq 1, 0 \leq \alpha \leq 1, \normkern{\nu} > \eta},\\
    \rev{\overline{\dipoleSet}_\eta} & := \set{\frac{\nu}{\normkern{\nu}}: \nu = \normkern{\Prob_{0}}^{-1} s ( \Prob_{\Param'}-\alpha\Prob_{\Param}), \Param,\Param' \in \ParamSpace, \norm{\Param-\Param'} \leq 1, 0 \leq \alpha \leq 1, \normkern{\nu} \leq \eta},
  \end{align*}
  so that
  \begin{equation}\label{eq:covdipole2}
    \covnum{\norm{\cdot}}{\dipoleSet}{\delta} \leq
    \covnum{\norm{\cdot}}{\dipoleSet_\eta}{\delta} + \covnum{\norm{\cdot}}{\overline{\dipoleSet}_\eta}{\delta}.
  \end{equation}
  By Theorem~\ref{thm:covnumSecant} we have $D := \normfclass{\dipoleSet}{F} \geq 1$, and 
  we will establish below that for $\eta := \frac{\coveps }{8 C''_{\FClass}}>0$:
  \begin{align}
    \covnum{\normrff{\cdot}}{\rev{\dipoleSet_\eta}}{\coveps} & \rev{\leq \max\paren{1,\tfrac{12 C_\BasicSet(C_{\FClass}+C'_\FClass +DC''_\FClass)}{\delta} }^{4(d+1)}};\label{eq:covextr}\\
\covnum{\normrff{\cdot}}{\rev{\overline{\dipoleSet}_\eta}}{\coveps}    & \leq \max\left(1,\tfrac{64C_\BasicSet(C_{\FClass}+C'_{\FClass}+C''_\FClass)}{\delta}\right)^{2d+1}. \label{eq:covrest}
  \end{align}
  It is clear that~\eqref{eq:covdipole2},~\eqref{eq:covextr},~\eqref{eq:covrest}
  lead to the announced estimate~\eqref{eq:maindipcovestimate} (using $D\geq 1$).

{\bf Step 1: covering numbers of $\dipoleSet_\eta$.}
By the first part of Theorem~\ref{thm:covnumSecant} we can exploit Lemma~\ref{lem:covnumsecant} with
 $Y := \set{\alpha \Prob_{\Param}/\normkern{\Prob_{0}} : 0\leq \alpha \leq 1,~\Param \in \ParamSpace}$, 
 $A:=1$, $B:=\normrff{\dipoleSet} = D$, $\norma{\cdot} = \normrff{\cdot}$, and $\normb{\cdot} = \normkern{\cdot}$.
 Since $\rev{\dipoleSet_\eta}  := \set{\frac{y-y'}{\normkern{y-y'}}, (y,y') \in \mathcal{Q}, \normkern{y-y'} > \eta}$ 
with 
\rev{$\mathcal{Q}  := \mathcal{Q}_{1} \cup \mathcal{Q}_{2}$ and}
\rev{\begin{align*}
\mathcal{Q}_{1} & := \set{(\Prob_{\Param'},\alpha\Prob_{\Param})/\normkern{\Prob_{0}}, \Param,\Param' \in \ParamSpace, \norm{\Param-\Param'} \leq 1, 0 \leq \alpha\leq 1},\\ 
\mathcal{Q}_{2} & := \set{(\alpha \Prob_{\Param'},\Prob_{\Param})/\normkern{\Prob_{0}}, \Param,\Param' \in \ParamSpace, \norm{\Param-\Param'} \leq 1, 0 \leq \alpha\leq 1},
\end{align*}}
 we obtain
\begin{equation}\label{eq:CoverC3}
\covnum{\normrff{\cdot}}{\rev{\dipoleSet_\eta}}{\coveps}
\leq 
\covnumsq{\normrff{\cdot}}{Y}{\tfrac{\coveps\eta}{4(1+B/A)}}
\stackrel{B/A = D \geq 1}{\leq}
\covnumsq{\normrff{\cdot}}{Y}{\tfrac{\coveps\eta}{8 D}}
=
\covnumsq{\normrff{\cdot}}{Y}{\tfrac{\coveps^{2}}{64 DC''_{\FClass}}}.
\end{equation}
Denoting $\weightSet=[0,\ 1]$, we have $Y = [\psi(\ParamSpace)]_{1,\weightSet}$ \rev{(using the notation of~\eqref{eq:DefMixSet})}, where $\psi: \Param \mapsto \Prob_{\Param}/\normkern{\Prob_{\Param}} = \Prob_{\Param}/\normkern{\Prob_{0}}$. As $\norm{\weightSet}_{1}=1$ and $\normrff{\psi(\ParamSpace)} \leq C_{\FClass}$, by Lemma \ref{lem:covnummix} with $\tau = 1/2$ we get
\begin{equation}
\covnum{\normrff{\cdot}}{Y}{\tfrac{\coveps^{2}}{64DC''_{\FClass}}}
\leq
\covnum{\norm{\cdot}_1}{\weightSet}{\tfrac{\coveps^{2}}{128DC''_{\FClass}C_{\FClass}}}
\cdot
\covnum{\normrff{\cdot}}{\psi(\ParamSpace)}{\tfrac{\coveps^{2}}{128DC''_{\FClass}}}.
\end{equation}%
As $\weightSet =  \Ball_{\RR^{1},\abs{\cdot}_1}(1/2,1/2)$, by Lemma~\ref{lem:covnumball} we get
\(
\covnum{\norm{\cdot}_1}{\weightSet}{\tfrac{\coveps^{2}}{128DC''_{\FClass}C_{\FClass}}}
 \leq \max\left(1,\frac{256DC''_{\FClass}C_{\FClass}}{\coveps^{2}}\right).
\)
Moreover, \rev{from Lemma~\ref{lem:ShiftBasedRF},} $\psi$ is $L_{\FClass}=C'_{\FClass}$-Lipschitz
with respect to $\norm{\cdot}$ and $\normrff{\cdot}$, thus, by Lemma~\ref{lem:covnumlipschitz} and assumption~\eqref{eq:covparamspace2}:
\[
\covnum{\normrff{\cdot}}{\psi(\ParamSpace)}{\tfrac{\coveps^{2}}{128DC''_{\FClass}}}
\leq 
\covnum{\norm{\cdot}}{\ParamSpace}{\tfrac{\coveps^{2}}{128DC''_{\FClass}C'_{\FClass}}}
\rev{\leq 
\max\paren{1,\tfrac{128 C_{\BasicSet} D C''_{\FClass}C'_{\FClass}}{\coveps^{2}}}^{d}}
\]
Combining the above we obtain (using $D\geq 1, C_\BasicSet\geq 1$; and 
$2ab \leq (a+b)^{2}$ with $a = DC''_{\FClass},b = C'_{\FClass}+C_{\FClass}$)
\begin{align*}
\covnum{\normrff{\cdot}}{\rev{\dipoleSet_\eta}}{\coveps}
& \leq 
\left[
\max\left(1,\tfrac{256DC''_{\FClass}C_{\FClass}}{\coveps^{2}}\right)\right]^{2}
\cdot
\rev{\max\paren{1,\tfrac{128 C_{\BasicSet} D C''_{\FClass}C'_{\FClass}}{\coveps^{2}}}^{2d}}\notag\\
  & \rev{\leq \max\paren{1,\tfrac{256  C_\BasicSet D C''_{\FClass}(C'_\FClass +C_\FClass)}{\delta^2} }^{2(d+1)}} \notag\\
  &\rev{\leq \max\paren{1,\tfrac{12 C_\BasicSet(DC''_{\FClass}+C'_\FClass +C_\FClass)}{\delta} }^{4(d+1)},}    %
\end{align*}
i.e. we have obtained~\eqref{eq:covextr}.

{\bf Step 2: local tangent approximation of $\overline{\dipoleSet}_\eta$.}
To control $\covnum{\normrff{\cdot}}{\overline{\dipoleSet}_\eta}{\coveps}$, the principle will be
to approximate $\overline{\dipoleSet}_\eta$ by an appropriate ``tangent space'', then use
Lemma~\ref{lem:covnumClose}.

To this end, let $E$ denote the \rev{algebraic} dual of smooth functions \rev{that are bounded with bounded derivatives} on $\RR^\sampleDim$.
\rev{The semi-norm $\normrff{\cdot}$ is extended naturally to $E$ as $\normrff{\HH} := \sup_{\freq} \abs{ \langle \HH, \rfeat \rangle} \in [0,\infty]$ for any $\HH\in E$;
  let $\widetilde{E}:=\{\HH \in E: \normrff{\HH}<\infty\}$. Note that all finite signed measures are elements of $\widetilde{E}$.}    
  Given $\Param,\Delta \in \RR^{\sampleDim}$ and $\beta \in \RR$, define \rev{$\xi = \xi_{\Param,\Delta,\beta}\in \widetilde{E}$} by its action on functions $g: \RR^{\sampleDim} \to \CC$ that are bounded with bounded gradient:
\[
\inner{\xi,g} := 
\normkern{\Prob_{0}}^{-1} \cdot \Exp_{\Sample\sim  \Prob_{\Param}} \left\{\inner{\nabla g(\Sample),\Delta}+\beta g(\Sample)\right\}.
\]
Let $\Ball$ be the ball of radius $\rev{2}$ in $\RR^{\sampleDim} \times \RR$ equipped with the norm $\norm{(\Delta,\beta)}_{\mathtt{mix}}:= \norm{\Delta}+\abs{\beta}$. 

 Consider $\nu = \normkern{\Prob_0}^{-1} s(\Prob_{\Param'}-\alpha \Prob_{\Param})$ with $s \in \set{-1,+1}$, $0 \leq \alpha \leq 1$ and $0<\norm{\Param'-\Param} \leq 1$, 
and denote $t  := \normkern{\nu}$. We will show that there exists $(\Delta,\beta) \in \Ball$ such that $\HH := \nu/\normkern{\nu}$ satisfies 
\begin{equation}\label{eq:tangapprox}
\normrff{\HH-\xi_{\Param,\Delta,\beta}} \leq  C''_{\FClass} \norm{\Param'-\Param} \leq \rev{2} C''_{\FClass} t.
\end{equation}
Using this approximation, and the fact that to approximate any element of $\overline{\dipoleSet}_\eta$
we can assume $t\leq \eta = \delta/(8 C''_\FClass)$,
we apply Lemma~\ref{lem:covnumClose}  (with $Z=\overline{\dipoleSet}_\eta,Y=\set{\xi_{\Param,\Delta,\beta}, \Param \in \ParamSpace, (\Delta,\beta) \in \Ball}$, $\delta'=\eps=\delta/4$) to obtain
\begin{equation}
  \label{eq:tangcover}
    \covnum{\normrff{\cdot}}{\overline{\dipoleSet}_\eta}{\delta} = 
    \covnum{\normrff{\cdot}}{\overline{\dipoleSet}_\eta}{2(\delta'+\eps)}  
    \leq \covnum{\normrff{\cdot}}{\set{\xi_{\Param,\Delta,\beta}, \Param \in \ParamSpace, (\Delta,\beta) \in \Ball}}{\tfrac{\delta}{4}}.
    \end{equation}
We now prove~\eqref{eq:tangapprox}. Since $\kernel$ is shift-invariant and locally characteristic on $\BasicSet$ by Proposition~\ref{prop:DefKShift} we have $\normkern{\Prob_{\Param}} = \normkern{\Prob_{\Param'}} = \normkern{\Prob_{0}}>0$. 
Denote $x := \Param'-\Param$. Since $\kernel$ is $1$-strongly locally characteristic we have
\begin{align*}
t^{2} &= \normkern{\Prob_{0}}^{-2} \normkern{\Prob_{\Param'}-\alpha\Prob_{\Param}}^{2}
=
1+\alpha^{2}-2\alpha \nkernel(\Param',\Param)
=
(1-\alpha)^{2}+2\alpha(1- \nkernel(\Param',\Param))
 \geq 
(1-\alpha)^{2} + \alpha \norm{x}^{2}.
\end{align*}
Setting $\beta := s(1-\alpha)/t$ and $\Delta := sx/t$ we get using Lemma~\ref{lem:gxybound} %
\[
  \norm{\Delta}+\abs{\beta} = \frac{1-\alpha+\norm{x}}{t} \leq \frac{(1-\alpha)+\norm{x}}{\sqrt{(1-\alpha)^{2}+\alpha \norm{x}^{2}}} = g(1-\alpha,\norm{x}) \leq \rev{2}.
\]
Since $\abs{\inner{\freq,\Delta}} \leq \norm{\freq}_{\star}\norm{\Delta}$ and $\norm{x}^{2}/t = (\norm{x}/t) \norm{x} = \norm{\Delta} \norm{x} \leq \rev{2} \norm{x}$, by
a Taylor expansion with integral remainder term we obtain
\begin{align}
\abs{s(e^{\jmath \inner{\freq,x}}-1)/t-\jmath \inner{\freq,\Delta}}
&=
\abs{(e^{\jmath \inner{\freq,x}}-1)/t-\jmath \inner{\freq,s\Delta}}
=
\abs{(e^{\jmath t\inner{\freq,s\Delta}}-1)/t-\jmath \inner{\freq,s\Delta}} \notag \\
&\leq
\sup_{0 \leq \tau \leq t} \abs{\frac{d^2}{dt^2}e^{\jmath \tau \inner{\freq,s\Delta}}} \cdot \frac{t}{2}
=  \inner{\freq,s\Delta}^{2}\frac{t}{2}
\leq \norm{\freq}_{\star}^{2} \norm{\Delta}^{2} \frac{t}{2}
=
\norm{\freq}_{\star}^{2} \frac{\norm{x}^{2}}{2t} \notag \\
& \leq \norm{\freq}_{\star}^{2}  \norm{x}. \label{eq:qapprox1}
\end{align}
For each $\freq$, since $\inner{\Prob_{\Param},\rfeat} = e^{\jmath \inner{\freq,\Param}} \inner{\Prob_{0},\rfeat}$
we have
\begin{align*}
\normkern{\Prob_{0}} \inner{\nu,\rfeat} e^{-\jmath \inner{\freq,\Param}}
& = 
 s\inner{\Prob_{x}-\alpha \Prob_{0},\rfeat} 
= s\inner{\Prob_{0},\rfeat} \left(e^{\jmath \inner{\freq,x}}-1+1-\alpha\right)
= t  \inner{\Prob_{0},\rfeat} \left(s(e^{\jmath \inner{\freq,x}}-1)/t+\beta\right).
\end{align*}
Since $\rfeat$ and its gradient $\nabla \rfeat = \rfeat \cdot \jmath \freq$ are bounded on $\RR^{\sampleDim}$, with $\xi := \xi_{\Param,\Delta,\beta}\in E$ we have
\begin{align*}
\normkern{\Prob_{0}} \normkern{\nu} 
\inner{\xi,\rfeat}e^{-\jmath \inner{\freq,\Param}}  
&=  t \normkern{\Prob_{0}}
 \inner{\xi,\rfeat} e^{-\jmath \inner{\freq,\Param}} 
=
t \cdot \Exp_{\Sample\sim  \Prob_{\Param}} \left\{ (\jmath\inner{ \freq ,\Delta}+\beta)\rfeat(X) \right\}
e^{-\jmath \inner{\freq,\Param}}\\
  &=  t \inner{\Prob_0,\rfeat} \cdot  \left(\jmath \inner{\freq,\Delta}+\beta\right),
\end{align*}
thus from the last two displays and~\eqref{eq:qapprox1} we obtain
\begin{align*}
\normkern{\Prob_{0}} \abs{\inner{\nu- \normkern{\nu}\xi,\rfeat}}
&= 
t \abs{\inner{\Prob_{0},\rfeat}} \cdot \abs{s(e^{\jmath \inner{\freq,x}}-1)/t-\jmath \inner{\freq,\Delta}}
 \leq t  \abs{\inner{\Prob_{0},\rfeat}} \norm{\freq}_{\star}^{2}  \norm{x}.
\end{align*}
Dividing both hand sides by $\normkern{\Prob_{0}}\normkern{\nu} = t\normkern{\Prob_{0}}$ and taking the supremum over $\freq$ yields 
\[
\normrff{\mu-\xi} 
\leq 
\normkern{\Prob_{0}}^{-1} 
\sup_{\freq} \left(\abs{\inner{\Prob_{0},\rfeat}} \norm{\freq}_{\star}^{2}\right) 
\norm{x}
= 
\normkern{\Prob_{0}}^{-1}  \normrffdd{\Prob_{0}} \norm{x}
=  C''_{\FClass} \norm{x}.
\]
We conclude using that $\norm{\Param'-\Param} = \norm{x} = \norm{\Delta} t \leq 2 t$.

{\bf Step 3: $\coveps$-covering of $\set{\xi_{\Param,\Delta,\beta}}$.}
Define $\coveps_{1} := \coveps/(\rev{4}(C''_{\FClass}+C'_{\FClass}))$, $\coveps_{2} := \coveps/(2 (C'_{\FClass}+C_{\FClass}))$, and consider $\enet_{1}$ a $\coveps_{1}$-cover of $\ParamSpace$ with respect to $\norm{\cdot}$
and $\enet_{2}$ a $\coveps_{2}$-cover of $\Ball$ with respect to $\norm{\cdot}_{\mathtt{mix}}$.
\rev{In order to exploit~\eqref{eq:tangcover}, we} now show that $\set{\xi_{\tilde{\Param},\Delta',\beta'}, \tilde{\Param} \in \enet_{1},(\Delta',\beta') \in \enet_{2}}$ is a $\coveps$-covering of $\set{\xi_{\Param,\Delta,\beta}, \Param \in \ParamSpace, (\Delta,\beta) \in \Ball}$.

Given any $\Param \in \ParamSpace$, $(\Delta,\beta) \in \Ball$ there are $\tilde{\Param} \in \enet_{1}$, $(\Delta',\beta') \in \enet_{2}$ such that $\norm{\tilde{\Param}-\Param} \leq \coveps_{1}$, and $\norm{(\Delta',\beta')-(\Delta,\beta)}_{\mathtt{mix}} \leq \coveps_{2}$. Observe that 
$\abs{e^{\jmath \inner{\freq,\tilde{\Param}-\Param}}-1} \leq \abs{\inner{\freq,\tilde{\Param}-\Param}} \leq \norm{\freq}_{\star}\norm{\tilde{\Param}-\Param}$, and%
\begin{align*}
\abs{\inner{\jmath\freq,\Delta}+\beta} 
&\leq \norm{\freq}_{\star}\norm{\Delta}+\abs{\beta} 
\leq \max(\norm{\freq}_{\star},1) \cdot \norm{(\Delta,\beta)}_{\mathtt{mix}}
 \rev{\leq 2}\max(\norm{\freq}_{\star},1) \leq \rev{2}(\norm{\freq}_{\star}+1).
\end{align*}
Since $\normkern{\Prob_{0}} \inner{\xi_{\Param,\Delta,\beta},\rfeat} =  \inner{\Prob_{0},\rfeat} e^{\jmath \inner{\freq,\Param}} (\jmath \inner{\freq,\Delta}+\beta)$ (and similarly with $\tilde{\Param}$, $\Delta'$, $\beta'$) we get
\begin{align*}
\normkern{\Prob_{0}} 
\abs{\inner{\xi_{\Param,\Delta,\beta}-\xi_{\tilde{\Param},\Delta,\beta},\rfeat}}
& =
\abs{\inner{\Prob_0,\rfeat}} \cdot \abs{
\left(e^{\jmath \inner{\freq,\tilde{\Param}-\Param}}-1\right)  (\jmath \inner{\freq,\Delta}+\beta)}\\
& \leq 
\abs{\inner{\Prob_0,\rfeat}} \cdot
\norm{\freq}_{\star} \cdot \norm{\tilde{\Param}-\Param} \cdot \rev{2} (\norm{\freq}_{\star}+1)\\
  & \leq \abs{\inner{\Prob_0,\rfeat}} \cdot (\norm{\freq}_{\star}^{2}+\norm{\freq}_{\star}) \cdot \rev{2} \coveps_{1},
\end{align*}
so that
\begin{equation}
  \label{eq:covest1}
\normkern{\Prob_{0}}
\normrff{\xi_{\Param,\Delta,\beta}-\xi_{\tilde{\Param},\Delta,\beta}}
 \leq 
\sup_{\freq} \set{
\abs{\inner{\Prob_0,\rfeat}} \cdot (\norm{\freq}_{\star}^{2}+\norm{\freq}_{\star})} \cdot \rev{2} \coveps_{1}
\leq (\normrffdd{\Prob_{0}}+\normrffd{\Prob_{0}}) \rev{2} \coveps_{1}.
\end{equation}
On the other hand,
\begin{align*}
 \normkern{\Prob_{0}}
\abs{\inner{\xi_{\tilde{\Param},\Delta,\beta}-\xi_{\tilde{\Param},\Delta',\beta'},\rfeat}}
& =
\abs{\inner{\Prob_0,\rfeat}} \cdot \abs{
\jmath \inner{\freq,(\Delta'-\Delta)}+(\beta'-\beta)}\\
&\leq 
\abs{\inner{\Prob_0,\rfeat}} \cdot 
(\norm{\freq}_{\star}+1) \cdot \norm{(\Delta',\beta')-(\Delta,\beta)}_{\mathtt{mix}}\\
& \leq \abs{\inner{\Prob_0,\rfeat}} \cdot 
                                                                                         (\norm{\freq}_{\star}+1) \cdot \coveps_{2},
\end{align*}
so that
\begin{equation}
  \label{eq:covest2}
\normkern{\Prob_{0}}
\normrff{\xi_{\tilde{\Param},\Delta,\beta}-\xi_{\tilde{\Param},\Delta',\beta'}}
 \leq 
\sup_{\freq}
 \set{ \abs{\inner{\Prob_0,\rfeat}} \cdot 
(\norm{\freq}_{\star}+1)} \cdot \coveps_{2}
\leq (\normrffd{\Prob_{0}}+1) \coveps_{2}.
\end{equation}
By a triangle inequality we combine~\eqref{eq:covest1}-\eqref{eq:covest2} to get
 \begin{align*}
\normrff{\xi_{\Param,\Delta,\beta}-\xi_{\tilde{\Param},\Delta',\beta'}}
& \leq (C''_{\FClass}+C'_{\FClass})\rev{2}\coveps_{1} + (C'_{\FClass}  + C_{\FClass}) \coveps_{2}
 = \coveps.
 \end{align*}
 \rev{ To conclude this step, we have established that
   \begin{align}
     \covnum{\normrff{\cdot}}{\set{\xi_{\Param,\Delta,\beta}, \Param \in \ParamSpace, (\Delta,\beta) \in \Ball}}{\delta}
     &\leq 
     \covnum{\norm{\cdot}}{\ParamSpace}{\tfrac{\delta}{4(C''_{\FClass}+C'_{\FClass})}}
     \covnum{\norm{\cdot}_{\mathtt{mix}}}{\Ball}{\tfrac{\delta}{2(C'_{\FClass}+C_{\FClass})}}\notag\\
     & \leq      \max\left(1, \tfrac{4C_{\BasicSet} (C''_{\FClass}+C'_{\FClass})}{\coveps}\right)^{d}
	\max\left(1, \tfrac{16(C'_{\FClass}+C_{\FClass})}{\coveps}\right)^{\sampleDim+1}\notag\\
     & \leq \max\left(1,\tfrac{16C_\BasicSet(C_{\FClass}+C'_{\FClass}+C''_\FClass)}{\delta}\right)^{2d+1},
       \label{eq:covnumtgt}
   \end{align}
   using assumption~\eqref{eq:covparamspace2} and Lemma~\ref{lem:covnumball} for the second estimate, since $\Ball$ is a ball of radius \rev{$2$} with respect to $\norm{\cdot}_{\mathtt{mix}}$ in $\RR^{\sampleDim+1}$;
   and finally $C_\BasicSet \geq 1$ for the last estimate. Plugging in~\eqref{eq:covnumtgt} into~\eqref{eq:tangcover} yields~\eqref{eq:covrest}, and the proof is done. }
 \end{proof}

\subsection{Kernel mean embedding for Gaussians}\label{app:KernelGaussian}
The following lemma characterizes the mean map kernel on any pair of Gaussians.
\begin{lemma}\label{lem:GaussKernelMeanEmbedding}
Consider a Gaussian kernel $\kernel_{\mR}(\sample,\sample'):=\exp\left(-\tfrac12 \normmah{\sample-\sample'}{\mR}^2\right)$, where $\mR$ is an arbitrary invertible covariance matrix. For any two Gaussians $\Prob_1=\mathcal{N}(\Param_1,\covar_1),~\Prob_2=\mathcal{N}(\Param_2,\covar_2)$, the mean kernel defined from $\kernel_{\mR}$ using~\eqref{eq:DefMeanMapEmbedding} is
\begin{equation}
\label{eq:MeanKernelGaussians}
\kernel_{\mR}(\Prob_1,\Prob_2)=\frac{\sqrt{\det{\mR}}}{\sqrt{\det{\covar_1+\covar_2+\mR}}}\exp\left(-\tfrac12 \normmah{\Param_1-\Param_2}{\covar_1+\covar_2+\mR}^2\right).
\end{equation}
\end{lemma}
\begin{proof}
As 
\(
\kernel_{\mR}(\sample,\sample')
= \sqrt{\det{2\pi\mR}} \cdot \Prob_{\mR}(\sample-\sample')
\)
where $\Prob_{\mR}=\mathcal{N}(0,\mR)$,
we have
\begin{align*}
\kernel_{\mR}(\Prob_1,\Prob_2)
=&
\sqrt{\det{2\pi\mR}}
\int_\sample\Prob_1(\sample)
\underbrace{\left(\int_{\sample'} \Prob_2(\sample')\Prob_{\mR}(\sample-\sample')d\sample'\right)}_{\Prob_{2} \star \Prob_{\mR}=\mathcal{N}(\Param_{2},\covar_{2}+\mR)=: \mathcal{N}(\Param_{3},\covar_{3})=\Prob_{3}}d\sample 
\end{align*}
We conclude using a property on products of Gaussians \cite[Equation (5.6)]{Ahrendt2005}.
\[
\int \Prob_1(\sample)\Prob_3(\sample)d\sample= \frac{1}{
\sqrt{\det{2\pi(\covar_1+\covar_3)}}}\exp\left(-\tfrac12 \normmah{\Param_1-\Param_3}{\covar_1+\covar_3}^2\right).\qedhere
\]
\end{proof}

\subsection{Proof of Lemma~\ref{lem:GaussianKernel}}\label{sec:proofLemmaSigmaK}
Denote $K = K_{\sigma}$ for brevity. If $u \geq \sigma$ then $1-K(u) \geq 1-e^{-1/2} \approx 0.39 > 1/3$. Now, if $0<u \leq \sigma$: by concavity of the function $t \mapsto 1-e^{-t/2\sigma^{2}}$ on the interval $[0, \sigma^{2}]$, we have 
\[
1-e^{-t/2\sigma^{2}} \geq %
(t/\sigma^2) \cdot (1-e^{-1/2})/ > t/3\sigma^{2}
\]
for $0 \leq t \leq \sigma^{2}$, hence with $t = u^{2}$ we get $1-K(u) \geq  u^{2}/3\sigma^{2}$. This shows that $1-K(u) \geq \min(1,(u/\sigma)^{2})/3$.

If $\sigma^{2} \leq 1/2$ then $t \mapsto h(t) := (1-t/2)\exp(\tfrac{t}{2\sigma^{2}})$ is non-decreasing on $[0,1]$ with $h(0)=1$. For $0 \leq u \leq 1$ we obtain $(1-u^{2}/2)/K(u) = h(u^{2}) \geq 1$ hence $K(u) \leq  1-u^{2}/2$. %

We have $K'(u) = -\tfrac{u}{\sigma^{2}}\ \exp(-\tfrac{u^{2}}{2\sigma^{2}})$, $K''(u) = (\tfrac{u^{2}}{\sigma^{2}}-1)\exp(-\tfrac{u^2}{2\sigma^{2}})/\sigma^{2}$, $K'''(u) = (3-\tfrac{u^{2}}{\sigma^{2}})\frac{u}{\sigma^{4}}\exp(-\tfrac{u^2}{2\sigma^{2}})$.
Since $\sigma^{2} \leq 1/4$, for $u \geq 1 \geq \sigma$  we have $K''(u) \geq 0$. Hence, $K'$ is negative and increasing on $[1,\infty)$ and we get $K'_{\max} = |K'(1)| = \exp(-\tfrac{1}{2\sigma^{2}})/\sigma^{2}$. Since $K'_{\max}> K_{\max}=K(1)$ we have (cf~\eqref{eq:MutualCoherenceRBF1})
\[
C(K) = \max(K_{\max},\rev{2K'_{\max}+K''_{\max})}) \rev{\leq} 2(K'_{\max}+K''_{\max}).
\]
Similarly, since $\sigma^2\leq 1/3$, for $u \geq 1 \geq \sqrt{3}\sigma$ we have $K'''(u) \leq 0$ hence $K''$ is positive decreasing on $[1,\infty)$ and $K''_{\max} = K''(1) = \tfrac{1}{\sigma^{2}}\left(\tfrac{1}{\sigma^{2}}-1\right) \exp(-\tfrac{1}{2\sigma^{2}})$. 
As a result $K'_{\max}+K''_{\max} = \tfrac{1}{\sigma^{4}}\exp(-\tfrac{1}{2\sigma^{2}})$ and $C(K) \rev{\leq}  \tfrac{2}{\sigma^{4}}\exp(-\tfrac{1}{2\sigma^{2}})$.

Given $c\geq 2$, putting $\sigma^\star_k := (\sqrt{2c\log(ek)})^{-1} \in [0,\tfrac{1}{2}]$,
  since the function $t \mapsto t^2 \exp(-t/2)$ is nonincreasing for $t\geq 4$,
  it holds for any $\sigma \leq \sigma^\star_k$ that
\begin{align*}
  C(K) \leq 2 g(1/\sigma^{2}) \leq 2 g(1/(\sigma_{k}^{\star})^{2}) &= 2 (\sigma^\star_k)^{-4}\exp\left(-{\tfrac{1}{2(\sigma^\star_k)^{2}}}\right)
    = 8c^2 \cdot \log^{2} (ek) \cdot (ek)^{-c} \\
  & = \frac{8c^{2}}{2k-1} \cdot \log^{2} (ek) \cdot (ek)^{-c}(2k-1)
  \leq \frac{8c^{2}e^{-c}}{2k-1} ,
\end{align*}
where we used at the last inequality that the function $k\mapsto  \log^{2} (ek) \cdot (ek)^{-c}(2k-1)$
is nonincreasing for $k\geq 1$ if \rev{$c\geq 4$}. The choice $c=8$ leads to
$C(K) \leq \frac{3}{\rev{16}(2k-1)}$.

\subsection{Proof of Equation~\eqref{eq:MomentsDiracGaussianTechnical}} \label{app:MomentsDipolesDiracGauss}

For Diracs since $\norm{\cdot} = \norm{\cdot}_{2}/\sep$ and $\norm{u} = 1$ 
we write $u = \sep v$ where $\norm{v}_2 = 1$. With the probability distribution $\freqdist$ on $\freq$~from~\eqref{eq:randomfeaturesamplingdensity}, since $\abs{\inner{\Prob_{0},\rfeat}} = 1/w(\freq) \leq 1$, $\mGamma = s^{-2}\mI_{\sampleDim}$ (see Definition~\ref{def:DiracGaussian}) and $C_{\freqdist}^{-2} = \normkern{\Prob_{0}}^{2}$ (see~\eqref{eq:P0NormInvSq}), we obtain 	
\begin{eqnarray*} 
\Exp_{\freq \sim \freqdist} \set{\abs{\inner{\Prob_{0},\rfeat}}^{2q} \inner{\freq,u}^{2q}}
&=&
\int_{\RR^{\sampleDim}} 
w^{-2q}(\freq)
\inner{\freq,\sep v}^{2q}
C_{\freqdist}^{-2}w^{2}(\freq) p_{\mathcal{N}(0,s^{-2}\mI_\sampleDim)}(\freq) d\freq\notag\\
 &\stackrel{w \geq 1,q \geq 1}{\leq}&
 \normkern{\Prob_{0}}^{2} \sep^{2q} \int_{\RR^{\sampleDim}} 
\inner{\freq,v}^{2q}
p_{\mathcal{N}(0,s^{-2}\mI_\sampleDim)}(\freq) d\freq\notag\\
 &{=}&
\normkern{\Prob_{0}}^{2} \sep^{2q } s^{-2q} \cdot \Exp_{\freq \sim \mathcal{N}(0,s^{-2}\mI_\sampleDim)} 
\inner{s\freq,v}^{2q}\notag\\
 &\stackrel{\freq' := s\freq}{=} &
\normkern{\Prob_{0}}^{2}
(\sep/s)^{2q} \cdot \Exp_{\freq' \sim \mathcal{N}(0,\mI_\sampleDim)} \inner{\freq',u}^{2q} 
\stackrel{(*)}{=} 
\normkern{\Prob_{0}}^{2}(\sep/s)^{2q} \cdot 
\Exp_{\xi \sim \mathcal{N}(0,1)} \xi^{2q}
\end{eqnarray*}
where in (*) we use that as $\norm{v}_{2} = 1$ and $\freq' \sim \mathcal{N}(0,\mI_\sampleDim)$, $\xi := \inner{\freq',u}$ is standard Gaussian.

For Gaussians since $\norm{\cdot} = \normmah{\cdot}{\covar}/\sep$ and $\norm{u}=1$ 
we write $u = \sep \covar^{1/2} v$ where $\norm{v}_{2}=1$. For $q \geq 1$ we have $(1+2qs^{-2})^{-\sampleDim/2} \leq (1+2s^{-2})^{-\sampleDim/2} = \normkern{\Prob_{0}}^{2}$ (see~\eqref{eq:P0NormInvSq}). Since  $\abs{\inner{\Prob_{0},\rfeat}} := e^{-\freq^{T} \Cov \freq/2} \leq 1$ 
and
$\freqdist(\freq) := p_{\mathcal{N}(0,s^{-2}\Cov^{-1})}(\freq)$ (see~\eqref{eq:randomfeaturesamplingdensity} and Definition~\ref{def:DiracGaussian}), we obtain
\begin{eqnarray*}
\abs{\inner{\Prob_{0},\rfeat}}^{2q} \freqdist(\freq)
&=& \sqrt{\det{2\pi s^2\Cov}} e^{-(2q+s^2)\freq^{T}\Cov\freq/2}
= \tfrac{\sqrt{\det{2\pi s^{2}\Cov}}}{\sqrt{\det{2\pi (2q+s^{2})\Cov}}} 
p_{\mathcal{N}(0,(2q+s^{2})^{-1}\Cov^{-1})}(\freq)\\
&=& (1+2qs^{-2})^{-\sampleDim/2}p_{\mathcal{N}(0,(2q+s^{2})^{-1}\Cov^{-1})}(\freq)
\leq \normkern{\Prob_{0}}^{2} \cdot p_{\mathcal{N}(0,(2q+s^{2})^{-1}\Cov^{-1})}(\freq),
\end{eqnarray*}
hence
\begin{eqnarray*}
\Exp_{\freq \sim \freqdist} \abs{\inner{\Prob_{0},\rfeat}}^{2q} \inner{\freq,u}^{2q}
&:=&
\int_{\RR^{\sampleDim}} 
\abs{\inner{\Prob_{0},\rfeat}}^{2q} \inner{\freq,\sep\Cov^{1/2}v}^{2q}
 \freqdist(\freq) d\freq\\
&\leq &
\normkern{\Prob_{0}}^{2}\sep^{2q} \int_{\RR^{\sampleDim}} 
\inner{\covar^{1/2}\freq,v}^{2q}
p_{\mathcal{N}(0,(2q+s^{2})^{-1}\Cov^{-1})}(\freq) 
d\freq\\
&=&
\normkern{\Prob_{0}}^{2}\sep^{2q}\cdot \Exp_{\freq \sim \mathcal{N}(0,(2q+s^{2})^{-1}\Cov^{-1})} 
\inner{\covar^{1/2}\freq,v}^{2q}\notag\\
& \stackrel{\freq' := \Cov^{1/2}\freq}{=} &
\normkern{\Prob_{0}}^{2}\sep^{2q} \cdot \Exp_{\freq' \sim \mathcal{N}(0,(2q+s^{2})^{-1}\mI_\sampleDim)} \inner{\freq',v}^{2q} \notag\\
& \stackrel{\freq = \sqrt{2q+s^{2}}\freq'}{=} &
\normkern{\Prob_{0}}^{2}\sep^{2q}(\sqrt{2q+s^{2}})^{-2q} \cdot
\Exp_{\freq \sim \mathcal{N}(0,\mI_{\sampleDim})} \inner{\freq,v}^{2q} \\
&\stackrel{q \geq 1, (*)}{\leq}&
\normkern{\Prob_{0}}^{2}(\sep/\sqrt{2+s^{2}})^{2q}
\cdot \Exp_{\xi \sim \mathcal{N}(0,1)} \xi^{2q},\notag
\end{eqnarray*}
where in (*) we reasoned as for Diracs. Finally it is known that for any integer $q \geq 1$, 
$\Exp_{\xi \sim \mathcal{N}(0,1)} \xi^{2q} = (2q-1)!!$, where
\(
(2q-1)!! = \prod_{i=1}^{q} (2i-1) \leq 2^{q-1} q!\;.
\)
Using ~\eqref{eq:DefDiracGaussianVariance} we recognize that in both cases
\[
\Exp_{\freq \sim \freqdist} \abs{\inner{\Prob_{0},\rfeat}}^{2q} \inner{\freq,u}^{2q} 
\leq \normkern{\Prob_{0}}^{2}(\sep^{2}/\sigma^{2}(s))^{q}  \Exp_{\xi \sim \mathcal{N}(0,1)} \xi^{2q}
\leq \normkern{\Prob_{0}}^{2}(2\sep^{2}/\sigma^{2}(s))^{q} \tfrac{q!}{2}. \qedhere
\]
\subsection{Proof of Lemma~\ref{lem:SeparationNecessary}}
\label{sec:separationnecessaryDiracs}

\begin{proof}
\rev{We exhibit two probability distributions $\mProb,\mProb' \in \ModelCT(\HypClass_{k,\sep,R})$ such that $\dnormloss{\mProb'-\mProb}{}/\normkern{\mProb'-\mProb}$ is bounded from below.}
Consider $\Param_{0} \in \RR^{\sampleDim}$ with $\norm{\Param_{0}}_{2}=1$
and set $\Param_{+} := \tfrac{\sep}{2} \Param_{0}, \Param_{-} = -\Param_{+}$ (hence $\norm{\Param_{+}} = \norm{\Param_{-}} = \sep/2 \leq R/2$ for small enough $\sep$). Observe that $\sep =\norm{\Param_{+}-\Param_{-}}_2$. Setting $\alpha :=\tfrac{R}{2\sep}$, define $\hyp = (c_{1},\ldots,c_{k})$ with $c_{1} = c_{+}$, $c_{l}=c_{-}$ for $l \geq 2$ where 
\begin{align*}
c_{+}=\Param_{+} + \alpha(\Param_{+} - \Param_{-}),\quad c_{-}=\Param_{-} + \alpha(\Param_{-} - \Param_{+}).
\end{align*}
Since $\norm{c_{+}}\leq R$, $\norm{c_{-}}\leq R$ \rev{and $\norm{c_{+}-c_{-}}_{2} = (1+\alpha) \norm{\Param_{+}-\Param_{-}}_{2} = \sep+R/2$} we have $\hyp  \in \HypClass_{\PCAdim,\rev{\sep},R}$.

Define two mixtures $\mProb = \frac12 (\delta_{\Param_{+}} + \delta_{\Param_{-}}),~\mProb' = \delta_{0} \in\rev{\ModelCT(\HypClass_{k,\sep,R})}$. 
As $\loss(\Param_{+},\hyp) =  \loss(\Param_{-},\hyp)=(\alpha\sep)^{p}=(R/2)^{p}$ and $\loss\paren{0,\hyp} = \paren{1/2+\alpha}^{p}\sep^{p}=(R/2)^{p}\paren{1+\sep/R}^{p}$, we have
\begin{align}\label{eq:necessaryproofLoss}
\inner{\mProb' - \mProb,\loss(\cdot,\hyp)}
=~(R/2)^{p}\paren{\paren{1+\sep/R}^{p}-1}. %
\end{align}
\rev{Now set $\hyp' = (c'_{1},\ldots,c'_{k})$ with $c'_{1}=0$, for all $l$. Again, $\hyp' \in \HypClass$ and, as $\loss(\Param_{+},\hyp') = \loss(\Param_{-},\hyp') = (\sep/2)^{p}$ and $\loss\paren{0,\hyp'} = 0$ we have $\inner{\mProb' - \mProb,\loss(\cdot,\hyp')}
=~-(\sep/2)^{p}$. For $p \in \{1,2\}$ we get
\begin{align}\label{eq:necessaryproofDLoss}
\inner{\mProb' - \mProb,\loss(\cdot,\hyp)-\loss(\cdot,\hyp')}
=~(R/2)^{p}\paren{\paren{1+\sep/R}^{p}-1}+(\sep/2)^{p} \geq p(R/2)^{p} \frac{\sep}{R} .
\end{align}
This yields the lower bound $\dnormlossf{\mProb - \mProb'}{} \geq \abs{\inner{\mProb' - \mProb,\loss(\cdot,\hyp)-\loss(\cdot,\hyp')}} \geq p(R/2)^{p} \tfrac{\sep}{R}$. }

We now upper bound $\norm{\SketchingOperatorProb(\mProb)-\SketchingOperatorProb(\mProb')}_{2}$.
\rev{Denote $g(t) := \SketchingOperator(t \Param_0)$, $t \in \RR$, by assumption
  the function $g$ is of class $\mathcal{C}^2$ and
  \[
    \norm{\SketchingOperatorProb(\mProb)-\SketchingOperatorProb(\mProb')}_{2}
    =\norm{ g(0) - \frac{1}{2}g\paren{-\frac{\sep}{2}} - \frac{1}{2}g\paren{\frac{\sep}{2}} }
    = \frac{\sep^2}{8} \norm{g''(0)} + o(\sep^2),
   \]
}
Given 
\rev{\eqref{eq:necessaryproofDLoss}}
we get a constant $c_{\SketchingOperator}>0$ such that \rev{for small enough $\sep$ and any $R \geq \sep$,\\ $\dnormlossf{\mProb - \mProb'}{} / \norm{\SketchingOperatorProb(\mProb)-\SketchingOperatorProb(\mProb')}_{2} \geq c_{\SketchingOperator} R^{p-1} /\sep$.} %
\end{proof}

\subsection{Link between risks with and without $\sep$-separation} \label{sec:link_risk_separation}

\begin{proof}[Proof of Lemma~\ref{lem:ClustRiskLipschitzWrtDist}]
First, denoting $\eta := d(\hyp,\hyp')$ with $\hyp = (c_{1},\ldots,c_{k})$, $\hyp' = (c'_{1},\ldots,c'_{k})$, we show that for any $\sample \in \RR^{\sampleDim}$ we have
\[
\min_{1 \leq j \leq k} \norm{\sample-c'_{j}} \leq \min_{1 \leq i \leq k} \norm{\sample-c_{i}} + \eta
\]
Indeed, denoting  $i^{\star}$ such that $\norm{\sample-c_{i^{\star}}}_{2} = \min_{1 \leq i \leq k} \norm{\sample-c_{i}} $ and $j^{\star}$ such that $\norm{c_{i^{\star}}-c'_{j^{\star}}}_{2} \leq d(\hyp,\hyp')$, we obtain with the triangle inequality 
\begin{align*}
\min_{1 \leq j \leq k} \norm{\sample-c'_{j}} & \leq \norm{\sample-c'_{j^{\star}}}_{2}
 = \norm{\sample-c_{i^{\star}}+c_{i^{\star}}-c'_{j^{\star}}}_{2} \\
& \leq \norm{\sample-c_{i^{\star}}}_{2}+\norm{c_{i^{\star}}-c'_{j^{\star}}}_{2}
   \leq\min_{1 \leq i \leq k} \norm{\sample-c_{i}} + \eta.
\end{align*}
For $k$-medians we obtain $\loss(\sample,\hyp') \leq \loss(\sample,\hyp) + \eta$ hence
\begin{align*}
\Risk_{k-\mathtt{medians}}(\Prob,\hyp') &= \Exp_{\Sample \sim \Prob}  \loss(\Sample,\hyp')
\leq \Exp_{\Sample \sim \Prob}  \loss(\Sample,\hyp) + \eta
\leq  \Risk_{k-\mathtt{medians}}(\Prob,\hyp) + \eta.
\end{align*}
For $k$-means we have instead
\begin{equation*}
\begin{split}
\Risk_{k-\mathtt{means}}(\Prob,\hyp')
&\leq \Exp_{\Sample \sim \Prob} \left(  \min_{1 \leq i \leq k} \norm{\Sample-c_{i}}_{2} + \eta\right)^2   
=  \Risk_{k-\mathtt{means}}(\Prob,\hyp) +2 \eta \Risk_{k-\mathtt{medians}}(\Prob,\hyp) +\eta^2
\end{split}
\end{equation*}
With Jensen's inequality we have $\Risk_{k-\mathtt{medians}}(\Prob,\hyp) \leq \sqrt{\Risk_{k-\mathtt{means}}(\Prob,\hyp)}$, yielding
\begin{equation*}
\sqrt{\Risk_{k-\mathtt{means}}(\Prob,\hyp')}
\leq  \sqrt{\Risk_{k-\mathtt{means}}(\Prob,\hyp)} + \eta.
\end{equation*}
Exchanging the role of $\hyp$ and $\hyp'$ yields a complementary inequality which completes the proof.
\end{proof}

\begin{proof}[Proof of Lemma~\ref{lem:DistToSeparatedSet}]
  Let $\vc=(c_1,\ldots,c_k) \in \HypClass_{k,0,R}$. We proceed by constructing in a greedy way an
  $\sep$-separated subset of $C_1:=\set{c_1,\ldots,c_k}$ which is also an $\sep$-cover of that set.
  The construction is standard. Starting at $i=1$, pick any
  $c'_i \in C_i$; put $C_{i+1} := C_i \setminus B(c'_i,\sep)$ (where $B(c,\sep)$ denotes the open
  ball of center $c$ and radius $\eps$). Iterate until $C_{i+1}=\emptyset$; denote $i^*$ the
  last iteration. Since the cardinality
  of $C_i$ is decreasing, we have $i^*\leq k$ iterations. Let $\vc'=(c_1',c'_2,\ldots,c'_{i^*},c'_{i^*},
  \ldots,c'_{i^*})$ (the last element is repeated as needed to attain $k$ centroids).
  Since $\set{c'_1,\ldots,c'_{i^*}} \subset \set{c_1,\ldots,c_k}$, obviously $d(\vc'\|\vc)=0$.
  On the other hand, by construction $\set{c_1,\ldots,c_k} \subset \bigcup_{1 \leq j \leq i^*} B(c'_j,\sep)$
  so that $d(\vc\|\vc')< \sep$. Finally, also by construction for any $i<i^*$, $c'_{i+1} \not\in \bigcup_{1 \leq j \leq i} B(c'_j,\sep)$ and therefore $\vc'$ is $\sep$-separated, so that $\vc' \in \HypClass_{k,\sep,R}$.
  Additionally, by the above construction it is clear that any $\sep$-isolated centroid of $\vc$
  must be selected (once) at some iteration as one of the centroids $c_j', 1\leq j\leq i^*$.
\end{proof}

To prove Lemma~\ref{lem:interpretableclusteringriskbound} we first establish a refined version of Lemma~\ref{lem:DistToSeparatedSet}.

 \begin{lemma}\label{lem:DistToSeparatedSet2}
  Given $\rev{\sep \geq 0}$ and $\vc \in \HypClass_{k,0,R}$, there exists $\vc'\in \HypClass_{k,\sep,R}$ such that \rev{$d(\vc,\vc') = d(\vc,\HypClass_{k,\sep,R})$ and such that all $2\sep$-isolated centroids
    of $\vc$,  $\set{c_{i}, i \in I_{2\sep}(\vc)}$ are centroids
  of $\vc'$  (repeated centroids with indices in $I_{2\sep}(\vc)$  may appear only once in $\vc'$.)}.
\end{lemma}
\begin{proof}
  Let $\vc'$ such that $\vc' \in \HypClass_{k,\sep,R}$ and $d(\vc,\vc')=d(\vc,\HypClass_{k,\sep,R})$
  (the distance $d(\vc,\HypClass_{k,\sep,R})$ is attained, since $\HypClass_{k,\sep,R}$ is a compact set).
  We know by Lemma~\ref{lem:DistToSeparatedSet} that $d(\vc,\vc')< \sep$ must hold.
  Let $c_i$ be any $2\sep$-isolated centroid of $\vc$, and $c'$ a centroid of $\vc'$ such that
  $\norm{c_i-c'}<\sep$. By the triangle inequality, for any other centroid $c_j\neq c_i$ of
  $\vc$, $\norm{c_j-c'}\geq \norm{c_{j}-c_{i}}-\norm{c_{i}-c'}>\sep$, and since $d(c_j\|\vc')<\sep$, the latter distance is attained
  for a centroid of $\vc'$ different from $c'$. Hence moving arbitrarily $c'$ can only leave $d(c_j\|\vc')$
  unaltered or smaller, while obviously also the distance $d(c'_j\|\vc)$ remains unlaltered for
  all other centroids $c'_j \neq c'$ of $\vc'$ which are unchanged.
  We can therefore replace $c'$ by $c_i$ in $\vc'$
  while only  making $d(\vc\|\vc')$, as well as $d(\vc'\|\vc)$,  possibly smaller
  ($d(c'\|\vc)$ as well as $d(c_i\|\vc')$ are set to zero with this operation, the other
  distances can only shrink by the above argument).
  We can repeat this operation for all $2\sep$-isolated centroids of $\vc$, leading to the announced
  claim.
\end{proof}

\begin{proof}[Proof of Lemma~\ref{lem:interpretableclusteringriskbound}]
  Recalling $\hyp^{\star}=(c_1,\ldots,c_k)\in \HypClass_{k,0,R^{\star}}$ is the collection of centroids of $\Prob^\star=\sum_{i=1}^k \alpha_i
  \delta_{c_i}$, let $h=(c'_1,\ldots,c'_k)$ be any element in $\HypClass_{k,2\sep,R^\star}$.
  Denote $I(\hyp^{\star},h) := \set{i; 1\leq i \leq k: \exists j: c_i=c'_j}$ the index set of centroids
of $\hyp^{\star}$ that are also found in $h$. Then
\begin{align*}
\Risk_{\mathtt{k-medians}}(\Prob^{\star},\hyp) 
= \sum_{i=1}^k \alpha_i \min_{1\leq j \leq k} \norm{c_i-c_j'}
\leq \paren{\sum_{i\not\in I(\hyp^{\star},h)} \alpha_i} d(\vc\|h).
\end{align*}
Similarly, $\Risk_{\mathtt{k-means}}(\Prob^{\star},\hyp)\leq \paren{\sum_{i\not\in I(\hyp^{\star},h)} \alpha_i} d(\vc\|h)^2$.
We apply these estimates to the centroid sets $h$ with the guarantees of  Lemma~\ref{lem:DistToSeparatedSet}, resp.~\ref{lem:DistToSeparatedSet2} (and take the minimum of the two).
Namely, Lemma~\ref{lem:DistToSeparatedSet} guarantees the existence of $\hyp\in\HypClass_{k,2\sep,R^{\star}}$
with $I_{2\sep}(\hyp^{\star}) \subset I(\hyp^{\star},\hyp)$ and $d(\hyp^{\star},\hyp) \leq 2\eps$, while 
Lemma~\ref{lem:DistToSeparatedSet2} guarantees the existence of $\hyp'\in\HypClass_{k,2\sep,R^{\star}}$
with $I_{4\sep}(\hyp^{\star}) \subset I(\hyp^{\star},\hyp')$ and $d(\hyp^{\star},\hyp') = d(\hyp^{\star},\HypClass_{k,2\sep,R^{\star}}).$
Combining with Lemma~\ref{lem:bounddisttosepclust} and Lemma~\ref{lem:ClustRiskLipschitzWrtDist} we obtain the result.
\end{proof}

\section{Remaining proofs for  Sections~\ref{sec:clustering} and~\ref{sec:gmm}} %
\label{sec:separatedornot}

\subsection{Existence of an (unconstrained) GMM risk minimizer}
\label{sec:prexistencegmmmin}

Let $\hyp^{(n)}=(c_1^{(n)},...,c_k^{(n)},\alpha_1^{(n)},...,\alpha_k^{(n)})$ be a sequence such that
\[\Risk_{\mathtt{GMM}}(\Prob,\hyp^{(n)}) \stackrel{n\rightarrow \infty}{\longrightarrow} \Risk_{\mathtt{GMM}}^* := \inf_{h \in \RR^{kd}\times \Simplex_{k-1}} \Risk_{\mathtt{GMM}}(\Prob,\hyp).\]
 By continuity %
 of $h \mapsto \Risk_{\mathtt{GMM}}(\Prob,\hyp)$, it
suffices to prove that $(\hyp^{(n)})_{n\geq 1}$ has an accumulation point in order to establish the
existence of a minimum of $\Risk_{\mathtt{GMM}}(\Prob,\hyp)$. If all the centroids
$c_i^{(n)}$ remain bounded, this is the case, by compactness. If not, we can assume without
loss of generality (up to permutation and taking a subsequence) that $\norm{c_1^{(n)}} \rightarrow
\infty$. %

Let  $\hyp^{(n)}_0:=(0,c^{(n)}_2,\ldots,c^{(n)}_k,\alpha^{(n)}_1,\ldots,\alpha^{(n)}_k)$. Then,
denoting, for any $c \in \RR^d$,  $\phi_c(t) := \exp(-\normmah{t}{\Cov}^2/2)$ the unnormalized
 density of the Gaussian distribution centered
 in $c$ and of covariance $\Cov$,
it holds for any input point $\sample$:
\begin{align*}
  -\log \Prob_{\hyp^{(n)}_0}(\sample) - (-\log \Prob_{\hyp^{(n)}}(\sample))
&= \log \paren{ \frac{\sum_{i=1}^k {\alpha^{(n)}_{i}} \phi_{c^{(n)}_i}(\sample)}{\alpha^{(n)}_1 \phi_0(\sample)  + \sum_{j=2}^k {\alpha^{(n)}_{j}} \phi_{c^{(n)}_j}(\sample)}}\\
  & = \log \paren{1 + \frac{\alpha^{(n)}_1(\phi_{c^{(n)}_1}(\sample) - \phi_{0}(\sample))}{\alpha^{(n)}_1 \phi_0(\sample)  + \sum_{j=2}^k {\alpha^{(n)}_{j}} \phi^{(n)}_{c_j}(\sample)}}  \\
  & \leq \log \paren{1 + \paren{\frac{\phi_{c_1^{(n)}}(\sample)}{\phi_0(\sample)} - 1} _+}\\
  & = \paren{\log \paren{\frac{\phi_{c_1^{(n)}}(\sample)}{\phi_0(\sample)}}}_+\\
  & = \frac{1}{2}\paren{\normmah{\sample}{\Cov}^2-\normmah{\sample-c_1^{(n)}}{\Cov}^2}_+ \\
  & =  \paren{\inner{\sample,c^{(n)}_1}_{\Cov} - \frac{1}{2} \normmah{c^{(n)}_1}{\Cov}^2}_+ .\\
\end{align*}
Taking expectations (note that integrability follows from the existence of the first moment
of $\Prob$, itself following from the assumption of GMM loss integrability under $\Prob$, which implies
existence of moments up to order 2), %
\[
  \Risk_{\mathtt{GMM}}(\Prob,\hyp_0^{(n)}) - \Risk_{\mathtt{GMM}}(\Prob,\hyp^{(n)})
  \leq \Exp_{\Sample \sim \Prob}\left[\paren{\inner{\Sample,c^{(n)}_1}_{\Cov} - \frac{1}{2} \normmah{c^{(n)}_1}{\Cov}^2}_+\right].
\]
Let $f_c: x\mapsto \paren{\inner{x,c}_{\Cov} - \frac{1}{2} \normmah{c}{\Cov}^2}_+$;
we have that $\sup_c f_c(x) = \frac{1}{2}\normmah{x}{\Cov}^2$, and $f_c$ converges pointwise to 0 as $\normmah{c}{\Cov} \rightarrow \infty$. Since $\normmah{c^{(n)}_1}{\Cov} \rightarrow \infty$,
and $\Prob$ has finite second order moments,
by dominated convergence we get that the right-hand side above converges to 0, and that
$\lim_{n\rightarrow \infty} \Risk_{\mathtt{GMM}}(\Prob,\hyp_0^{(n)}) =
\lim_{n\rightarrow \infty} \Risk_{\mathtt{GMM}}(\Prob,\hyp^{(n)}) = \Risk_{\mathtt{GMM}}^*$.
Repeating this operation as necessary with other centroids diverging to infinity,
we see that we can replace the sequence $\hyp^{(n)}$ by a sequence remaining in a compact
and with the same limit for the risk, for which an accumulation point exists, attaining the
minimum of the risk.

\subsection{Control of $\dnormloss{\dipoleSet}{}$}\label{subsec:dnormlossclustgmm}

For compressive $k$-means / $k$-medians, with the loss defined in~\eqref{eq:KMeansMediansLoss}, we consider a constrained hypothesis class $\HypClass$ such that $\ModelCT(\HypClass) = \MixSetSep{k}$, where $\BasicSet := \BasicSet_{\mathtt{Dirac}} = \rev{(\ParamSpace_R,\norm{\cdot}_2/\sep,\embd)}$ is as in Definition~\ref{def:DiracGaussian},
depending on some separation parameter and radius $0<\sep \leq R$; \rev{we recall $\ParamSpace_R = \Ball_{\RR^\sampleDim,\norm{\cdot}_{2}}(0,R)$ and $\embd(\Param)=\delta_\Param$,
  and underline that while the separation parameter $\eps$ does not change the base distribution set $\embd(\ParamSpace)$, it will determine separation in the mixture and dipole
  sets derived from it.}

 \begin{lemma}\label{lem:DiracDLoss}
Consider $0<\sep \leq R$, $\BasicSet = \BasicSet_{\mathtt{Dirac}}$ based on the parameter set $\ParamSpace_{R} = \set{\Param \in \RR^{\sampleDim}: \norm{\Param}_{2} \leq R}$. Consider $\LossClass(\HypClass)$ associated to $k$-means (resp. $k$-medians) with $\rev{\HypClass \subseteq} \HypClass_{R} := \set{\hyp = (c_{1},\ldots,c_{k}), \norm{c_{l}}_{2} \leq R}$. For any shift-invariant kernel $\kernel$ that is $1$-strongly locally characteristic with respect to $\BasicSet$ we have
\begin{eqnarray}
\dnormloss{\dipoleSet}{} 
\leq 
\rev{2 \cdot} \normkern{\Prob_{0}}^{-1} \cdot (2R)^{p}  %
 \end{eqnarray}
with $p=2$ for $k$-means, $p=1$ for $k$-medians. 
 \end{lemma}
 For compressive Gaussian mixture modeling we use $\BasicSet := \BasicSet_{\mathtt{Gauss}}$ as in Definition~\ref{def:DiracGaussian}.

 \begin{lemma}\label{lem:GaussDLoss}
 Consider $0<\sep \leq R$, $\BasicSet = \BasicSet_{\mathtt{Gauss}}$ based on $\ParamSpace_{R} = \set{\Param \in \RR^{\sampleDim}: \normmah{\Param}{\Cov} \leq R}$, and $\LossClass(\HypClass)$ associated to $k$-mixtures of\footnote{not-necessarily separated} Gaussians $\Prob_{l} = \mathcal{N}(c_{l},\covar)$, with $\rev{\HypClass \subseteq} \HypClass_{R} := \set{\hyp = (c_1,\ldots,c_{k},\alpha), \normmah{c_{l}}{\covar} \leq R, \alpha \in \Simplex_{k-1}}$.

For any shift-invariant kernel $\kernel$ that is $1$-strongly locally characteristic with respect to $\BasicSet$ we have
\begin{eqnarray}
\dnormloss{\dipoleSet}{} 
\leq 
 \rev{2 \cdot} \normkern{\Prob_{0}}^{-1}  \cdot 2R^{2} %
 \end{eqnarray}
\end{lemma}

\begin{proof}[Proof of Lemmas-\ref{lem:DiracDLoss}-\ref{lem:GaussDLoss}]
Since $\kernel$ is $1$-strongly locally characteristic with respect to $\BasicSet$ we can use Theorem~\ref{thm:RadiusGeneric}. Slightly abusing notation (confusing $\BasicSet = (\ParamSpace,\metricParam,\embd)$ with $\embd(\ParamSpace)$) we denote $\normfclass{\BasicSet}{G} := \normfclass{\embd(\ParamSpace)}{G}$ and observe that $\normfclass{\monopoleSet}{G} = \normkern{\Prob_{0}}^{-1}  \cdot \normfclass{\BasicSet}{G}$, and that $\embd: \Param \mapsto \Prob_{\Param}$ is $L'_{\GClass}$-Lipschitz with respect to $\norm{\cdot}$, $\normfclass{\cdot}{G}$ if, and only if  $\psi$ is $L_{\GClass}$-Lipschitz with respect to $\norm{\cdot}$ and $\normfclass{\cdot}{G}$, with $L_{\GClass} = L'_{\GClass} \normkern{\Prob_{0}}^{-1} $. By Theorem~\ref{thm:RadiusGeneric}, if we can show that $\embd: \Param \mapsto \Prob_{\Param}$ is $L'_{\DLossClass}$-Lipschitz with respect to $\norm{\cdot}$, $\dnormloss{\cdot}{}$ then $\psi$ has the desired Lipschitz property with $L_{\DLossClass} \leq \normkern{\Prob_{0}}^{-1}  L'_{\DLossClass}$ and
\[
\normkern{\Prob_{0}}^{-1}  \cdot \dnormloss{\BasicSet}{}
=
\dnormloss{\monopoleSet}{} \leq \dnormloss{\dipoleSet}{} \leq \normkern{\Prob_{0}}^{-1}  
\rev{\paren{L'_{\GClass}+\dnormloss{\BasicSet}{}}}.
\]
The rest of the proof consists in characterizing $\dnormloss{\BasicSet}{}$ and bounding $L'_{\DLossClass}$.

For this we consider $\dloss(\cdot,\hyp,\hyp') = \loss(\cdot,\hyp)-\loss(\cdot,\hyp') \in \DLossClass(\HypClass_{R})$ where $\hyp,\hyp' \in \HypClass_{R}$. 

With $\BasicSet = \BasicSet_{\mathtt{Dirac}}$ and the loss associated to compressive clustering, given $\hyp = (c_1,\ldots,c_k)$, for each $\Param \in \ParamSpace_{R}$ the triangle inequality yields $\norm{\Param-c_{l}}_{2} \leq \norm{\Param}_{2}+\norm{c_{l}}_{2} \leq 2R$ hence $0 \leq \loss(\Param,\hyp) \leq (2R)^p$ where we recall that $p=2$ for $k$-means and $p=1$ for $k$-medians. Similarly $0 \leq \loss(\Param,\hyp') \leq (2R)^{p}$ hence $g(\Param) := \Exp_{\Sample \sim \Prob_{\Param}} \dloss(\Sample,\hyp,\hyp') = \dloss(\Param,\hyp,\hyp')$ satisfies $g(\Param) \leq (2R)^{p}$. This shows that
\[
\dnormloss{\BasicSet_{\mathtt{Dirac}}}{} \leq (2R)^{p}.
\]
The bound is reached using $\Param$ such that $\norm{\Param}_{2}=R$, $c_{1}=\ldots = c_{k} = -\Param$, $\hyp = (c_{1},\ldots,c_{k})$, $\hyp' = -\hyp$. 

Given $\Param,\Param' \in \ParamSpace_{R}$, let $i$ be an index such that $\loss(\Param',\hyp) = \min_{l} \norm{\Param'-c_{l}}_{2}^{p} = \norm{\Param'-c_i}_{2}^{p}$. By definition, $\loss(\Param,\hyp) = \min_{l} \norm{\Param-c_{l}}_{2}^{p} \leq \norm{\Param-c_i}_{2}^{p}$ hence 
\(
\loss(\Param,\hyp)-\loss(\Param',\hyp) \leq \norm{\Param-c_i}_{2}^{p}-\norm{\Param'-c_i}_{2}^{p}. 
\)
Similarly, with $j$ such that $\loss(\Param,\hyp') = \|\Param-c'_j\|_{2}^{p}$ (where $\hyp' = (c'_{1},\ldots,c'_{k})$) we get 
\(
\loss(\Param',\hyp')-\loss(\Param,\hyp')  \leq \norm{\Param'-c'_j}_{2}^{p}-\norm{\Param-c'_j}_{2}^{p}
\)
hence
\begin{align*}
g(\Param)-g(\Param')
= 
\left[\loss(\Param,\hyp)-\loss(\Param,\hyp')\right]
-\left[\loss(\Param',\hyp)-\loss(\Param',\hyp')\right]
&= 
\left[\loss(\Param,\hyp)-\loss(\Param',\hyp)\right]
+
\left[\loss(\Param',\hyp')-\loss(\Param,\hyp')\right]\\
& \leq \norm{\Param-c_i}_{2}^{p}-\norm{\Param'-c_i}_{2}^{p} + \norm{\Param'-c'_j}_{2}^{p}-\norm{\Param-c'_j}_{2}^{p}.
\end{align*}
For $k$-medians, $p=1$ and the reversed triangle inequality further yields
\begin{align*}
\norm{\Param-c_i}_{2}-\norm{\Param'-c_i}_{2} + \norm{\Param'-c'_j}_{2}-\norm{\Param-c'_j}_{2}
& \leq 2 \norm{\Param-\Param'}_{2} = 2 (2R)^{p-1} \norm{\Param-\Param'}_{2}.
\end{align*}
In the case of $k$-means, $p=2$ and we use
\begin{align*}
\norm{\Param-c_i}_{2}^2-\norm{\Param'-c_i}_{2}^2 + \norm{\Param'-c'_j}_{2}^2-\norm{\Param-c'_j}_{2}^2
&= 2 \inner{\Param-\Param',c'_{j}-c_{i}}\\
& \leq 2 \norm{\Param-\Param'}_{2} \cdot \norm{c'_{j}-c_{i}}_{2}\\
& \leq 4R \norm{\Param-\Param'}_{2}
= 2 (2R)^{p-1} \norm{\Param-\Param'}_{2}.
\end{align*}
By symmetry we obtain $\abs{g(\Param) - g(\Param')}\leq 2(2R)^{p-1} \norm{\Param-\Param'}_{2}$. As this holds for any $\Param,\Param' \in \ParamSpace_{R}$ and $g \in \DLossClass$, and as $\norm{\Param-\Param'}_{2} = \sep \norm{\Param-\Param'}$, we get $L'_{\DLossClass} \leq \sep 2(2R)^{p-1}=(\sep/R)\dnormloss{\BasicSet_{\mathtt{Dirac}}}{}$.

With $\BasicSet = \BasicSet_{\mathtt{Gauss}}$ and the loss associated to compressive GMM, we prove 
\rev{at the end of this section} %
that for any $\hyp,\hyp' \in \HypClass_{R}$ the function $g(\Param) := \Exp_{\Sample \sim \Prob_{\Param}} \dloss(\Sample,\hyp,\hyp')$ satisfies
\begin{align}
\abs{g(\Param)} &\leq 2R^{2}
\label{eq:DLossBoundGMM}\\
\abs{g(\Param)-g(\Param')} & \leq 2R \sep \norm{\Param-\Param'}.
\label{eq:DLossLipGMM}
\end{align}
\rev{where the first bound is reached.}
We obtain $\dnormloss{\BasicSet_{\mathtt{Gauss}}}{} = 2R^{2}$, $L'_{\DLossClass} \leq 2\sep R = (\sep/R)\dnormloss{\BasicSet_{\mathtt{Gauss}}}{}$.

In both cases since $R \geq \sep$ we have 
$\rev{L'_{\DLossClass}+\dnormloss{\BasicSet}{} \leq 2 \dnormloss{\BasicSet}{}}$.
\end{proof}

The following lemma, which applies to any family of absolutely continuous probability distributions $\set{\Prob_{\Param}: \Param \in \ParamSpace}$ on $\RR^d$, will be soon specialized to Gaussians with fixed known covariance.
\begin{lemma}\label{lem:BoundedMLLoss}
  Consider a family of probability distributions $\set{\Prob_{\Param}: \Param \in \ParamSpace}$ on $\RR^d$ having a density with respect to the Lebesgue measure. Recalling
  $\textnormal{H}$ denotes the differential entropy~\eqref{eq:DefEntropy}, assume that
\begin{align}
\textnormal{H}_{\min} &:= \inf_{\Param \in \ParamSpace} \Entropy{\Prob_{\Param}} > -\infty;\\
\textnormal{H}_{\max} &:= \sup_{\Param,\Param' \in \ParamSpace} \Entropy{\Prob_{\Param}}+\KLdiv{\Prob_{\Param}}{\Prob_{\Param'}} < \infty.
\end{align}
For any $\Prob_{\hyp} := \sum_{l=1}^{k}\alpha_{l} \Prob_{\Param_{l}}$, where $\alpha \in \Simplex_{k-1}$, $\Param_{l} \in \ParamSpace$ we have for any $\Param \in \ParamSpace$:
\[
\textnormal{H}_{\min} 
\leq \Exp_{\Sample \sim \Prob_{\Param}} \left[-\log \Prob_{\hyp}(\Sample)\right] 
\leq \textnormal{H}_{\max}.
\]
The lower and the upper bounds are both tight.
\end{lemma}
\begin{proof}
  By the \rev{definition} of the Kullback-Leibler divergence and \rev{its convexity
    properties}, we have
  \rev{
    \begin{align*}
      \Entropy{\Prob_{\Param}}
\leq 
\Entropy{\Prob_{\Param}}
+\KLdiv{\Prob_{\Param}}{\Prob_{\hyp}}
&= \Entropy{\Prob_{\Param}}
                                        +\KLdiv{\Prob_{\Param}}{\sum_{l=1}^{k}\alpha_{l}\Prob_{\Param_{l}}} \\
      & \leq \Entropy{\Prob_{\Param}}
        +\sum_{l=1}^{k}\alpha_{l}\KLdiv{\Prob_{\Param}}{\Prob_{\Param_{l}}}
\leq
\Entropy{\Prob_{\Param}} + \sup_{\Param' \in \ParamSpace} \KLdiv{\Prob_{\Param}}{\Prob_{\Param'}}.
    \end{align*}
    }
For a given $\Param$, both the lower and the upper bound are tight. The conclusion immediately follows.
\end{proof}
This translates into a concrete result for Gaussian mixtures with fixed known covariance.
\begin{proof}[Proof of Equations~\eqref{eq:DLossBoundGMM}-\eqref{eq:DLossLipGMM} - end of the proof of Lemma~\ref{lem:GaussDLoss}]

To establish~\eqref{eq:DLossBoundGMM} we exploit Lemma~\ref{lem:BoundedMLLoss}. The entropy of a Gaussian is $\Entropy{\Prob_{\Param}} = \tfrac{1}{2} \log \det{2\pi e \covar}$ which is independent of $\Param$, hence $\textnormal{H}_{\min}=\tfrac{1}{2} \log \det{2\pi e \covar}$. 
The Kullback-Leibler divergence has a closed form expression in the case of multivariate Gaussians (see e.g. \citealp{Duchi2007}):
\begin{equation}
\label{eq:kl_gauss}
\KLdiv{\mathcal{N}(\Param_1,\covar_1)}{\mathcal{N}(\Param_2,\covar_2)} 
= \tfrac{1}{2}\left[\log\frac{\det{\covar_2}}{\det{\covar_1}} + \mathrm{tr}\left(\covar_2^{-1}\covar_1\right)-\sampleDim + \left(\Param_2-\Param_1\right)^T \covar_2^{-1}\left(\Param_2-\Param_1\right) \right];
\end{equation}
hence $\KLdiv{\Prob_{\Param}}{\Prob_{\Param'}} =\tfrac12 \normmah{\Param-\Param'}{\covar}^2$. \rev{Since $\normmah{\Param}{\covar} \leq R$ when $\Param \in \ParamSpace_{R}$, we have}
\(
\textnormal{H}_{\max} = \textnormal{H}_{\min}+ \tfrac{1}{2} \normmah{\Param-\Param'}{\covar}^2 \leq \textnormal{H}_{\min} + 2 R^{2}.
\) 
By Lemma~\ref{lem:BoundedMLLoss} we obtain~\eqref{eq:DLossBoundGMM} as follows 
\[
\abs{g(\Param)} = 
\abs{\Exp_{\Sample \sim \Prob_{\Param}} \dloss(\Sample,\hyp,\hyp')}
=
\abs{\Exp_{\Sample \sim \Prob_{\Param}} [-\log \Prob_{\hyp}(\Sample)]
-\Exp_{\Sample \sim \Prob_{\Param}} [-\log \Prob_{\hyp'}(\Sample)]
} \leq  \textnormal{H}_{\max}-\textnormal{H}_{\min} \leq 2R^{2}.
\]
\rev{where the bound is tight.} 

\rev{We now turn to Equation~\eqref{eq:DLossLipGMM}.}
Denoting $\loss_{\hyp}(\sample) := -\log \Prob_{\hyp}(\sample)$ and $f_{\hyp}(\Param) := 
\Exp_{\Sample \sim \Prob_{\Param}} \loss_{\hyp}(\Sample)$,
since $\Prob_{\Param}(\sample) = \Prob_{0}(\sample-\Param) = \Prob_{0}(\Param-\sample)$ we have 
\begin{align*}
f_{\hyp}(\Param) 
& = \int \Prob_{\Param}(\sample) \loss_{\hyp}(\sample) d\sample
= \int \Prob_{0}(\Param-\sample) \loss_{\hyp}(\sample) d\sample
= \int \Prob_{0}(\sample) \loss_{\hyp}(\Param-\sample) d\sample;\\
\nabla f_{\hyp}(\Param) 
&=
\int \Prob_{0}(\sample) \nabla \loss_{\hyp}(\Param-\sample) d\sample 
=
\int \Prob_{0}(\Param-\sample) \nabla \loss_{\hyp}(\sample) d\sample 
= 
\Exp_{\Sample \sim \Prob_{\Param}} \nabla \loss_{\hyp}(\Sample);\\
\nabla \loss_{\hyp}(\sample)
&=
-\frac{\nabla \Prob_{\hyp}(\sample)}{\Prob_{\hyp}(\sample)}
=
-\frac{
\sum_{l}\alpha_{l} \Prob_{\Param_{l}}(\sample) \cdot 
\frac{\nabla \Prob_{\Param_{l}}(\sample)}{\Prob_{\Param_{l}}(\sample)}
}{
\Prob_{\hyp}(\sample)
}
=
-\sum_{l} \frac{
\alpha_{l} \Prob_{\Param_{l}}(\sample)
}{
\Prob_{\hyp}(\sample)
}
\cdot \frac{\nabla \Prob_{\Param_{l}}(\sample)}{\Prob_{\Param_{l}}(\sample)}
=
-\sum_{l} \beta_{l}(\sample) \cdot \nabla \log \Prob_{\Param_{l}}(\sample),
\end{align*}
where $\beta_{l}(\sample) := \frac{
\alpha_{l} \Prob_{\Param_{l}}(\sample)
}{
\sum_{l}\alpha_{l} \Prob_{\Param_{l}}(\sample)
}\geq 0$ satisfies $\sum_{l} \beta_{l}(\sample) =1$. Since 
$\nabla \log \Prob_{\Param_{l}}(\sample) = -\covar^{-1}(\sample-\Param_{l})$, 
we obtain
\begin{align*}
\nabla f_{\hyp}(\Param) 
&=
\Exp_{\Sample \sim \Prob_{\Param}}
\sum_{l} \beta_{l}(\Sample) \cdot \covar^{-1}(\Sample-\Param_{l})
=
\covar^{-1}\Exp_{\Sample \sim \Prob_{\Param}} 
\big(\Sample -\sum_{l} \beta_{l}(\Sample) \Param_{l}
\big)
=
\covar^{-1}\big(\Param-\sum_{l} \gamma_{l} \cdot \Param_{l}\big),
\end{align*}
with $\gamma_{l}:= \Exp_{\Sample \sim \Prob_{\Param}} \beta_{l}(\Sample) \geq 0$, $\sum_{l}\gamma_{l} = 1$. 
\rev{Similarly we have $\nabla f_{\hyp'}(\Param) = \covar^{-1}\left(\Param-\sum_{l} \gamma'_{l} \cdot \Param'_{l}\right)$
where $\gamma'_{l}\geq0$ and $\sum_{l} \gamma'_{l}=1$.}
\rev{Since $g(\Param) = f_{\hyp}(\Param)-f_{\hyp'}(\Param)$, and $\normmah{\Param_{l}}{\covar} \leq R$, $\normmah{\Param'_{l}}{\covar} \leq R$, we get}
\begin{align*}
\normmah{\nabla g(\Param)}{\covar^{-1}}
=
\norm{\covar^{1/2} (\nabla f_{\hyp}(\Param)-\nabla f_{\hyp'}(\Param))}_{2} 
&= 
\norm{\covar^{-1/2}\big(\sum_{l} \gamma'_{l} \cdot \Param'_{l}-\sum_{l} \gamma_{l} \cdot \Param_{l}\big)}_{2}\\
&=
\normmah{\sum_{l} \gamma'_{l} \cdot \Param'_{l}-\sum_{l} \gamma_{l} \cdot \Param_{l}}{\covar}
\leq 2R.
\end{align*}
To obtain~\eqref{eq:DLossLipGMM}, given $\Param,\Param'$, defining $\Param(t) := \Param+t(\Param'-\Param)$ we have
\begin{eqnarray*}
\abs{g(\Param')-g(\Param)} 
&=& 
\abs{\int_{0}^{1} \frac{d}{dt} g(\Param(t)) dt}
= 
\abs{\int_{0}^{1} \inner{ \nabla g(\Param(t)),\Param'-\Param} dt}
\leq 
\int_{0}^{1} \normmah{\nabla g(\Param(t))}{\covar^{-1}} \normmah{\Param'-\Param}{\covar} dt\\
&\leq& 2R  \normmah{\Param'-\Param}{\covar} = 2R\sep \norm{\Param-\Param'}.
\end{eqnarray*}
\end{proof}

\subsection{Elements to control the bias term}\label{subsec:biasclustgmm}
While Gaussian Mixture Modeling is a maximum likelihood task, with a negative log-likelihood loss, both $k$-means and $k$-medians are \emph{compression-type tasks}
as defined in~\citeppartone{}:
 \begin{definition}[Compression-type task, \citeppartone{Definition~\ref{P1-def:comptypetask}}] \label{def:comptypetask}
    We call the learning task a {\em compression-type task} if the loss can be written as
   $\loss(\sample,\hyp) = \divg^{p}(\sample,P_{\hyp}\sample)$, where $d$ is a metric on the sample space $\SampleSpace$, $p>0$, and
$P_{\hyp}: \SampleSpace \to \SampleSpace$ is a projection function, i.e.,
$P_{\hyp} \circ P_{\hyp} = P_{\hyp}$ and $\divg(x,P_{\hyp}x) \leq  \divg(x,P_{\hyp}x')$ for all $x,x' \in \SampleSpace$.
 \end{definition}
\noindent For $k$-means and $k$-medians, $d(\sample,\sample') = \norm{\sample-\sample'}_{2}$ is the Euclidean distance on $\SampleSpace = \RR^{\sampleDim}$. Given $\hyp = (c_{1},\ldots,c_{k})$, the function $P_{\hyp}$ maps $\sample$ to the closest $c_{j}$, with ties broken arbitrarily, and can be used to define a {\em Voronoi partition} $W_{j}(\hyp) := P_{\hyp}^{-1}(c_{j}) = \set{\sample \in \RR^{\sampleDim}: P_{\hyp} \sample = c_{j}}$, i.e. a collection of pairwise disjoint sets such that $\cup_{j} W_{j}(\hyp) = \RR^{\sampleDim}$ and $W_{j}(\hyp) \subseteq V_{j}(\hyp)$ with $V_{j}(\hyp)$ the Voronoi cells defined in~\eqref{eq:DefVoronoiCell}. 
The  push-forward $P_\hyp\Prob$ of a probability distribution $\Prob$ through $P_\hyp$ is the probability distribution of  $Y = P_\hyp X$ when $X \sim \Prob$. Here it reads more explicitly as $P_{\hyp} \Prob = \sum_{j=1}^{k} \alpha_{j} \delta_{c_{j}}$ with $\alpha_{j} = \Prob(\Sample \in W_{j}(\hyp))$.

\rev{The goal of the next result is to have a device to relate the excess risk with respect to hypotheses in the restricted class $\HypClass$, which will be controlled
  via Theorem~\ref{thm:LRIPsuff_excess} \eqref{eq:MainBoundExcessRisk}-\eqref{eq:DefMDist2},
  to the excess risk with respect to the optimal in an unconstrained (or less constrained) class, $\HypClassRef$, e.g. $\HypClassRef = (\RR^{\sampleDim})^{k}$ for $k$-means (resp. $\HypClassRef = (\RR^{\sampleDim})^{k} \times \Simplex_{k-1}$ for GMM).
  Observe that the main control~\eqref{eq:MainBoundExcessRisk} on
  the restricted hypothesis class takes the form
  \[
    \forall \hyp_{0} \in \HypClass : \drisk_{\hyp_{0}}(\Prob,\hat{\hyp}) 
  \leq \distIOPexgen^{\HypClass}_{\hyp_{0}}(\Prob,\Model) + \Delta(\Prob,\empProb),
\]
where the trailing rest term $\Delta(\Prob,\empProb)$ does not depend on $\hyp_0$. If $\hyp^\star$ denotes the optimal hypothesis over the larger class $\HypClassRef$, we deduce
from the above:
\[
  \forall \hyp_{0} \in \HypClass : \drisk_{\hyp^\star}(\Prob,\hat{\hyp})
  = \drisk_{\hyp_0}(\Prob,\hat{\hyp}) +
  \drisk_{\hyp^\star}(\Prob,\hyp_0)
  \leq \paren{\drisk_{\hyp^\star}(\Prob,\hyp_0) + \distIOPexgen^{\HypClass}_{\hyp_{0}}(\Prob,\Model)} + \Delta(\Prob,\empProb).
\]
It is therefore of interest to further upper bound the first term in the above estimate.
This is what we obtain in the following result.
}
\begin{lemma}\label{lem:BiasKMeansConstrained}
Consider a compression-type task with $P_{\hyp}$ defined for any $\hyp \in \HypClassRef$. With the notations and assumptions of Theorem~\ref{thm:LRIPsuff_excess} on a class $\HypClass \subseteq \HypClassRef$ with $\Model = \ModelCT(\HypClass)$, %
considering
$\Prob$ a probability distribution on $\SampleSpace$ with integrable loss, $\hyp^{\star} \in \arg\min_{\hyp \in \HypClassRef} \Risk(\Prob,\hyp)$, and $\Prob^\star := P_{\hyp^{\star}}\Prob$ we have 
\begin{equation}\label{eq:BiasKMeansConstrained1}
\inf_{\hyp_{0}\in \HypClass}
\left\{
\drisk_{\hyp^{\star}}(\Prob,\hyp_{0})+\distIOPexgen_{\hyp_{0}}^{\HypClass}(\Prob,\Model)\right\}
 \leq
\divp_{\hyp^{\star}}^{\rev{\HypClassRef}}(\Prob\|\Prob^\star) + (2+\nu)C_{\SketchingOperatorProb}\norm{\SketchingOperatorProb(\Prob)-\SketchingOperatorProb(\Prob^\star)}_{2} \notag
 + d^{\HypClass}(\Prob^{\star},\Model),
 \end{equation}
with
 \begin{equation}\label{eq:BiasKMeansConstrained2}
d^{\HypClass}(\Prob^{\star},\Model)
  := \inf_{\mProb \in \Model} 
\left\{
\sup_{\hyp \in \HypClass} \left(\Risk(\Prob^{\star},\hyp)-\Risk(\mProb,\hyp)\right)
+ (2+\nu)C_{\SketchingOperatorProb}\norm{\SketchingOperatorProb(\Prob^{\star})-\SketchingOperatorProb(\mProb)}_{2}\right\}.
\end{equation}
The same holds for a maximum likelihood task with $\loss(\sample,\hyp) = -\log \Prob_{\hyp}(\sample)$, using $\Model = \ModelML(\HypClass)$, $\Prob^\star = \Prob_{\hyp^{\star}}$ and
 \begin{equation}\label{eq:BiasMaxLikelihoodConstrained2}
d^{\HypClass}(\Prob^{\star},\Model)
  := \inf_{\mProb \in \Model} 
\left\{
\sup_{\hyp \in \HypClass} \left(\KLdiv{\Prob^{\star}}{\Prob_{\hyp}}-\KLdiv{\mProb}{\Prob_{\hyp}}\right)
+ (2+\nu)C_{\SketchingOperatorProb}\norm{\SketchingOperatorProb(\Prob^{\star})-\SketchingOperatorProb(\mProb)}_{2}\right\}.
\end{equation}
\end{lemma}
\begin{proof}
  Let $\hyp_{0}, \hyp \in \HypClass$, and $\mProb \in \Model$ be given.
  Since $\drisk_{a}(\Prob,b) + \drisk_{b}(\Prob,c) = \drisk_{a}(\Prob,c)$ for $a,b,c \in \HypClassRef$, we have 
\begin{align*}
\drisk_{\hyp^{\star}}(\Prob,\hyp_{0}) + \drisk_{\hyp_{0}}(\Prob,\hyp)-\drisk_{\hyp_{0}}(\mProb,\hyp)
=&
\drisk_{\hyp^{\star}}(\Prob,\hyp)
-\drisk_{\hyp^{\star}}(\mProb,\hyp)
-\drisk_{\hyp_{0}}(\mProb,\hyp^{\star})\\
=&
\drisk_{\hyp^{\star}}(\Prob,\hyp)
-\drisk_{\hyp^{\star}}(\mProb,\hyp)
+\drisk_{\hyp^{\star}}(\mProb,\hyp_{0})\\
=&
\left(\drisk_{\hyp^{\star}}(\Prob,\hyp)
-\drisk_{\hyp^{\star}}(\Prob^{\star},\hyp)\right)\\
&+\left(\drisk_{\hyp^{\star}}(\Prob^{\star},\hyp)
-\drisk_{\hyp^{\star}}(\mProb,\hyp)\right)
+\drisk_{\hyp^{\star}}(\mProb,\hyp_{0}).
\end{align*}
Taking the supremum over $\hyp \in \HypClass \subseteq \HypClassRef$ 
and 
\rev{denoting $\divp_{\hyp^{\star}}^{\HypClass}(\Prob^{\star}\|\mProb) := \sup_{\hyp \in \HypClass} 
\left(\drisk_{\hyp^{\star}}(\Prob^{\star},\hyp)
-\drisk_{\hyp^{\star}}(\mProb,\hyp)\right)$ (even though $\hyp^{\star}$ may not belong to $\HypClass$) yields}
\begin{align*}
\drisk_{\hyp^{\star}}(\Prob,\hyp_{0}) + \divp_{\hyp_{0}}^{\HypClass}(\Prob\|\mProb) 
\leq &
\divp_{\hyp^{\star}}^{\HypClassRef}(\Prob\|\Prob^{\star})
+
\divp_{\hyp^{\star}}^{\HypClass}(\Prob^{\star}\|\mProb)
+\drisk_{\hyp^{\star}}(\mProb,\hyp_{0}),
\end{align*}
hence by a triangle inequality
\begin{align}
\drisk_{\hyp^{\star}}(\Prob,\hyp_{0}) + \divp_{\hyp_{0}}^{\HypClass}(\Prob\|\mProb) 
+ (2+\nu)C_{\SketchingOperatorProb}\norm{\SketchingOperatorProb(\Prob)-\SketchingOperatorProb(\mProb)}_{2}
\leq  
& \divp^{\rev{\HypClassRef}}_{\hyp^{\star}}(\Prob\|\Prob^{\star})
+ (2+\nu)C_{\SketchingOperatorProb}\norm{\SketchingOperatorProb(\Prob)-\SketchingOperatorProb(\Prob^\star)}_{2}\notag\\
&
+\divp_{\hyp^{\star}}^{\HypClass}(\Prob^{\star}\|\mProb) +\drisk_{\hyp^{\star}}(\mProb,\hyp_{0})\notag\\
& + (2+\nu)C_{\SketchingOperatorProb}\norm{\SketchingOperatorProb(\Prob^{\star})-\SketchingOperatorProb(\mProb)}_{2}.
\label{eq:TmpTTTT}
\end{align}
The joint infimum of~\eqref{eq:TmpTTTT} over $\hyp_{0} \in \HypClass$ and $\mProb \in \Model$ yields
\begin{align}
\inf_{\hyp_{0} \in \HypClass}
\Big\{
\drisk_{\hyp^{\star}}(\Prob,\hyp_{0}) 
+& \distIOPexgen_{\hyp_{0}}^{\HypClass}(\Prob,\Model)
\Big\}
\leq  
\divp_{\hyp^{\star}}^{\rev{\HypClassRef}}(\Prob\|\Prob^{\star})
+
(2+\nu)C_{\SketchingOperatorProb}\norm{\SketchingOperatorProb(\Prob)-\SketchingOperatorProb(\Prob^{\star})}_{2}\notag\\
&
+\inf_{\mProb \in \Model} 
\left\{
\Big[\divp_{\hyp^{\star}}^{\HypClass}(\Prob^{\star}\|\mProb)
+ \inf_{\hyp_{0} \in \HypClass} \drisk_{\hyp^{\star}}(\mProb,\hyp_{0})\Big]
+ (2+\nu)C_{\SketchingOperatorProb}\norm{\SketchingOperatorProb(\Prob^{\star})-\SketchingOperatorProb(\mProb)}_{2}
\right\}.
\label{eq:TmpTTTT1}
\end{align}
For any $\hyp \in \HypClassRef$ we have
\(
\drisk_{\hyp^{\star}}(\Prob^{\star},\hyp)-\drisk_{\hyp^{\star}}(\mProb,\hyp)
=
\Risk(\Prob_{\star},\hyp)-\Risk(\mProb,\hyp)
+\Risk(\mProb,\hyp^{\star})-\Risk(\Prob^{\star},\hyp^{\star})
\)
hence $\divp_{\hyp^{\star}}^{\HypClass}(\Prob^{\star}\|\mProb) = \sup_{\hyp \in \HypClass} \left(\Risk(\Prob_{\star},\hyp)-\Risk(\mProb,\hyp)\right)+\Risk(\mProb,\hyp^{\star})-\Risk(\Prob^{\star},\hyp^{\star})
$.

To conclude observe that for a compression-type task, we have $\Risk(\Prob^\star,\hyp^{\star}) = 0$ and, for $\mProb \in \ModelCT(\HypClass)$, 
\[
\inf_{\hyp_{0} \in \HypClass} \drisk_{\hyp^{\star}}(\mProb,\hyp_{0}) 
= \inf_{\hyp_{0} \in \HypClass} \left\{\Risk(\mProb,\hyp_{0})-\Risk(\mProb,\hyp^{\star})\right\}
= -\Risk(\mProb,\hyp^{\star}).
\]
As a result $\divp_{\hyp^{\star}}^{\HypClass}(\Prob^{\star}\|\mProb)+\inf_{\hyp_{0} \in \HypClass} \drisk_{\hyp^{\star}}(\mProb,\hyp_{0}) = \sup_{\hyp \in \HypClass} \left(\Risk(\Prob_{\star},\hyp)-\Risk(\mProb,\hyp)\right)$.

For a maximum likelihood task, $\Risk(\Prob^{\star},\hyp^{\star})=\Entropy{\Prob^{\star}}$ with $H(\cdot)$ the entropy, and $\Risk(\mProb,\hyp_{0}) = \KLdiv{\mProb}{\Prob_{\hyp_{0}}}+\Entropy{\mProb}$ for $\hyp_{0} \in \HypClass$, hence 
\begin{align*}
\inf_{\hyp_{0} \in \HypClass} \drisk_{\hyp^{\star}}(\mProb,\hyp_{0}) 
=& \inf_{\hyp_{0} \in \HypClass} \left\{\Risk(\mProb,\hyp_{0})-\Risk(\mProb,\hyp^{\star})\right\}
   = \Entropy{\mProb}-\Risk(\mProb,\hyp^{\star});
\end{align*}
as a result
\begin{align*}
\divp_{\hyp^{\star}}^{\HypClass}(\Prob^{\star}\|\mProb)+\inf_{\hyp_{0} \in \HypClass} \drisk_{\hyp^{\star}}(\mProb,\hyp_{0}) 
=&
 \sup_{\hyp \in \HypClass} \left(\Risk(\Prob_{\star},\hyp)-\Risk(\mProb,\hyp)\right) + \Entropy{\mProb}-\Entropy{\Prob^{\star}}\notag\\
=& \sup_{\hyp \in \HypClass} \left(\KLdiv{\Prob^{\star}}{\Prob_\hyp}-\KLdiv{\mProb}{\Prob_\hyp}\right).
 \end{align*}
\end{proof}
Next we deal with the term $\divp_{\hyp^{\star}}^{\HypClassRef}(\Prob\|\Prob^{\star})$ in Lemma~\ref{lem:BiasKMeansConstrained}. For $k$-medians  by \citeppartone{Lemma~\ref{P1-lem:LemmaBiasTerm}} we have $\divp_{\hyp^{\star}}^{\HypClassRef}(\Prob\|\Prob^{\star})=0$. We now show that this also holds for $k$-means when $\HypClassRef = (\RR^{\sampleDim})^{k}$.
\begin{lemma}\label{lem:BiasKMeans}
Consider $\loss$ the loss associated to $k$-means on a class $\HypClass$ and $\Prob$ a probability distribution on $\RR^{\sampleDim}$ with integrable loss. 
\begin{itemize}
\item  $\divp_{\hyp_{0}}^{\HypClass}(\Prob\| P_{\hyp_{0}}\Prob) = 0$ for each $\hyp_{0} = (c_{1},\ldots,c_{k}) \in \HypClass$ such that 
\begin{equation}\label{eq:CentroidAssumptionExcessRiskDiv}
\Prob(\Sample \in W_{j}(\hyp_{0})) \neq 0
\Longrightarrow
c_{j} = \Exp_{\Prob}  (\Sample|\Sample \in W_{j}(\hyp_{0})),\qquad 
\forall 1 \leq j \leq k,
\end{equation}

\item $\divp_{\hyp^{\star}}^{\HypClassRef}(\Prob\| P_{\hyp^{\star}}\Prob) = 0$ for each $\hyp^{\star} \in \arg\min_{\hyp \in \HypClassRef} \Risk(\Prob,\hyp)$ with $\HypClassRef = (\RR^{\sampleDim})^{k}$.
\item 
If $\HypClass \subseteq \HypClass_{R} := \set{\hyp = (c_{1},\ldots,c_{k}), \norm{c_{l}}_{2} \leq R}$ then 
\begin{equation}
\divp_{\hyp_{0}}^{\HypClass}(\Prob\| P_{\hyp_{0}}\Prob) \leq 4R  \cdot \Risk_{\mathtt{k-medians}}(\Prob,\hyp_{0}),\qquad
\forall \hyp_{0} \in \HypClass.
\end{equation}
\end{itemize}
\end{lemma}
\begin{proof}
By~\citeppartone{Equation~\eqref{P1-proj2}}, for any $\sample \in \SampleSpace = \mathbb{R}^{d}$ and $\hyp \in \HypClass$ we have
\[
\|x-P_{\hyp}x\|_{2}^{2} \leq \|x-P_{\hyp}P_{\hyp_{0}} x\|_{2}^{2}
= \|x-P_{\hyp_{0}}x\|_{2}^{2}+\|P_{\hyp_{0}}x-P_{\hyp}P_{\hyp_{0}}x\|_{2}^{2}+2 \langle x-P_{\hyp_{0}}x,P_{\hyp_{0}}x-P_{\hyp}P_{\hyp_{0}}x \rangle.
\]
It follows thus from~\citeppartone{Equation~\eqref{P1-eq:boundcomp1}} that 
\begin{align*}
  \drisk_{\hyp_{0}}(\Prob,\hyp)-\drisk_{\hyp_{0}}(P_{\hyp_{0}}\Prob,\hyp)
& \leq 
\mathbb{E}_{X \sim \Prob} \left\{
\|X-P_{\hyp}X\|_{2}^{2}
-\|X-P_{\hyp_{0}}X\|_{2}^{2}-\|P_{\hyp_{0}}X-P_{\hyp}P_{\hyp_{0}}X\|_{2}^{2}
\right\}\\
& \leq
2\mathbb{E}_{X \sim \Prob} \langle X-P_{\hyp_{0}}X,P_{\hyp_{0}}X-P_{\hyp}P_{\hyp_{0}}X \rangle.
\end{align*}
We have $ \langle x-P_{\hyp_{0}}x,P_{\hyp_{0}}x-P_{\hyp}P_{\hyp_{0}}x \rangle = \sum_{j=1}^{k} \mathbf{1}(x \in W_{j}(\hyp_{0})) \langle x-c_{j},c_{j}-P_{\hyp}c_{j}\rangle$ hence
\begin{align*}
\mathbb{E}_{X \sim \Prob} \langle X-P_{\hyp_{0}}X,P_{\hyp_{0}}X-P_{\hyp}P_{\hyp_{0}}X \rangle
&=
\sum_{j=1}^{k}\mathbb{E}_{X \sim \Prob} \left\{ 1(X \in W_{j}(\hyp_{0}))  \langle X-c_{j},c_{j}-P_{\hyp}c_{j}\rangle \right\}.
\end{align*}
When $\hyp_{0}$ satisfies~\eqref{eq:CentroidAssumptionExcessRiskDiv}, each term on the right hand side vanishes, either because $\Prob(\Sample \in W_{j}(\hyp_{0})) = 0$ or because $c_{j} = \Exp_{\Prob}  (\Sample|\Sample \in W_{j}(\hyp_{0}))$.
As a result $\drisk_{\hyp_{0}}(\Prob,\hyp)-\drisk_{\hyp_{0}}(P_{\hyp_{0}}\Prob,\hyp) \leq 0$ for any $\hyp$. As $\divp_{\hyp}^{\HypClassRef}$ is non-negative, we get $\divp_{\hyp_{0}}^{\HypClassRef}(\Prob\|P_{\hyp_0}\Prob) = 0$.

If the support of $\Prob$ contains at least $k$ elements then the unconstrained $k$-means optimizer $\hyp^{\star}$ on $\HypClassRef$ satisfies the centroid condition~\eqref{eq:CentroidCondition}, which implies that for $1 \leq j \leq k$ we have $\Prob(\Sample \in W_{j}(\hyp_{0})) >0$ and $c_{j} = \mathbb{E}_{X\sim \Prob} (X | X \in W_{j}(\hyp_{0}))$, hence assumption~\eqref{eq:CentroidAssumptionExcessRiskDiv} holds and we can use the result established above. It is straightforward to check that~\eqref{eq:CentroidAssumptionExcessRiskDiv} holds as well
if the support of $\Prob$ contains at most $k-1$ elements, i.e., if $\Prob$ is a mixture of $k-1$ Diracs.

When $\HypClass \subset \HypClass_{R}$ we have $\langle \sample-P_{\hyp_{0}}\sample,P_{\hyp_{0}}\sample-P_{\hyp}P_{\hyp_{0}}\sample \rangle \leq \norm{\sample-P_{\hyp_{0}}\sample}_{2} \cdot 2R$ hence $\drisk_{\hyp_{0}}(\Prob,\hyp)-\drisk_{\hyp_{0}}(P_{\hyp_{0}}\Prob,\hyp) \leq 4R \cdot  \Exp_{\Sample \sim \Prob} \norm{\Sample-P_{\hyp_{0}}\Sample}_{2} = \Risk_{\mathtt{k-medians}}(\Prob,\hyp_{0})$.
 \end{proof}
The term $d^{\HypClass}(\Prob^{\star},\Model)$ can also be simplified for clustering when $\Model = \ModelCT(\HypClass)$.
\begin{lemma}\label{lem:BoundSeparationTermClustering}
Consider $\Prob^{\star} :=  \sum_{i=1}^{k} \alpha_{i} \delta_{c_{i}}$
where $c_{1},\ldots,c_{k} \in \RR^{\sampleDim}$, $\alpha \in \Simplex_{k-1}$ and the $k$-medians (resp $k$-means) task with a class $\HypClass$. For $k$-medians we have 
\[
\sup_{\hyp' \in \HypClass} 
\Big(\Risk_{\mathtt{k-medians}}(\Prob^{\star},\hyp')-\Risk_{\mathtt{k-medians}}(P_{\hyp}\Prob^{\star},\hyp')\Big)
=
\Risk_{\mathtt{k-medians}}(\Prob^{\star},\hyp),\qquad \forall \hyp \in \HypClass.
\]
For $k$-means and any $\hyp_{0} \in \HypClass$ such that~\eqref{eq:CentroidAssumptionExcessRiskDiv} holds with $\Prob := \Prob^{\star}$ we have
\[
\sup_{\hyp' \in \HypClass} 
\Big(\Risk_{\mathtt{k-means}}(\Prob^{\star},\hyp')-\Risk_{\mathtt{k-means}}(P_{\hyp_{0}}\Prob^{\star},\hyp')\Big)
=
\Risk_{\mathtt{k-means}}(\Prob^{\star},\hyp_{0}).
\]
If  $\HypClass \subseteq \HypClass_{R} := \set{\hyp = (c_{1},\ldots,c_{k}), \norm{c_{l}}_{2} \leq R}$ we further have for each $\hyp \in \HypClass$
\[
\Risk_{\mathtt{k-means}}(\Prob^{\star},\hyp) \leq 
\sup_{\hyp' \in \HypClass} 
\Big( \Risk_{\mathtt{k-means}}(\Prob^{\star},\hyp')-\Risk_{\mathtt{k-means}}(P_{\hyp}\Prob^{\star},\hyp')\Big)
\leq
\Risk_{\mathtt{k-means}}(\Prob^{\star},\hyp)+4R \cdot \Risk_{\mathtt{k-medians}}(\Prob^{\star},\hyp).
\]
\end{lemma}
\begin{proof}
Since $\Risk(P_{\hyp}\Prob^{\star},\hyp)=0$ for both $k$-means and $k$-medians we have 
\begin{align*}
\sup_{\hyp' \in \HypClass} 
\Big(
\Risk(\Prob^{\star},\hyp')-\Risk(P_{\hyp}\Prob^{\star},\hyp') 
\Big)
&= 
\sup_{\hyp' \in \HypClass} 
\Big(
\drisk_{\hyp}(\Prob^{\star},\hyp')-\drisk_{\hyp}(P_{\hyp}\Prob^{\star},\hyp') 
\Big)
+\Risk(\Prob^{\star},\hyp)\\
&= \divp_{\hyp}^{\HypClass}(\Prob^{\star},P_{\hyp}\Prob^{\star}) +\Risk(\Prob^{\star},\hyp)
\end{align*}
For $k$-medians, we conclude using that  $\divp_{\hyp}^{\HypClass}(\Prob^{\star},P_{\hyp}\Prob^{\star})=0$  by \citeppartone{Lemma~\ref{P1-lem:LemmaBiasTerm}}. Using Lemma~\ref{lem:BiasKMeans} yields the results for $k$-means.
\end{proof}

\subsection{Proof of Theorems~\ref{thm:mainkmeansthm} and Theorem~\ref{thm:maingmmthm}}\label{subsec:clustfinalproof}

With the separated hypothesis class $\HypClassSep := \HypClass_{k,2\sep,R}$ defined in~\eqref{eq:DefSeparation} (resp. in~\eqref{eq:HypClassGMM}), the model set is a separated mixture model, $\Model(\HypClass) := \ModelCT(\HypClass) \subset \MixSetSep{k}$, where $\BasicSet = \BasicSet_{\mathtt{Dirac}}$ (resp. $\Model(\HypClass) := \ModelML(\HypClass) \subseteq \MixSetSep{k}$ with $\BasicSet = \BasicSet_{\mathtt{Gauss}}$) as in Definition~\ref{def:DiracGaussian}., 
By definition of $\HypClassSep$ (cf. \eqref{eq:DefSeparation} and~\eqref{eq:HypClassGMM}) the centers $c_{l}$ associated to any $\hyp \in \HypClassSep$ satisfy $\max_{l}\norm{c_{l}}_{2} \leq R$ (resp. $\normmah{c_{k}}{\covar} \leq R$ for GMM) hence 
\rev{the parameter space $\ParamSpace$ is the ball of radius $R$ with respect to the Euclidean norm (resp. the Mahalanobis norm $\normmah{\cdot}{\covar}$).} 

The function $\SketchingOperator$ is defined using the random Fourier feature family $(\FClass_{\mathtt{Dirac}},\freqdist_{\mathtt{Dirac}})$ (resp. $(\FClass_{\mathtt{Gauss}},\freqdist_{\mathtt{Gauss}})$) as in Definition~\ref{def:DiracGaussian}, with scale factor $s>0$.  By the derivations in Section~\ref{sec:mmddiracgauss} 
 the induced average kernel $\kernel$ is shift-invariant and $1$-strongly locally characteristic with respect to $\BasicSet$. 
 
We now control the constants $C_{\ParamSpace},A,B,C$ from Theorem~\ref{thm:LRIPDiracGaussMixture}. By~\eqref{eq:P0NormInvSq} we have 
\begin{equation}\label{eq:P0NormInvSqThmPf}
\normkern{\Prob_{0}}^{-1} = 
\begin{cases}
C_{\freqdist} = \left[\Exp_{\freq \sim \mathcal{N}(0,s^{-2} \mI_{\sampleDim})} w^{2}(\freq)\right]^{1/2}, & \text{for Diracs},\\
\left(1+2/s^2\right)^{\sampleDim/4},& \text{for Gaussians}.
\end{cases}
\end{equation}

Consider first the Dirac setting. Since $\ParamSpace$ is the Euclidean ball of radius $R \geq \sep$ (see Definition~\ref{def:DiracGaussian}), by Lemma~\ref{lem:covnumball} we get that~\eqref{eq:covparamspace1} holds with  $C_\ParamSpace=\rev{4R}$. Moreover, recall that $w(\freq) = (1 + s^2 d^{-1} \norm{\freq}^2)$. Then, \rev{since $(\sep/s)^2 = 1/(\sigma_{k}^{\star})^{2} = 16 \log(ek)$}, elementary calculations give:
\begin{align*}
\sep \sup_{\freq} \tfrac{\norm{\freq}_{2}}{w(\freq)} = \sep \sqrt{\frac{d}{4s^{2}}} = \sqrt{\frac{d}{4(\sigma_{k}^{\star})^{2}}} = \sqrt{4d \log(ek)};
 & \qquad
\sep^{2} \sup_{\freq} \tfrac{\norm{\freq}_{2}^{2}}{w(\freq)} = \sep^{2} \frac{d}{s^{2}} = 16d \log(ek);\\
\Exp_{\freq \sim \mathcal{N}(0,s^{-2} \mI_{\sampleDim})}  \norm{\freq}_{2}^{2} = s^{-2} d;
& \qquad \Exp_{\freq \sim \mathcal{N}(0,s^{-2} \mI_{\sampleDim})}  \norm{\freq}_{2}^{4} = s^{-4} (d^{2}+2d);\\
\end{align*}
\begin{align*}
    \normkern{\Prob_{0}}^{-2} = A &= \Exp_{\freq \sim \mathcal{N}(0,s^{-2} \mI_{\sampleDim})} w^{2}(\freq)  
    \rev{= 1+2 s^{2} d^{-1}\Exp_{\freq} \norm{\freq}_{2}^{2} + s^{4}d^{-2}\Exp_{\freq} \norm{\freq}_{2}^{4}
    = 1+ 2+(1+2/d) \leq 6};\\
    B &= 1+\sep^{2} \left(\sup_{\freq} \tfrac{\norm{\freq}_{2}}{w(\freq)}\right)^{2}
  \rev{= 1 + 4d \log(ek) \leq 5d\log(ek)};\\
      C &=
    64 \rev{A\sqrt{2B}} C_\ParamSpace \sep^{-1}  \paren{1 + \sep \sup_\freq \tfrac{\norm{\freq}_{2}}{w(\freq)} + \sep^2
      \sup_\freq \tfrac{\norm{\freq}^2_{2}}{w(\freq)}}  \\
    & \rev{\lesssim  \sqrt{d\log(ek)}(R/\sep)
    \paren{1+\sqrt{4d \log(ek)} + 16d \log(ek)}
    \lesssim (d\log(ek))^{3/2} R/\sep},
  \end{align*}
\rev{  where $\lesssim$ denotes an inequality up to a numerical multiplicative factor. It follows that $\min(12e \log^{2}(ek),2B) \lesssim \log(ek) \min(\log(ek),d)$, $k ABC \lesssim k(d\log(ek))^{5/2} R/\sep$, and 
  $\log(kABC) \lesssim 1+\log(kd) + \log(R/\sep)$. As a result there is a numerical constant $C'$ such that~\eqref{eq:NMeasuresDiracGauss} holds as soon as }
 \[
\rev{ m \geq C' \coveps^{-2} \log(ek) \min(\log(ek),d) \cdot k \cdot \left\{ kd \cdot \left[1+\log(kd)+\log(R/\sep)+\log(1/\coveps)\right]+\log(1/\zeta)\right\}.}
 \]
 Rearranging the terms to put the dominant terms forward, this holds under the assumption~\eqref{eq:sksizekmeans}.

Consider now the GMM setting.
  As in the Dirac case, since $\ParamSpace$ is the ball of radius $R$ in the Mahalanobis distance $\normmah{\cdot}{\covar}$, by Lemma~\ref{lem:covnumball} we get that~\eqref{eq:covparamspace1} holds with $C_\ParamSpace=4R$. Then, by \eqref{eq:MainSeparationAssumption}, $\sep^2 = (2+s^2)/(\sigma^{\star}_k)^2 = 16 (2+s^2) \log(ek) \rev{\asymp s^{2} \log(ek)}$ 
  \rev{where $a \asymp b$ means $a \lesssim b$ and $b \lesssim a$,} and we have:
  \begin{align*}
    \normkern{\Prob_{0}}^{-1} = \sqrt{A} 
    & = \left(1+2/s^2\right)^{\sampleDim/4} ;\\
    B &= 1+\sep^{2}  \lesssim s^2 \log(ek);\\
          C &=
    64 \rev{A\sqrt{2B}} C_\ParamSpace \sep^{-1}  \paren{1 + \sep + \sep^2}  \lesssim
R \left(1+2/s^2\right)^{\sampleDim/\rev{2}} s^2 \log(ek) .
  \end{align*}
  It follows that $\min(12e \log^{2}(ek),2B) \lesssim \log(ek) \min(\log(ek),s^2)$,
  $k ABC \lesssim Rk(1+2/s^2)^{\rev{d}}s^4\log^2(ek)$, and 
  $\log(kABC) \lesssim 1+ \log (R) + d/s^2 + \log(ks)$. As a result there is a numerical constant $C'$ such that~\eqref{eq:NMeasuresDiracGauss} holds as soon as 
 \[
  m \geq C' \coveps^{-2} (1+2/s^2)^{d/2} \log(ek) \min(\log(ek),s^2)
     \cdot k \cdot \{kd \cdot [1+\log(R) + d/s^2 + \log(ks) + \log(1/\coveps)] + \log(1/\zeta) \}.
   \]
 Rearranging the terms to put the dominant terms forward, this holds under the assumption~\eqref{eq:sksizeGMM}.
 
We have all ingredients to apply Theorem~\ref{thm:LRIPDiracGaussMixture} hence, with probability at least $1-\zeta$ the operator $\SketchingOperatorProb$ induced by $\SketchingOperator$ satisfies~\eqref{eq:mainRIPKernelClustering}.

As a result, under the assumption~\eqref{eq:sksizekmeans} (resp. ~\eqref{eq:sksizeGMM}) we have all ingredients to apply Theorem~\ref{thm:LRIPDiracGaussMixture} hence, with probability at least $1-\zeta$ the operator $\SketchingOperatorProb$ induced by $\SketchingOperator$ satisfies~\eqref{eq:mainRIPKernelClustering} (resp.~\eqref{eq:mainRIPKernelGMM}). 

By~\eqref{eq:mainRIPKernelClustering} (resp.~\eqref{eq:mainRIPKernelGMM}) and the shift-invariance of $\kernel$, for any $\Param,\Param' \in \RR^{\sampleDim}$ such that $\norm{\Param-\Param'}_{2} \leq \sep$ (resp. $\normmah{\Param-\Param'}{\covar} \leq \sep$) we have $\norm{\SketchingOperatorProb(\Prob_\Param)-\SketchingOperatorProb(\Prob_{\Param'})}_{2}^{2} \leq (1+\coveps) \normkern{\Prob_{\Param}-\Prob_{\Param'}}^{2} = 2\normkern{\Prob_{0}}^{2}(1+\coveps)(1-\nkernel(\Param-\Param'))$ hence
\begin{equation*}
\norm{\SketchingOperatorProb(\Prob_\Param)-\SketchingOperatorProb(\Prob_{\Param'})}_{2}^{2} 
\leq 2\normkern{\Prob_{0}}^{2}(1+\coveps)
\begin{cases}
1-e^{-\frac{\norm{\Param-\Param'}_{2}^{2}}{2\sep^{2}(\sigma^{\star}_{k})^{2}}} & \text{(for clustering)};\\
1-e^{-\frac{\normmah{\Param-\Param'}{\covar}^{2}}{2\sep^{2}(\sigma^{\star}_{k})^{2}}} & \text{(for GMM)}.
 \end{cases}
 \end{equation*}
As $f: u \mapsto 1-e^{-\tfrac{u}{2\sep^{2}(\sigma^{\star}_{k})^{2}}}$ is concave we have $f(u) \leq f(0)+u f'(0)$ for each $u \in \RR$, hence $\Param \mapsto \SketchingOperatorProb(\Prob_{\Param})$ is $L$-Lipschitz with respect to $\norm{\cdot}_{2}$ (resp. $\normmah{\cdot}{\covar}$) in $\RR^{\sampleDim}$ and $\norm{\cdot}_{2}$ in $\CC^{\nMeasures}$, with $L=\normkern{\Prob_{0}}\sqrt{1+\delta}/(\sep \sigma^{\star}_k)$. 
For clustering we have $\SketchingOperatorProb(\Prob_{\Param}) = \SketchingOperatorProb(\delta_{\Param}) = \SketchingOperator(\Param)$ and $\normkern{\Prob_{0}} \leq 1$ (by~\eqref{eq:P0NormInvSqThmPf} and the fact that $w \geq 1$) hence the claimed Lipschitz property of $\SketchingOperator$.

A second consequence of~\eqref{eq:mainRIPKernelClustering} is the LRIP~\eqref{eq:lowerDRIP} on $\Model(\HypClass) \subseteq \MixSetSep{k}$) with $\eta=0$ and $C_\SketchingOperatorProb :=8\sqrt{2k/(1-\coveps)}\norm{\dipoleSet}_{\DLossClass(\HypClass)}$. By Theorem~\ref{thm:LRIPsuff_excess}, since $\hat{\hyp}$ satisfies~\eqref{eq:DefQuasiOptimumClusteringViaProxy} (resp.~\eqref{eq:DefQuasiOptimumGMMViaProxy}), 
we get
\begin{align*} 
\forall \hyp_{0} \in \HypClass:\qquad
  \drisk_{\hyp_{0}}(\Prob,\hat{\hyp}) \leq 
\distIOPexgen_{\hyp_{0}}^{\HypClass}(\Prob,\Model(\HypClass))+  (2+\nu) C_{\SketchingOperatorProb}  \norm{\SketchingOperatorProb(\Prob)-\SketchingOperatorProb(\empProb))}_{2} + C_{\SketchingOperatorProb}\nu'.
\end{align*}
Since $\HypClass \subseteq \HypClassRef$ denoting $\Prob^{\star} := P_{\hyp^{\star}}\Prob$ (resp. $\Prob^{\star} := \Prob_{\hyp^{\star}}$), we have by Lemma~\ref{lem:BiasKMeansConstrained}, with $d(\Prob^{\star},\HypClass)$ as in ~\eqref{eq:BiasKMeansConstrained2} (resp. as in~\eqref{eq:BiasMaxLikelihoodConstrained2}):
\begin{eqnarray}
  \drisk_{\hyp^{\star}}(\Prob,\hat{\hyp})
  &=&  \drisk_{\hyp^{\star}}(\Prob,\hyp_{0})  +  \drisk_{\hyp_{0}}(\Prob,\hat{\hyp})
  =      \inf_{\hyp_{0} \in \HypClassOpt}
\left\{\drisk_{\hyp^{\star}}(\Prob,\hyp_{0})  +  \drisk_{\hyp_{0}}(\Prob,\hat{\hyp})\right\}\notag\\
      &\leq &
       (2+\nu) C_{\SketchingOperatorProb}  \norm{\SketchingOperatorProb(\Prob)-\SketchingOperatorProb(\empProb))}_{2} + C_{\SketchingOperatorProb}\nu'
+        \inf_{\hyp_{0} \in \HypClass}
    \left\{
    \drisk_{\hyp^{\star}}(\Prob,\hyp_{0}) +
    \distIOPexgen_{\hyp_{0}}^{\HypClass}(\Prob,\Model(\HypClass))
    \right\}
       \notag\\
   & \leq &
   (2+\nu)C_{\SketchingOperatorProb}
\norm{\SketchingOperatorProb(\Prob)-\SketchingOperatorProb(\empProb)}_{2} 
+ C_{\SketchingOperatorProb}\nu'\notag\\
&&+ \Big[\divp_{\hyp^{\star}}^{\HypClassRef}(\Prob\|\Prob^\star)+(2+\nu)C_{\SketchingOperatorProb}
\norm{\SketchingOperatorProb(\Prob)-\SketchingOperatorProb(\Prob^{\star})}_{2} 
    \Big]+d(\Prob^{\star},\HypClass)\notag.
\end{eqnarray}
The excess risk divergence term $\divp_{\hyp^{\star}}^{\HypClassRef}(\Prob\|\Prob^{\star})$ vanishes for $k$-medians by \citeppartone{Lemma~\ref{P1-lem:LemmaBiasTermBis}}. Since $\HypClassRef = (\RR^{\sampleDim})^{k}$ it also vanishes for $k$-means by Lemma~\ref{lem:BiasKMeans}.  

To conclude, we explicit the involved constants. For clustering since $\normkern{\Prob_{0}}^{-1} \leq \sqrt{6}$, by Lemma~\ref{lem:DiracDLoss} we get $\norm{\dipoleSet}_{\DLossClass(\HypClass)} \leq  \rev{2} (2R)^{p} \normkern{\Prob_{0}}^{-1} \leq \sqrt{24}(2R)^{p}$. Since $8 \sqrt{2}\sqrt{24} = 8 \sqrt{48} \leq 8 \cdot 7 = 56$ we obtain:
\begin{align}
C_\SketchingOperatorProb^{\mathtt{clust.}} \leq  56\sqrt{k/(1-\delta)} (2R)^{p}.
\end{align}
For GMM, Lemma~\ref{lem:GaussDLoss} yields $\norm{\dipoleSet}_{\DLossClass(\HypClass)} \leq  \rev{4}R^2 \normkern{\Prob_{0}}^{-1}$. By~\eqref{eq:P0NormInvSqThmPf}, since $8\sqrt{2} \cdot 4 = 32\sqrt{2} \leq 46$ we obtain
\begin{align}
C_\SketchingOperatorProb^{\mathtt{GMM}} \leq  46\sqrt{k/(1-\delta)} R^{2} \left(1+2/s^2\right)^{\sampleDim/4}.
\end{align}

\subsection{Proof of Lemma~\ref{lem:bounddisttosepclust}}\label{pf:bounddisttosepclust}
Consider $\hyp \in \HypClass$ and $\mProb := P_{h} \Prob^{\star} = \sum_{i=1}^{k}\alpha_{i} \delta_{P_{\hyp} c_{i}}$. 
%Gathering the terms in $\mProb$ indexed by $i \neq j$ such that $P_{\hyp}c_{i} = P_{\hyp}c_{j}$, we see that $\mProb \in \Model^{\star}_{\hyp}$. 
For $k$-medians, by Lemma~\ref{lem:BoundSeparationTermClustering} and the Lipschitz property of $\SketchingOperator$, 
\begin{align*}
\sup_{\hyp' \in \HypClass} \Big(\Risk_{\mathtt{k-medians}}(\Prob^{\star},\hyp')-\Risk_{\mathtt{k-medians}}(\mProb,\hyp')\Big) 
& \leq \Risk_{\mathtt{k-medians}}(\Prob^{\star},\hyp)\\
\norm{\SketchingOperatorProb(\Prob^{\star})-\SketchingOperatorProb(\mProb)}_{2}
 = 
\norm{\sum_{i=1}^{k}\alpha_{i}\big(\SketchingOperator(c_{i})-\SketchingOperator(P_{\hyp}c_{i})\big)}_{2}
& \leq 
\sum_{i=1}^{k}\alpha_{i} \norm{\SketchingOperator(c_{i})-\SketchingOperator(P_{\hyp}c_{i})}_{2}
\leq
\sum_{i=1}^{k}\alpha_{i} L\norm{c_{i}-P_{\hyp}c_{i}}_{2}\\
&=L \cdot \Risk_{\mathtt{k-medians}}(\Prob^{\star},\hyp)
\end{align*}
By the definition~\eqref{eq:BiasKMeansMainThm} of $d(\Prob^{\star},\HypClass)$ this implies
$d(\Prob^{\star},\HypClass) \leq  \inf_{\hyp \in \HypClass} (1+(2+\nu) C_{\SketchingOperatorProb} L) \Risk_{\mathtt{k-medians}}(\Prob^{\star},\hyp)$.
Turning to $k$-means, since $\HypClass \subseteq  \HypClass_{k,2\sep,R} \subset \HypClass_{R} := \set{\hyp = (c_{1},\ldots,c_{k}), \norm{c_{l}}_{2} \leq R}$ by Lemma~\ref{lem:BoundSeparationTermClustering} we get
\begin{align*}
\sup_{\hyp' \in \HypClass} \Big(\Risk_{\mathtt{k-means}}(\Prob^{\star},\hyp')-\Risk_{\mathtt{k-means}}(\mProb,\hyp')\Big) 
& \leq \Risk_{\mathtt{k-means}}(\Prob,\hyp)+ 4R \cdot \Risk_{\mathtt{k-medians}}(\Prob,\hyp).
\end{align*}
The rest of the proof is the same as for $k$-medians. Since $L \leq \sqrt{1+\delta}/s$, $\sep = 4s \sqrt{\log(ek)}$ and 
$C_{\SketchingOperatorProb} \leq 56\sqrt{k/(1-\delta)} (2R)^{p}$ we have $C_{\SketchingOperatorProb}L \leq 224 \sqrt{k\log(ek)(1+\delta)/(1-\delta)}(2R)^{p}/\sep$.

\pagebreak
\twocolumn[\section*{Table of notations}]
\newcommand{\reftwo}[1]{\ref{#1}}
\newcommand{\eqreftwo}[1]{(\reftwo{#1})}

\begin{supertabular}{|p{0.137\textwidth}p{0.313\textwidth}|}
\hline
$\sample \in \SampleSpace$ & sample and sample space\\
$\vy$ & sketch \eqreftwo{eq:GenericSketching} \\
$\SketchingOperator$ & sketching function \eqreftwo{eq:GenericSketching}\\
$\SketchingOperatorProb$ & sketching operator \eqreftwo{eq:SketchingOperatorProbDef} \\
\hline
$\Prob,\mProb$ & probabilities on sample space\\
$\HH$, $\nu$ & measures on sample space\\
$\inner{\Prob,f}$ &  $\Exp_{X \sim \Prob}f(X)$\\
$\inner{\HH,f}$ & $\int f(x) d\HH(x)$\\
$\KLdiv{\Prob}{\Prob'}$ & KL-divergence \eqreftwo{eq:DefEntropy}\\
$\Entropy{\Prob}$ & differential entropy \\
\hline
$\hyp$ & hypothesis\\
$\HypClass \subseteq \HypClassRef$ & classes of hypotheses \\
$\loss(\cdot,\hyp)$ & loss function  \\
$\Risk$, $\drisk_{\hyp}$ & risk \eqreftwo{eq:BestHyp}, excess risk \eqreftwo{eq:defexcessrisk}\\
$\hyp^\star=\hyp^{\star}_{\Prob}$ & best hypothesis \eqreftwo{eq:BestHyp} \\
$\proxyRisk$ & generic proxy for the risk \eqreftwo{eq:DefRiskProxy} \\
$\proxyRisk_{\mathtt{clust.}}$,  $\proxyRisk_{\mathtt{GMM}}$ & specific proxies \eqreftwo{eq:RiskProxyKMeans}; \eqreftwo{eq:RiskProxyGMM} \\
$\hat{\hyp}$ & learned hypothesis \eqreftwo{eq:DefRiskProxy} \\
$P_{\hyp}$ & projection function for comp.-type task (Sec. \reftwo{sec:main-kmeans}) \\
\hline
$\LossClass = \LossClass(\HypClass)$ & class of loss functions (Th. \reftwo{thm:LRIPsuff_excess})\\
$\DLossClass = \DLossClass(\HypClass)$ & class of loss differences \eqreftwo{eq:DefDLossClass} \\ 
\hline
$\kernel(x,x')$ & generic psd kernel \eqreftwo{eq:DefIntegralRepresentation}\\
$\kernel(\Prob,\Prob')$ & kernel mean embedding \eqreftwo{eq:DefMeanMapEmbedding}\\
\hline
$\normfclass{\HH}{G}$ &  $\sup_{f \in \GClass} \abs{\inner{\HH,f}}$, \eqreftwo{eq:DefFNorm} \\
$\normkern{\HH}$ & MMD norm \eqreftwo{eq:DefMMD} \\
$\dnormloss{\cdot}{}$ & task-driven norm \eqreftwo{eq:DefDLossClass} \\
\hline
$\norm{\cdot}$ & generic norm on $\SampleSpace$\\
$\norm{\cdot}_{\star}$ & dual norm \eqreftwo{eq:DefDualNorm} \\
$\norm{\cdot}_{2}$, $\inner{x,x'}$ & Euclidean norm, inner product\\
$\normmah{\cdot}{\covar}$ & Mahalanobis norm \eqreftwo{eq:DefMahalanobisNorm} \\
\hline
$ \divp_{\hyp}^{\HypClass}(\Prob\|\Prob')$ & excess-risk divergence \eqreftwo{eq:lossdiv}  \\
$\distIOPexgen_{\hyp}^\HypClass(\Prob,\Model)$ & bias term wrt. model \eqreftwo{eq:DefMDist2} \\
$d_{\FClass}(\Prob,\Prob')$ & feature-based metric \eqreftwo{eq:DefYetAnotherMetric} \\
$d(\Prob,\HypClass)$ & distance to constraint \eqreftwo{eq:BiasKMeansMainThm}, \eqreftwo{eq:BiasGMMMainThm}\\
$d(\mathbf{c}\|\mathbf{c}'), d(\mathbf{c},\mathbf{c}')$ & distance between $k$-uples \eqreftwo{eq:disttoseparated} \\
\hline
$\Model$ & model set (of probabilities)\\
$\ModelCT_{\hyp}$, $\ModelCT(\HypClass)$ & compression-type model set \eqreftwo{eq:DefKDiracs} \\
$\ModelML_{\hyp}$, $\ModelML(\HypClass)$ & max. likelihood model set \eqreftwo{eq:ModelGMM} \\
$\secant = \secant_{\kernel}(\Model)$ & normalized secant set \eqreftwo{eq:DefNormalizedSecantSet} \\
\hline
$\ParamSpace$, $\ParamSpace_{R}$& parameter set \eqreftwo{eq:DefMixSetSep}; Def. \reftwo{def:DiracGaussian}\\
$\metricParam$& metric on $\ParamSpace$ \eqreftwo{eq:DefMixSetSep}\\
$\embd$ & embedding \eqreftwo{eq:DefMixSetSep} \\
$\BasicSet$ & $(\ParamSpace,\metricParam,\embd)$ parametric model\\
$\MixSetSep{k}$ & $2$-separated $k$-mixtures \eqreftwo{eq:DefMixSetSep} \\
$\dipoleSet = \dipoleSet_\kernel(\BasicSet)$ & set of dipoles \eqreftwo{eq:NormalizedDipoleSet} \\
$\monopoleSet = \monopoleSet_\kernel(\BasicSet)$ & set of monopoles \eqreftwo{eq:NormalizedMonopoleSet} \\
$\nkernel(\Param,\Param')$ & $\BasicSet$-normalized  kernel \eqreftwo{eq:defnkernel}\\
$K$, $\mathtt{K}$ & kernel-related func. \eqreftwo{eq:DefKernelF},\eqreftwo{eq:KShiftBased}\\
$K_{\sigma}$ & Gaussian kernel \eqreftwo{eq:DefDiracGaussianVariance} \\
\hline
$\FClass = \set{\rfeat}_{\freq \in \Omega}$ & generic class of features \eqreftwo{eq:DefRandomFeatureFunction} \\
$\freqdist$ & probability distribution on feature parameters $\freq$ \eqreftwo{eq:DefIntegralRepresentation} \\
$w(\freq)$ & weights (Def. \reftwo{def:weightedRFF})\\
$C_{\freqdist}$ & normalization constant \eqreftwo{eq:DefCfreqdist} \\
$s$ & scale factor \eqreftwo{eq:freqdistGMM}\\
$\sigma(s)$, $\sigma^{\star}_{k}$ & parameters of Gaussian kernel \eqreftwo{eq:DefDiracGaussianVariance} and \eqreftwo{eq:MainSeparationAssumption} \\
$\FClass'$, $\FClass''$ & classes derived from Fourier feature class $\FClass$ (Lem. \reftwo{lem:ShiftBasedRF} and \reftwo{lem:TangentLocationBased}) \\ 
\hline
$\sep$, $R$ & separation, domain bound \eqreftwo{eq:DefSeparation} \\
$\HypClass_{k,\sep,R}$ & constrained hypothesis class \eqreftwo{eq:DefSeparation} \eqreftwo{eq:HypClassGMM} \\
$\Prob^\star$ &  projected distribution (Thm. \reftwo{thm:mainkmeansthm}, Thm. \reftwo{thm:maingmmthm}) \\
$V_{l}(\hyp)$, $W_{l}(\hyp)$ & Voronoi cell, Voronoi partition\\
$\alpha_{l}(\Prob,\hyp)$ & Voronoi weights \eqreftwo{eq:DefVoronoiCell} \\
\hline
$C_{\SketchingOperatorProb}$, $C(K)$, $K_{\max}$, $K'_{\max}$, $K''_{\max}$ & constants
related to kernel $K$ \eqreftwo{eq:lowerDRIP}\eqreftwo{eq:MutualCoherenceRBF1} \\
$L_{\FClass}$ & Lem. \reftwo{lem:ShiftBasedRF} \\
$C_{\FClass}$, $C_{\FClass}'$, $C_{\FClass}''$ & Lem. \reftwo{lem:TangentLocationBased} \\
\hline
$\norm{\mathcal{E}}$ & radius of a set of measures \eqreftwo{eq:DefSetRadius} \\
$\ConcFn(t)$ & concentration function \eqreftwo{eq:PointwiseConcentrationFn} \\
$\covnum{\norm{\cdot}}{A}{\eps}$ & covering numbers \\
$\Ball$ & ball \\
$[Y]_{k,\weightSet}$ & mixture set \eqreftwo{eq:DefMixSet} \\ \hline
\end{supertabular}
\onecolumn

\bibliographystyle{abbrvnat}
\bibliography{complearn_preprint}
\end{document}